\documentclass[journal]{IEEEtran}
\usepackage[switch]{lineno}
\usepackage{graphicx}
\usepackage{float}
\usepackage{subfigure}
\usepackage{amssymb}
\usepackage{array}
\usepackage{amsmath}
\usepackage{url}
\usepackage{multirow}
\usepackage{booktabs}
\usepackage{threeparttable}
\usepackage{fixltx2e}
\usepackage{float}
\usepackage{booktabs}
\usepackage{graphicx}
\usepackage{diagbox}
\usepackage{mathtools}
\usepackage{amsmath}
\usepackage{bbding}
\usepackage{pifont}
\usepackage{wasysym}
\usepackage{tikz}
\usepackage{calc}
\usepackage{amssymb}
\usepackage{multirow}
\usepackage{rotating}
\usepackage{colortbl}
\usepackage{color}
\usepackage{ragged2e}
\usepackage{caption}
\usepackage{cite}
\usepackage{gensymb}
\usepackage{makecell}
\usepackage{enumitem} 
\usepackage[ruled,vlined,linesnumbered]{algorithm2e} 
\usepackage{multicol}
\usepackage{amsthm}
\newtheorem{proposition}{\protect\propositionname}
\providecommand{\propositionname}{Proposition}

\providecommand{\definitionname}{Definition}
\newtheorem{lemma}{\protect\lemmaname}
\providecommand{\lemmaname}{Lemma}
\newtheorem{theorem}{\protect\theoremname}
\providecommand{\theoremname}{Theorem}

\providecommand{\corollaryname}{Corollary}
\newtheorem{remark}{\protect\remarkname}
\providecommand{\remarkname}{Remark}

\providecommand{\assumptionname}{Assumption}
\definecolor{lightblack}{RGB}{48,122,182}
\definecolor{mygray}{gray}{.8}

\usepackage{bm}
\def\ie{\emph{i.e.}}

\usepackage{bbm}
\usepackage{xcolor}

\usepackage{algorithmic}
\usepackage{multirow}
\usepackage{colortbl}

\usepackage{lipsum}

\begin{document}
	
\title{Temporal Visual Semantics-Induced Human Motion Understanding with Large Language Models}

	\author{Zheng~Xing and
	Weibing~Zhao

	\thanks{\indent Weibing~Zhao is with Guangdong Laboratory of Machine Perception and Intelligent Computing, Shenzhen MSU-BIT University, China. (E-mail: weibingzhao@smbu.edu.cn)}

\thanks{\indent Zheng Xing is with the College of Computer Science and Software Engineering, Shenzhen University, Shenzhen 518060, China.}

} %

	\markboth{Class Files,~Vol.~X, No.~X, August~X}%
	{Shell \MakeLowercase{\textit{et al.}}: Bare Demo of IEEEtran.cls for IEEE Journals}
	
	\maketitle
	\begin{abstract}
Unsupervised human motion segmentation (HMS) can be effectively achieved using subspace clustering techniques. However, traditional methods overlook the role of temporal semantic exploration in HMS. This paper explores the use of temporal vision semantics (TVS) derived from human motion sequences, leveraging the image-to-text capabilities of a large language model (LLM) to enhance subspace clustering performance. The core idea is to extract textual motion information from consecutive frames via LLM and incorporate this learned information into the subspace clustering framework. The primary challenge lies in learning TVS from human motion sequences using LLM and integrating this information into subspace clustering. To address this, we determine whether consecutive frames depict the same motion by querying the LLM and subsequently learn temporal neighboring information based on its response. We then develop a TVS-integrated subspace clustering approach, incorporating subspace embedding with a temporal regularizer that induces each frame to share similar subspace embeddings with its temporal neighbors. Additionally, segmentation is performed based on subspace embedding with a temporal constraint that induces the grouping of each frame with its temporal neighbors. We also introduce a feedback-enabled framework that continuously optimizes subspace embedding based on the segmentation output. Experimental results demonstrate that the proposed method outperforms existing state-of-the-art approaches on four benchmark human motion datasets.

	\end{abstract}

	\section{Introduction}
	\label{Sec:Intro}
		\begin{figure}[t]
		\setlength{\abovecaptionskip}{0pt}  
		\setlength{\belowcaptionskip}{0pt}
		\centering{}
		\includegraphics[width=1\columnwidth]{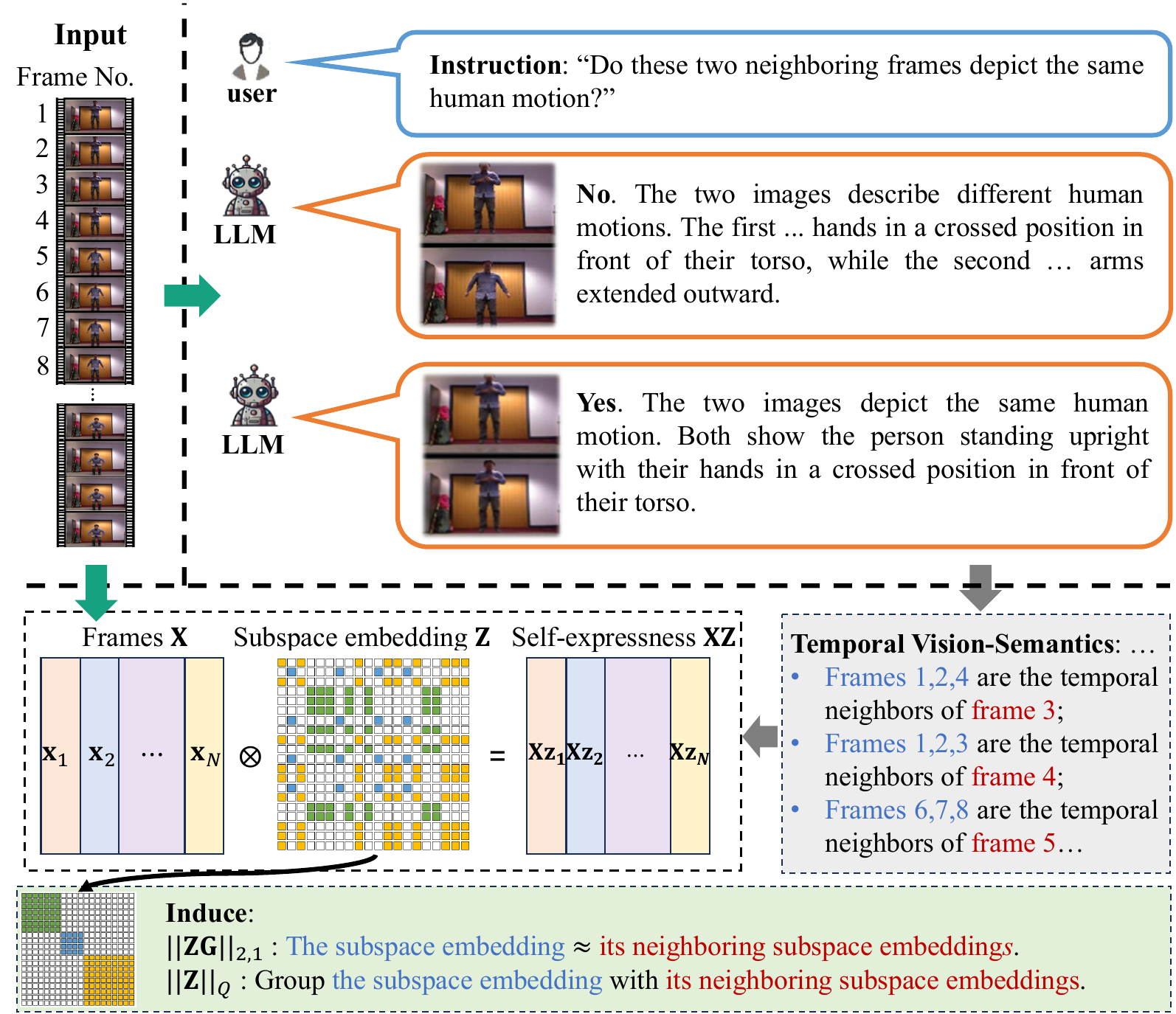}
			\vspace{0.001in}
		\caption{Framework of the proposed method.
		}
		\label{fig:background}
	\end{figure}

\IEEEPARstart{H}{uman} motion segmentation (HMS) has attracted significant attention in both industry and academic research due to its wide-ranging applications in video retrieval, virtual reality, and intelligent surveillance, particularly in human motion analysis \cite{zhou2012hierarchical, keuper2018motion, poppe2007vision}. The primary goal of unsupervised HMS is to partition frame sequences depicting human actions into non-overlapping, internally consistent groups without the need for training, serving as a preprocessing step for motion-related analysis tasks \cite{lin2016movement}. However, unsupervised HMS faces the challenge of motion primitive ambiguity due to temporal variability across different actions \cite{tierney2014subspace,li2015temporal,wang2018low,zhou2022consistency,bai2022human}.

Subspace clustering is a well-established strategy for HMS, aiming to partition a human motion sequence into distinct groups based on the assumption that the frames originate from multiple subspaces, with frames depicting the same motion belonging to the same subspace \cite{cao2015constrained, ng2002spectral, rahmani2017innovation,xing2025hmm,xing2023clustering,XingChen:J23ar}. In recent years, a popular and effective approach is to first embed the human motion frames into multiple subspaces to learn the subspace structure of the data, and then apply traditional clustering algorithms to the subspace embeddings \cite{wang2023self,wang2023bi,chen2023fast,tang2023affine,chen2024double,Trajmat:J25,Xing:C22,xing2023blockdiagonal,XinChe:C22}.

Human motion sequences inherently contain temporal information, which is crucial for HMS. As illustrated in Figure \ref{fig:background}, a human motion sequence typically consists of multiple motion segments. For example, a person may first perform a motion with both hands clasped for a period, then extend their hands for another period, followed by squatting. However, due to the ambiguity of different motions and the complexity of temporal correlations, extracting temporal information from human motion sequences remains a significant challenge \cite{chen2024double}.

Various subspace clustering algorithms have been proposed to achieve HMS by exploring the temporal information in the data. For instance, Wang et al. \cite{wang2022support} eliminate redundant connections between adjacent motions in the subspace embedding to extract informative instance data and capture the compact structure of human motion videos. Bai et al. \cite{bai2022human} extract informative features during subspace embedding to capture local temporal consistency in human motions. Zhou et al. \cite{zhou2022consistency} employ a multi-mutual consistency learning strategy to factorize source and target data into distinct multi-layer feature spaces, thereby learning the temporal information from the source domain. Despite efforts to explore and utilize temporal information, its inaccuracy and ambiguity may lead to incorrect segmentation results. While neural network algorithms have shown great promise in supervised tasks, they often fall short in unsupervised tasks due to the difficulty in exploring the underlying data structures.

This paper aims to learn temporal vision semantics (TVS) from human motion sequences by leveraging the image-to-text capabilities of a pre-trained large language model (LLM) to enhance HMS performance. The key idea is to extract textual motion information from consecutive frames using the LLM and integrate this learned information into the subspace clustering framework. By incorporating TVS into the subspace clustering, we aim to ensure that the segmentation output more effectively captures the temporal dynamics inherent in the human motion sequence.

Thus, we face the following \textbf{challenges}:
\begin{itemize}
	\item \textit{How to learn TVM using LLM?} 
	Although LLMs are widely used in image-to-text tasks, no research has explored how to leverage LLMs to assist in unsupervised HMS. The challenge lies in learning the textual temporal information that can be converted into a mathematical form, which can be used to induce the HMS.
	
	\item \textit{How to integrate TVM with HMS?} 
	We aim to learn the subspace embedding and perform segmentation based on this embedding, all induced by TVM. However, integrating TVM into both the subspace embedding and segmentation presents significant challenges.
\end{itemize}

In this study, we determine whether consecutive frames represent the same motion by querying the LLM. Based on its response, we subsequently learn the temporal relationships between frames, as illustrated in Figure~\ref{fig:background}. Building upon this, we propose a subspace clustering approach integrated with TVS, which combines subspace embedding with a temporal regularizer. This regularizer ensures that each frame shares similar subspace embeddings with its temporal neighbors. Segmentation is then performed using these subspace embeddings, with a temporal constraint that encourages the grouping of each frame with its temporal neighbors. Furthermore, we introduce a feedback-enabled framework that iteratively optimizes the subspace embeddings based on the segmentation output, ensuring continuous refinement of the model.

In summary, the main \textbf{contributions} are as follows:
\begin{itemize}{\setlength{\parsep}{-0.25ex}}
	\item \textit{Exploring TVS via LLMs and Integrating TVS with HMS}: 
	This paper introduces an approach that uses LLMs to learn TVS in human motion sequences. We develop a method that applies TVS to both subspace embedding and segmentation, ensuring neighboring consistency in both processes.
	
	\item \textit{Feedback-enabled Subspace Embedding}: 
	We propose a feedback-enabled strategy that allows the segmentation output to inform the subspace embedding. This is not merely a combination of two methods; rather, it enables the use of HMS output to induce subspace embedding, better capturing the underlying subspace structure in the human motion sequence.
\end{itemize}
We conduct extensive experiments on four benchmark datasets for HMS. The experimental results consistently demonstrate that our method outperforms existing state-of-the-art techniques, highlighting its superiority in HMS.

The remainder of this paper is structured as follows. Section~\ref{sec:rw} briefly introduces the related work including HMS and subspace clustering. Section~\ref{sec:Methodology} presents the details of the proposed method. Section~\ref{sec:Experiments} provides the experimental settings, performance comparisons, and ablation study. Finally, Section~\ref{sec:conclusion} concludes the paper.

	\section{Related Works}
\label{sec:rw}

\subsection{Human Motion Segmentation}

HMS is essential for accurately capturing human motion data, forming a basis for structural analysis, understanding, and practical applications \cite{aggarwal1999human,lan2015automated, wang2003recent, schulz2010automatic, li2018human}. Significant research efforts have led to notable achievements in this area. For instance, Zhong et al. \cite{zhong2004detecting} proposed a bipartite graph co-clustering framework to segment unusual activities in videos. Jenkins et al. \cite{fod2002automated} utilized zero-velocity crossing frames of angular velocity to partition motion data streams into different sequences. Barbic et al. \cite{barbivc2004segmenting} employed probabilistic principal component analysis to decompose human motion into distinct motions. Additionally, Beaudoin et al. \cite{beaudoin2008motion} introduced a framework for distilling a motion-motif graph from motion data collections. Spatio-temporal-based Convolutional Neural Networks (CNNs) \cite{lea2016segmental} and clustering-based approaches \cite{de2007temporal} have been proposed for segmenting streams of human motion into multiple activities. 
	Despite the capability of deep learning-based motion recognition models to complete HMS tasks by training with large datasets, the unsupervised model offers significant advantages in terms of interpretability and computational efficiency. Therefore, achieving HMS tasks through an unsupervised approach is highly beneficial.\cite{ gao2022human, jiang2016human, liu2017sensor, lin2014human,ZheHMS:J24,XinChe:C24,xing2024calibration,ZhaoW:C21}

However, these approaches may not fully exploit the temporal dynamics and semantic continuity inherent in human motion sequences, potentially limiting the accuracy and effectiveness of the segmentation results.

	\subsection{Subspace Clustering}
Subspace clustering discerns and segregates distinct motion types into their respective subspaces, thereby addressing human motion segmentation tasks \cite{lin2013online, hoai2011joint, guo1994understanding, gehrig2010towards}. Its capacity to manage intricate, high-dimensional motion data, combined with robustness to noise and data variability, ensures reliability in practical applications. Furthermore, by exploiting the inherent low-dimensional structures within complex motion datasets, subspace clustering facilitates the further analysis and utilization of human motion data \cite{bradski2002motion, gong2013structured, dimiccoli2021graph, shao2012human, kahol2003gesture,lin2016human,lu2004repetitive,ZheChe:J25,ZhaoW:C23,XinChe:C25,xing2024block,Xing:C25}.
	
Subspace clustering serves to ascertain the low-dimensional embedding of a high-dimensional manifold. Specifically, assuming that vectorized frames corresponding to identical actions reside within the same subspace, the subspace embedding property inherent to the data can be harnessed to derive a representative subspace embedding~\cite{zhou2022consistency,liu2013robust,tierney2014subspace,qin2022maximum,bai2022human,wang2018learning,wang2022support}. Consider, in particular, a matrix composed of column-wise frames $\mathbf{X} = [\mathbf{x}_{1}, \mathbf{x}_{2}, \dots, \mathbf{x}_{N}] \in \mathbbm{R}^{D \times N}$. Each vectorized frame may be expressed as a linear combination of all frames, that is, $\mathbf{x}_{i} = \mathbf{X} \mathbf{z}_{i}, \quad \text{where} \quad \mathbf{z}_{i} \in \mathbb{R}^{N} \text{ is the subspace embedding of } \mathbf{x}_{i}$. Hence, we may write $\mathbf{X} = \mathbf{X} \mathbf{Z}$, where $\mathbf{Z} = [\mathbf{z}_{1}, \mathbf{z}_{2}, \dots, \mathbf{z}_{N}] \in \mathbbm{R}^{N \times N}$ constitutes the coefficient matrix. To accommodate noise, this relationship is extended to $\mathbf{X} = \mathbf{X} \mathbf{Z} + \mathbf{Y}$, where $\mathbf{Y}$ represents the noise component. The Frobenius norm of $\mathbf{Y}$ is employed as the loss function to penalize discrepancies. Consequently, the subspace clustering problem is formulated as the minimization of $\|\mathbf{X} - \mathbf{X} \mathbf{Z}\|_{\text{F}}^{2}$ with respect to $\mathbf{Z}$. Mathematically, as delineated in~\cite{xu2015reweighted}, the formulation is expressed as minimizing $\|\mathbf{X} - \mathbf{X} \mathbf{Z}\|_{\text{F}}^{2}$ subject to the constraints $\mathrm{diag}(\mathbf{Z}) = 0$ and $\mathbf{Z} \geq 0$, where $\|\cdot\|_{\text{F}}^{2}$ denotes the Frobenius norm. Herein, $\mathbf{Z}$ embodies the sought subspace embedding.

To eliminate the effect of the noise on subspace embedding, a quadratic term $\|\mathbf{Z}^{\mathrm{T}} \mathbf{Z}\|_{1}$ is employed~\cite{wang2011efficient}. This term promotes the expression of each datum as a linear combination of other data points, whilst a regularization term enforces the nullification of reconstruction coefficients between vectors originating from distinct subspaces. Specifically, the regularization term $\|\mathbf{Z}^{\mathrm{T}} \mathbf{Z}\|_{1}$ compels $\mathbf{z}_{ij} = 0$ whenever $\mathbf{x}_{i}$ and $\mathbf{x}_{j}$ reside in different subspaces. Given that $\mathbf{Z}$ is nonnegative and that $\|\mathbf{Z}^{\mathrm{T}} \mathbf{Z}\|_{1} = \mathbf{e}^{\mathrm{T}} \mathbf{Z}^{\mathrm{T}} \mathbf{Z} \mathbf{e}$, where $\mathbf{e}$ denotes the all-one vector, it follows that $\|\mathbf{Z}^{\mathrm{T}} \mathbf{Z}\|_{1} = \sum_{i,j} \mathbf{z}_{i}^{\mathrm{T}} \mathbf{z}_{j}$. The minimization of $\|\mathbf{Z}^{\mathrm{T}} \mathbf{Z}\|_{1}$ thereby encourages sparsity in both $\mathbf{z}_{i}$ and $\mathbf{z}_{j}$, leading the inner product $\mathbf{z}_{i}^{\mathrm{T}} \mathbf{z}_{j}$ to approach zero. Consequently, the subspace embedding problem is thus formulated as~\cite{xu2015reweighted}:

\begin{equation}
	\begin{aligned}
		\label{Prob:SC}
		&\underset{\mathbf{Z}}{\text{minimize}} \: \|\mathbf{X} - \mathbf{X} \mathbf{Z}\|_{\text{F}}^{2} + \|\mathbf{Z}^{\mathrm{T}} \mathbf{Z}\|_{1} \\
		&\text{subject to} \: \mathrm{diag}(\mathbf{Z}) = 0, \mathbf{Z} \geq 0
	\end{aligned}
\end{equation}
Traditional clustering methods are then applied to the subspace embedding matrix \( \mathbf{Z} \), and the resulting clusters represent the HMS, which can be matched to specific motions using the Hungarian algorithm.

Subspace clustering has recently gained significant attention due to its effectiveness in uncovering complex data structures and improving clustering performance in high-dimensional spaces \cite{xing2025k,xing2025blind,xing2024unsupervised}. For instance, the \textit{SIBMSC} method \cite{wang2023self} extends the information bottleneck principle to learn view-common representations, removing redundant information and leveraging mutual information for view-specific clustering. Similarly, the \textit{BTMSC} method \cite{wang2023bi} constructs a third-order tensor to capture high-order correlations, using the Bi-Nuclear Quasi-Norm for efficient tensor factorization. To improve robustness, the \textit{FSMSC} method \cite{chen2023fast} integrates view-shared anchor learning with a self-guided discriminative feature selection approach, addressing noisy views and cross-view diversity. The \textit{ARLRR} method \cite{tang2023affine} introduces affine and nonnegative constraints in low-rank self-representation learning to manage affine subspaces and errors. The \textit{DCTMSC} method \cite{chen2024double} employs a two-step discrete cosine transform approach to simplify tensor nuclear norm calculations and enhance local structural representation. The \textit{DCMVC} method \cite{cui2024dual} incorporates dynamic cluster diffusion and reliable neighbor-guided positive alignment to improve inter-cluster separation and within-cluster compactness. Subspace clustering methods incorporating temporal priors have proven effective in HMS tasks. For instance, the \textit{OSC} method \cite{tierney2014subspace} applies a one-neighbor consistency constraint for closer representations of temporal data, while the \textit{TSC} method \cite{li2015temporal} uses non-negative dictionary learning and temporal Laplacian regularization. The \textit{LTS} method \cite{wang2018low} captures temporal correlations in both source and target data with a graph regularizer and introduces a weighted low-rank constraint to reveal clustering structures. The \textit{CDMS} approach \cite{zhou2022consistency} leverages transfer subspace learning to capture multi-level information in videos. These methods, often formulated as unsupervised learning frameworks, typically adopt a self-representation strategy for motion segmentation. The \textit{DSAE} method \cite{bai2020dual} enhances representation learning by considering temporal correlations, while the \textit{VSDA} method \cite{bai2022human} employs a multi-neighbor auto-encoder to extract temporal features and a long-short distance enembedding/deembedding strategy to maintain representation consistency, further enhanced by a velocity-sensitive guidance mechanism.

However, existing subspace clustering approaches typically divide the process into two independent stages and often overlook the potential of incorporating temporal semantics to simultaneously enhance both subspace embedding and clustering. This oversight limits the alignment of the HMS output with the true sequential dynamics of human motion.

{\color{black}
\subsection{Temporal Vision Semantics from Large Language Models}
\label{sec:related_llm}

LLMs have progressed from text-only reasoning engines to unified multimodal systems capable of jointly understanding visual and linguistic information. 
Recent architectures such as GPT-4o, Gemini, Claude, DeepSeek, and Qwen3 integrate visual encoders with transformer-based text reasoning via large-scale contrastive pretraining, enabling them to perform complex cross-modal reasoning and semantic alignment between images and language. 
Unlike conventional convolutional or transformer-based visual encoders that rely on geometric or pixel-level similarity, multimodal LLMs exhibit \emph{semantic reasoning capability}. They can compare two visual scenes and judge whether they convey the same conceptual meaning based on high-level world knowledge and contextual understanding.

Extensive research in vision–language modeling has demonstrated the strong semantic reasoning capabilities of LLMs when integrated with visual inputs. Early studies validated these capabilities in tasks such as zero-shot visual question answering and high-fidelity caption generation, where LLMs interpret visual entities, relationships, and contextual meanings directly from raw imagery~\cite{guo2023images,ge2024visual}. Building upon these foundations, subsequent works extended vision semantics to three-dimensional and dynamic scenes, leveraging language-guided scene understanding and position-aware video representations for 3D perception~\cite{zhi2025lscenellm,zheng2025video}. Beyond visual applications, LLMs have also exhibited robust zero-shot reasoning and representational alignment abilities that enable general semantic understanding across modalities~\cite{kojima2022large,NEURIPS2024_0b77d3a8}.

Building upon this progress, a growing body of research has explored the integration of LLM-based visual semantics into temporal vision understanding. Recent work has examined whether video-oriented LLMs truly capture temporal reasoning or merely rely on knowledge and spatial perception~\cite{feng2025breaking}, while subsequent studies have demonstrated that LLMs can effectively learn temporal dependencies and causal relations across video frames~\cite{liu2024st}. Further developments employ language-guided attention mechanisms to align visual dynamics with textual motion cues, thereby enhancing spatial–temporal object understanding and fine-grained temporal reasoning~\cite{yuan2025videorefer,nie2024slowfocus}. Beyond short-term video grounding, recent efforts extend this semantic alignment to long-term sequence modeling, revealing that LLMs can encode cross-frame dependencies with human-level temporal abstraction~\cite{ding2025language,deng2025seq2time}.

In the domain of human motion analysis, LLMs have been increasingly adopted to reason about human activities and their semantic transitions~\cite{chen2024motionllm,li2025human,li2024sensorllm,wang2024scaling}. These studies demonstrate that LLMs can interpret and describe motion in natural language, distinguish subtle phase changes, and correlate sensor or visual signals with linguistic motion descriptions. Recent studies have advanced from graph-based relational modeling to LLM-driven semantic reasoning in motion understanding. For example, some works employ LLMs to anticipate long-term actions by treating video frames as language-like tokens and enhancing vision–language interaction through cross-modal reasoning~\cite{wang2025multimodal}, while others utilize graph attention mechanisms to capture individual–group interaction dynamics in collective activities~\cite{lu2019gaim}. 
Such findings inspire the present work, where we employ LLM-based reasoning to construct temporal semantics, serving as a high-level inductive prior for unsupervised human motion segmentation.

}

		\section{Methodology}
	\label{sec:Methodology}
In this section, we first introduce an LLM-based inference to identify the TVS. We then propose a feedback-enabled subspace embedding approach that incorporates TVS to efficiently determine the HMS with limited iterations.

{\color{black}
\subsection{LLM-driven Temporal Semantics Inference}
\label{sec:tvs_llm}

Human behavior unfolds as a continuous visual process, where adjacent frames in a motion sequence often exhibit high semantic correlation. To identify segments representing the same human action, we introduce the concept of TVS, which delineates the temporal neighborhood of each frame according to semantic consistency. Specifically, for a given frame $\mathbf{x}_i$, we aim to discover its left and right temporal neighbor bounds $(l_i, r_i)$ that enclose all frames depicting the same motion as $\mathbf{x}_i$.

Unsupervised discovery of such temporal neighborhoods is challenging using traditional machine learning methods, as the semantic boundary between motions is difficult to define purely through pixel comparisons. To address this challenge, we harness the visual reasoning capability of a LLM as a zero-shot semantic comparator. Instead of training an additional network, the LLM is instructed with a natural-language prompt to assess whether two consecutive frames represent the same human motion:
\begin{quote}
	\textit{“Do these two neighboring frames depict the same human motion?  Answer Yes or No.”}
\end{quote}
Given two adjacent frames $(\mathbf{x}_t, \mathbf{x}_{t+1})$, the LLM produces a binary response $\texttt{Yes}/\texttt{No}$, which is recorded as a Boolean variable $\text{eq}_{t} \in \{0,1\}$ indicating whether the two frames belong to the same motion segment. The sequence $\{\text{eq}_{1}, \dots, \text{eq}_{N-1}\}$ forms the adjacency pattern of temporal consistency across the sequence.
In our implementation, this is achieved through an API call to a multimodal LLM (e.g., GPT-4o, Gemini-2.0, or Claude-4.5), where both frames are provided in base64-encoded format, allowing the model to reason directly over image content.

\begin{algorithm}[t] 
	\color{black} 
	\caption{Learning TVS via LLM}
	\label{alg:tvs_construction}
	\textbf{Input}: Raw RGB frames sequence $\mathbf{X}=\{\mathbf{x}_1,\dots,\mathbf{x}_N\}$.~\\
	\textbf{Output}: TVS matrix $\mathbf{G}\in\mathbb{R}^{N\times N}$ and the set $\{\mathcal{N}_i\}_{i=1}^N$\\
	\begin{algorithmic}[1] 
		\STATE Initialize $\text{eq}\leftarrow\varnothing$; 
		\FOR{$i=1$ {\bfseries to} $N-1$}
		\STATE Encode images $\mathbf{x}_i,\mathbf{x}_{i+1}$ to base64 formats separately;
		\STATE Query LLM with prompt: \textit{``Do these two neighbouring frames depict the same human motion? Answer Yes or No.''};
		\STATE Parse response to $\text{eq}_i\in\{0,1\}$ (\texttt{Yes}$\!\to\!1$, \texttt{No}$\!\to\!0$). 
		\STATE \textbf{if} response ambiguous \textbf{then} re-query with stricter instruction (\textit{``Answer strictly with a single token: YES or NO.''}) and re-parse.
		\STATE Append $\text{eq}_i$ to list $\text{eq}$.
		\ENDFOR
		\FOR{$i=1$ {\bfseries to} $N$}
		\STATE Initialize $l_i\leftarrow i$; $r_i\leftarrow i$.
		\STATE  \textbf{while} $l_i>1$ {\bfseries and} $\text{eq}_{l_i-1}=1$ \textbf{do} $l_i\leftarrow l_i-1$.
		\STATE  \textbf{while} $r_i<N$ {\bfseries and} $\text{eq}_{r_i}=1$ \textbf{do} $r_i\leftarrow r_i+1$.
		\ENDFOR
		\STATE Initialize $\mathbf{G}\leftarrow \mathbf{0}_{N\times N}$.
		\FOR{$i=1$ {\bfseries to} $N$}
		\STATE $\mathcal{N}_i\leftarrow \{j\mid j\in[l_i,r_i],~j\neq i\}$;
		\STATE $G_{ii}\leftarrow -|\mathcal{N}_i|$;  \STATE \textbf{for each} $j\in\mathcal{N}_i$ \textbf{do} 
		\STATE \quad $G_{ij}\leftarrow 1$.
		\ENDFOR
	\end{algorithmic}
\end{algorithm}

Using the response $\{\text{eq}_{1}, \dots, \text{eq}_{N-1}\}$, we define the left and right temporal neighbor bounds for each frame $\mathbf{x}_i$ as follows:
\begin{align*}
	&l_i = \min \{ j \mid j \leq i, \; \mathbf{x}_{j},\mathbf{x}_{j+1},...,\mathbf{x}_{i} \text{ describe the same motion} \},\\
	&r_i = \max \{ j \mid j \geq i, \; \mathbf{x}_{i},\mathbf{x}_{i+1},...,\mathbf{x}_{j} \text{ describe the same motion} \}.
\end{align*}
Accordingly, the temporal neighborhood set $\mathcal{N}_i$ of $\mathbf{x}_i$ is given by
\[
\mathcal{N}_i = \{ j \mid j \in \{l_i, l_i+1, \dots, r_i\}, \, j \neq i, j\in\mathbb{Z}^{+} \}.
\]
Each frame $\mathbf{x}_i$ is thus associated with a temporally and semantically coherent segment. 

This structure serves as a foundation for constructing a TVS matrix $\mathbf{G} \in \mathbb{R}^{N \times N}$, whose entries encode the neighborhood connectivity as
\[
G_{i,j} = 
\begin{cases}
	-|\mathcal{N}_i|, & \text{if } i=j, \\
	1, & \text{if } j \in \mathcal{N}_i, \\
	0, & \text{otherwise}.
\end{cases}
\]
In practice, the TVS matrix is implemented as a Laplacian-like structure, where the diagonal term penalizes the number of semantic neighbors, and the off-diagonal entries reflect temporal affinity.

Algorithm~\ref{alg:tvs_construction}\footnote{https://github.com/y66y/TVSH} summarizes the complete computational procedure. For each pair of adjacent frames $(i, i+1)$, the LLM is queried once and the response recorded. Subsequently, left and right neighbor bounds $(l_i, r_i)$ are determined through a recursive traversal of the Boolean adjacency list. Finally, the TVS matrix $\mathbf{G}$ is constructed and saved for downstream processing.

\begin{remark}
	While the TVS introduces human-like temporal reasoning into motion segmentation, it only captures pairwise relationships between temporal consecutive frames, lacking global temporal dependencies. Therefore, a subsequent grouping stage is required to achieve globally consistent motion segmentation.
\end{remark}

}

	\subsection{Subspace embedding and Clustering Incorporating Vision Temporal Semantics}

	By leveraging matrix multiplication, we observe that the product \( \mathbf{Z}\mathbf{G} \) captures the similarity error between the representation of a given sequential point and its neighbors. Specifically, the term
	\[
	\mathbf{Z}\mathbf{G} = \left[\sum_{l \in \mathcal{N}_1} (\mathbf{z}_1 - \mathbf{z}_l), \sum_{l \in \mathcal{N}_2} (\mathbf{z}_2 - \mathbf{z}_l), \dots, \sum_{l \in \mathcal{N}_N} (\mathbf{z}_N - \mathbf{z}_l)\right].
	\]
	measures the similarity of the $i$th data and its neighbors defined by $\mathcal{N}_i$.
	To encourage the subspace embedding of the subspace embedding and the embedding of its neighbors to be as similar as possible, we introduce a structural regularization term \( \|\mathbf{Z}\mathbf{G}\|_{2,1} \), where \( \|\cdot\|_{2,1} \) denotes the \( l_1 \) norm of the vector formed by the \( l_2 \) norms of each column of the matrix. This norm encourages the columns of \( \mathbf{Z}\mathbf{G} \) to exhibit consistent behavior across neighboring points, promoting smoothness in the subspace representation. Mathematically, we express it as
	\[
	\|\mathbf{Z}\mathbf{G}\|_{2,1} = \sum_{i=1}^{N} \left\|\sum_{l \in \mathcal{N}_i} (\mathbf{z}_i - \mathbf{z}_l)\right\|_2 = \sum_{i=1}^{N} \sum_{l \in \mathcal{N}_i} \|\mathbf{z}_i - \mathbf{z}_l\|_2.
	\]
	We aim to minimize the subspace embedding error \( \|\mathbf{Z}\mathbf{G}\|_{2,1} \) to enhance the temporal consistency of the subspace embedding by promoting coherence between a subspace embedding and its neighboring subspace embeddings, which is crucial for capturing the dynamic temporal structure of the data.
	
{\color{black}
\begin{theorem}[Interpretation of the TVS Regularizer]
	\label{thm:graph_tv_equiv}
	The regularizer \( \|\mathbf{Z}\mathbf{G}\|_{2,1} \) represents an \emph{isotropic graph total variation} over the temporal graph defined by $\{\mathcal{N}_i\}$. Minimizing it enforces local smoothness within temporal neighborhoods while preserving discontinuities at motion boundaries, thus producing piecewise-constant embeddings consistent with human motion transitions.
\end{theorem}
	\begin{proof}
	See Appendix \ref{sec:proof_graph_tv}.
\end{proof}
}

{\color{black}
	\begin{theorem}[Consistency under Noisy LLM Adjacency]
		\label{thm:consistency_llm}
		If each LLM adjacency label is independently flipped with probability $p<\tfrac{1}{2}$ and each motion segment has length at least $L_{\min}$,  
		then the expected number of erroneous TVS boundaries scales as $O(pN)$.  
		Minimizing $\|\mathbf{Z}\mathbf{G}\|_{2,1}$ yields piecewise-constant embeddings that smooth out isolated errors,  
		ensuring segment-level consistency in expectation when $p$ is small and $L_{\min}$ is sufficiently large.
	\end{theorem}
		\begin{proof}
		See Appendix \ref{sec:proof_consistency_llm}.
	\end{proof}

	Theorem \ref{thm:consistency_llm} implies that the proposed framework is robust to occasional LLM misjudgments:  
	although local adjacency errors may occur, the TVS-induced regularizer preserves overall temporal coherence by enforcing smooth embeddings within segments.  
	Thus, the method maintains consistent human motion segmentation under moderate annotation noise.
	
}

We also propose a TVS-integrated segmentation on the subspace embedding. Specifically, suppose the number of clusters is \( K \). We introduce a cluster assignment indicator vector \( \mathbf{q}_{i} \in \mathbb{R}^{K} \) for the \( i \)-th frame, where the \( k \)-th element is set to 1 if the \( i \)-th frame is assigned to the \( k \)-th cluster, and all other elements are set to zero. We then define a clustering regularizer based on the subspace embedding \( \mathbf{Z} \) and the indicator matrix \( \mathbf{Q} = [\mathbf{q}_{1}, \mathbf{q}_{2}, \dots, \mathbf{q}_{N}] \in \mathbb{R}^{N \times K} \):
		\begin{align*}
			\|\mathbf{Z}\|_{\mathbf{Q}} &= \frac{1}{2} \sum_{i,j} \frac{|Z_{i,j}| + |Z_{j,i}|}{2} \|\mathbf{q}_{i} - \mathbf{q}_{j}\|_{2}^{2}\\
			&=\frac{1}{2}\sum_{i,j}|Z_{i,j}|\cdot\|\mathbf{q}_{i}-\mathbf{q}_{j}\|_{2}^{2}
			=\sum_{i,j}|Z_{i,j}\cdot\frac{\|\mathbf{q}_{i}-\mathbf{q}_{j}\|_{2}^{2}}{2}|\\
			&=\sum_{i,j}|Z_{i,j}\cdot\Theta_{i,j}|
			=\sum_{i,j}|(\bm{\Theta}\odot\mathbf{Z})_{ij}|
			=\|\bm{\Theta}\odot\mathbf{Z}\|_{1}
		\end{align*}
where \( Z_{i,j} \) is the \( (i,j) \)-th element of \( \mathbf{Z} \), and \( \Theta_{i,j} = \frac{\|\mathbf{q}_{i} - \mathbf{q}_{j}\|_2^2}{2} \). The first equation ensures symmetry by combining both \( |Z_{i,j}| \) and \( |Z_{j,i}| \), accounting for the interactions between off-diagonal terms, as \( Z_{i,j} \neq Z_{j,i} \) does not necessarily hold, thus incorporating these contributions into the final regularizer. The term \( \Theta_{i,j} \) measures the squared Euclidean distance between the cluster indicators \( \mathbf{q}_i \) and \( \mathbf{q}_j \), normalized by a factor of \( \frac{1}{2} \), capturing the dissimilarity between frames based on their clustering assignments and enforcing smoothness within clusters. The final expression \( \|\bm{\Theta} \odot \mathbf{Z}\|_1 \) represents the \( l_1 \)-norm of the element-wise product between \( \bm{\Theta} \) and \( \mathbf{Z} \), which encourages a sparse representation of the subspace embedding while optimizing the clustering assignments, ensuring that frames assigned to the same cluster exhibit more similar representations.
The term \( \|\mathbf{Z}\|_{\mathbf{Q}} \) will be minimized to optimize the clustering assignment.

	To ensure that \( \mathbf{q}_{i} \) functions effectively as a cluster assignment indicator, we impose the condition that \( \mathbf{Q} \) is a subset of
	\begin{align*}
		\mathcal{Q} = \left\{ \mathbf{Q} \in \{0,1\}^{N \times K} : \mathbf{Q} \mathbf{1}_{K \times 1} = \mathbf{1}_{N \times 1}, \mathbf{q}_{i} = \mathbf{q}_{j} \, \forall j \in \mathcal{N}_{i} \right\}.
	\end{align*}
This constraint ensures that the clustering assignment for the \( i \)-th frame and its neighbors are identical, i.e., \( \mathbf{q}_{i} = \mathbf{q}_{j} \, \forall j \in \mathcal{N}_{i} \), which enforces temporal consistency within cluster assignments.

	Building on the traditional subspace clustering formulation in \eqref{Prob:SC}, we develop a feedback-enabled framework that integrates the proposed subspace embedding, which incorporates temporal vision semantics, with the proposed clustering method. The optimization problem is formulated as follows:
		\begin{equation}
			\begin{aligned}
				\label{Prob:TVSH}
				&\underset{\mathbf{Z},\mathbf{Q}}{\text{minimize}} \:\|\mathbf{X}-\mathbf{X}\mathbf{Z}\|_{\text{F}}^{2}+\|\mathbf{Z}^{\mathrm{T}}\mathbf{Z}\|_{1}+\|\mathbf{Z}\mathbf{G}\|_{2,1}+\|\mathbf{Z}\|_{\mathbf{Q}}
				\\
				&\text{subject to} \:\mathrm{diag}(\mathbf{Z})=0,\mathbf{Z}\geq0,\mathbf{Q}\in\mathcal{Q}.
			\end{aligned}
		\end{equation}

		\begin{proposition}
			\label{prop:BDconstraint} 
		Let \( \mathbf{X} \in \mathbb{R}^{D \times N} \) be a matrix whose columns are drawn from a union of \( K \) distinct subspaces, with the subspace assignment indicated by \( \mathbf{Q}^* \). The optimal solution to the problem in \eqref{Prob:TVSH} is given by \( \mathbf{Q}^* \) and \( \mathbf{Z}^* \), where \( \mathbf{Z}^* \) is block-diagonal after permuted according to \( \mathbf{Q}^* \).
		
		\end{proposition}
		\begin{proof}
			See Appendix \ref{app:prop-BDconstraint}.
		\end{proof}

		Proposition~\ref{prop:BDconstraint} demonstrates that, under the assumption that frames are distributed across distinct subspaces, the optimal solution to problem (\ref{Prob:TVSH}) will align with the true segmentation. However, due to the influence of noise, the human motion data may not lie perfectly within the subspaces.
		Thus, there are inherent trade-offs in (\ref{Prob:TVSH}) due to practical dataset challenges such as image noise and subtle motions, which may cause frames not to align precisely with $K$ subspaces. Specifically, the term $\|\mathbf{X}-\mathbf{X}\mathbf{Z}\|_{\text{F}}^{2} + \|\mathbf{Z}^{\mathrm{T}}\mathbf{Z}\|_{1}$ ensures sparsity and accurate data subspace embedding in the outputs $\mathbf{z}_{1}, \mathbf{z}_{2}, ..., \mathbf{z}_{N}$, promoting distinctiveness among them. In contrast, the term $\|\mathbf{Z}\mathbf{G}\|_{2,1}$ requires these outputs to align with their neighbors' coefficients $\{\mathbf{z}_{l}\}_{l\in\mathcal{N}_{i}}$.
		This necessitates a balance between representing data across $K$ clusters and maintaining temporal, aiming to segment the data sequence into smaller segments where, for instance, in a segment $[\mathbf{x}_{i}, \mathbf{x}_{i+1}, ..., \mathbf{x}_{j}]$ with $j>i$, the subspace embedding are identical, i.e., $\mathbf{z}_{i} = \mathbf{z}_{i+1} = ... = \mathbf{z}_{j}$.
		Additionally, the term $\|\mathbf{Z}\|_{\mathbf{Q}}$ promotes effective grouping based on $\mathbf{Z}$ while adhering to the constraint $\mathbf{Q} \in \mathcal{Q}$, which stipulates that the cluster assignment of the $i$th frame must match that of its neighbors, introducing a further trade-off between dependent clustering and TVS considerations.

			{\color{black}
	\begin{theorem}[Impact of TVS on Segmentation]
		\label{thm:tvs_influence_simplified}
		Assume each motion segment generates data lying in one of $K$ linear subspaces with within-segment variance $\sigma^2$ and between-subspace separation $\Delta_{\text{sub}}^2>0$.  
		With independent LLM adjacency errors of rate $p<\tfrac{1}{2}$,  
		the expected segmentation error satisfies $\mathbb{E}[\mathrm{Err}_{\text{HMS}}] \le C_1\frac{p}{L_{\min}} + C_2\frac{\sigma^2}{\Delta_{\text{sub}}^2}$.
		When TVS boundaries align with true actions, the optimal solution \( \mathbf{Z}^* \) becomes block-diagonal, achieving exact segmentation.  
	\end{theorem}
			\begin{proof}
		See Appendix \ref{sec:proof_tvs_influence}.
	\end{proof}

	Theorem \ref{thm:tvs_influence_simplified} establishes that TVS improves segmentation robustness by suppressing random adjacency errors and stabilizing intra-segment embeddings.  
	The first term $C_1\frac{p}{L_{\min}}$ reflects the resilience to LLM-induced boundary noise, which diminishes as segment length increases, while the second term $C_2\frac{\sigma^2}{\Delta_{\text{sub}}^2}$ captures the dependence on subspace separability.  
	Perfectly aligned TVS boundaries yield theoretically exact segmentation, confirming the effectiveness of LLM-guided temporal reasoning in enhancing motion boundary localization.
	
	}

		\subsection{A Feedback-Enabled Optimization Algorithm}
		We employ the ADMM method \cite{boyd2011distributed} to solve the optimization problem formulated in \eqref{Prob:TVSH}. In order to separate the third term in \eqref{Prob:TVSH} from the other three terms, we introduce an additional variable $\mathbf{H}=\mathbf{Z}\mathbf{G}$. By incorporating an augmented Lagrangian multiplier to handle the introduced linear constraint, we can reformulate \eqref{Prob:TVSH} as the following problem:
		\begin{align}
			&\underset{\mathbf{Z},\mathbf{H},\mathbf{Q}}{\text{minimize}}  \quad\|\mathbf{X}-\mathbf{X}\mathbf{Z}\|_{\text{F}}^{2}+\|\mathbf{Z}^{\mathrm{T}}\mathbf{Z}\|_{1}+\|\mathbf{H}\|_{2,1}\nonumber\\
			&\quad\quad\quad+\langle\mathbf{F},\mathbf{H}-\mathbf{Z}\mathbf{G}\rangle+\frac{\gamma}{2}\|\mathbf{H}-\mathbf{Z}\mathbf{G}\|_{\text{F}}^{2}+\|\mathbf{Z}\|_{\mathbf{Q}}\label{Prob:NCGC_Lap}\\
			&\text{subject to}  \quad\mathrm{diag}(\mathbf{Z})=0,\mathbf{Z}\geq0,\mathbf{Q}\in\mathcal{Q}\nonumber
		\end{align}
		where $\mathbf{F}\in\mathbbm{R}^{N\times N}$ is the Lagrangian multiplier and $\gamma$ is
		an adaptive weight parameter for enforcing the condition $\mathbf{H}=\mathbf{Z}\mathbf{G}$.
		To solve \eqref{Prob:NCGC_Lap}, we adopt a feedback-enabled optimization strategy, where we iteratively solve three sub-problems for $\mathbf{Z}$, $\mathbf{H}$, and $\mathbf{Q}$ while keeping the other fixed, respectively.

		\subsubsection{$\mathbf{Z}$-solution} Fixing $\mathbf{H}$ and $\mathbf{Q}$, solve for $\mathbf{Z}$ by
		\begin{align}
			\underset{\mathbf{Z}}{\text{minimize}} & \quad\|\mathbf{X}-\mathbf{X}\mathbf{Z}\|_{\text{F}}^{2}+\|\mathbf{Z}^{\mathrm{T}}\mathbf{Z}\|_{1}+\langle\mathbf{F},\mathbf{H}-\mathbf{Z}\mathbf{G}\rangle\nonumber\\
			&\quad\quad\quad\quad\quad\quad\quad+\frac{\gamma}{2}\|\mathbf{H}-\mathbf{Z}\mathbf{G}\|_{\text{F}}^{2}+\|\mathbf{Z}\|_{\mathbf{Q}}\label{subProb:Z}\\
			\text{subject to} & \quad\mathrm{diag}(\mathbf{Z})=0,\mathbf{Z}\geq0\nonumber
		\end{align}
		Since $\mathbf{Z}$ consists of non-negative elements, we can rewrite the objective function in \eqref{subProb:Z} as a function:
		\begin{align}
			&\mathcal{J}(\mathbf{Z})=\|\mathbf{X}-\mathbf{X}\mathbf{Z}\|_{\text{F}}^{2}+\mathbf{e}^{\mathrm{T}}\mathbf{Z}^{\mathrm{T}}\mathbf{Z}\mathbf{e}+\langle\mathbf{F},\mathbf{H}-\mathbf{Z}\mathbf{G}\rangle\nonumber\\
			&\quad\quad\quad\quad\quad\quad\quad+\frac{\gamma}{2}\|\mathbf{H}-\mathbf{Z}\mathbf{G}\|_{\text{F}}^{2}+\|\bm{\Theta}\odot \mathbf{Z}\|_{1}	\label{eq:JZ}
		\end{align}
		
		The sub-problem defined in \eqref{subProb:Z} can be formulated as a convex quadratic programming problem with specific constraints for the variable $\mathbf{Z}$, involving the function $\mathcal{J}(\mathbf{Z})$ from \eqref{eq:JZ}. In this problem, we aim to minimize $\mathcal{J}(\mathbf{Z})$ while satisfying the given constraints. To tackle this, we employ the projected gradient method, which is a well-established approach known for its simplicity and effectiveness in solving such problems. This method is chosen as our preferred solution due to its suitability for our problem's requirements.

		\begin{algorithm}[t] 
			\caption{TVSH method.}
			\label{alg:TVSH}
			\textbf{Input}: $\mathbf{X}$.~\\
			\textbf{Output}: $\{\mathcal{C}_{1},\mathcal{C}_{2},...,\mathcal{C}_{K}\}$\\
			\begin{algorithmic}[1] 
				\STATE Initialize $\mathbf{G}$, $\mathbf{F}=\mathbf{1}$, $\rho=1.1, \gamma=0.1$. $\mathbf{H}=\mathbf{Z}\mathbf{G}$ where $\mathbf{Z}$ is the similarity matrix given by cosine measurement. $\mathbf{Q}$ is initialized by K-means \cite{ng2002spectral}.
				\REPEAT
				\STATE Find $\mathbf{Z}$ by solving \eqref{eq:d=0}.
				\STATE Calculate the projection $\mathbf{Z}\leftarrow \prod_{\mathcal{Z}}(\mathbf{Z})$ by solving \eqref{eq:PiZ}.
				\STATE Find $\mathbf{H}$ by solving \eqref{subProb:U};
				\STATE Update $\mathbf{F}\leftarrow\mathbf{F}+\gamma(\mathbf{H}-\mathbf{Z}\mathbf{G})$, $\gamma\leftarrow\rho\gamma$.
				\STATE Update $\mathbf{Q}$ by the following steps:
				\STATE Calculate  $\ensuremath{\mathbf{L}}=\mathbf{I}-\mathbf{D}^{-1/2}[(|\mathbf{Z}|+|\mathbf{Z}^{\mathrm{T}}|)/2]\mathbf{D}^{-1/2}$. \STATE Compute the smallest $K$ eigenvectors of $\mathbf{L}$ denoted by $\mathbf{V}=[\mathbf{v}_{1},\mathbf{v}_{2},...,\mathbf{v}_{K}]$. Denote the row  of $\mathbf{V}$ as $\{\mathbf{u}_{i}\}_{i=1}^{N}$. 
				\REPEAT
				\FOR{$k=1$ {\bfseries to} $K$}
				\STATE Form weight matrix $\mathbf{W}_{k}$ and calculate $\bm{{\mu}}_{k}$.
				\ENDFOR
				\FOR{$k=1$ {\bfseries to} $K$}
				\STATE $\mathcal{C}_{k}\leftarrow\{i\in\{1,2,...,N\}:k=\mathrm{arg\:min}_{k\in\{1,2,...,K\}}\:||\mathbf{u}_{i}-\bm{{\mu}}_{k}||_{2}^{2}w_{k,i}\}$
				\ENDFOR
					\UNTIL{$\mathbf{Q}$
						can not be changed.}
					\UNTIL{The objective function value of \eqref{Prob:NCGC_Lap} can not be decreased.}
				\end{algorithmic}
				
			\end{algorithm}

			Consider the partial derivative of $\|\bm{\Theta}\odot \mathbf{Z}\|_{1}$ with respect to each element \(Z_{ij}\):
			\[
			\frac{\partial}{\partial Z_{ij}} \left( \sum_{k,l} |\Theta_{kl} Z_{kl}| \right) = \frac{\partial}{\partial Z_{ij}} |\Theta_{ij} Z_{ij}|
			\]
			Using the properties of the absolute value function, we get:
			\[
			\frac{\partial}{\partial Z_{ij}} |\Theta_{ij} Z_{ij}| = \Theta_{ij} \cdot \text{sign}(\Theta_{ij} Z_{ij})
			\]
			where \(\text{sign}(x)\) is the sign function, defined as:
			\[
			\text{sign}(x) = \begin{cases}
				1, & \text{if } x > 0 \\
				0, & \text{if } x = 0 \\
				-1, & \text{if } x < 0
			\end{cases}
			\]
			Therefore, the partial derivative for each element \(Z_{ij}\) is:
			\[
			\frac{\partial \|\bm{\Theta} \odot \mathbf{Z}\|_{1}}{\partial Z_{ij}} = \Theta_{ij} \cdot \text{sign}(\Theta_{ij} Z_{ij})
			\]
			Combining all the partial derivatives into matrix form, the gradient of \(\|\bm{\Theta} \odot \mathbf{Z}\|_{1}\) with respect to \(\mathbf{Z}\) is:
			\[
			\frac{\partial \|\bm{\Theta} \odot \mathbf{Z}\|_{1}}{\partial \mathbf{Z}} = \bm{\Theta} \odot \text{sign}(\bm{\Theta} \odot \mathbf{Z})
			\]
			where \(\text{sign}(\bm{\Theta} \odot \mathbf{Z})\) is the matrix obtained by applying the sign function element-wise to \(\bm{\Theta} \odot \mathbf{Z}\).
			
			The derivative of $\mathcal{J}(\mathbf{Z})$ with respect to $\mathbf{Z}$ can be expressed as:
			$
			\partial\mathcal{J}(\mathbf{Z})=-2\mathbf{X}^{\mathrm{T}}(\mathbf{X}-\mathbf{X}\mathbf{Z})+2\mathbf{Z}\mathbf{E}-\mathbf{F}\mathbf{G}^{\mathrm{T}}-\gamma(\mathbf{H}-\mathbf{Z}\mathbf{G})\mathbf{G}^{\mathrm{T}}+\bm{\Theta}\odot\text{sign}(\bm{\Theta}\odot\mathbf{Z})
			$
			where $\mathbf{E}\in\mathbbm{R}^{N\times N}$ is an all-one matrix.

			Setting the derivative to zero gives
			\begin{align}
				&2\mathbf{X}^{\mathrm{T}}\mathbf{X}\mathbf{Z}+\mathbf{Z}(2\mathbf{E}+\gamma\mathbf{G}\mathbf{G}^{\mathrm{T}})\nonumber\\
				=&\mathbf{F}\mathbf{G}^{\mathrm{T}}+\gamma\mathbf{H}\mathbf{G}^{\mathrm{T}}+2\mathbf{X}^{T}\mathbf{X}-\bm{\Theta}\odot\text{sign}(\bm{\Theta}\odot\mathbf{Z}).
				\label{eq:d=0}
			\end{align}
			The equation presented is a well-known Sylvester equation in the form $\mathbf{A}\mathbf{Z}+\mathbf{Z}\mathbf{B}=\mathbf{C}$, where $\mathbf{A}=2\mathbf{X}^{\mathrm{T}}\mathbf{X}$, $\mathbf{B}=2\mathbf{E}+\gamma\mathbf{G}\mathbf{G}^{\mathrm{T}}$, and $\mathbf{C}=\mathbf{F}\mathbf{G}^{\mathrm{T}}+\gamma\mathbf{H}\mathbf{G}^{\mathrm{T}}+2\mathbf{X}^{T}\mathbf{X}-\bm{\Theta}\odot\text{sign}(\bm{\Theta}\odot\mathbf{Z})$. 
			
			We adopt Bartels-Stewart algorithm \cite{bartels1972solution} to solve $\mathbf{Z}$. Specifically, we first perform Schur decomposition on \( \mathbf{A} \) and \( \mathbf{B} \). The Schur decomposition of $\mathbf{A}$ and $\mathbf{B}$ is given by 
			$
			\mathbf{A} = \mathbf{Q}_A \mathbf{T}_A \mathbf{Q}_A^H
			$ and 
			$
			\mathbf{B} = \mathbf{Q}_B \mathbf{T}_B \mathbf{Q}_B^H
			$, where $\mathbf{Q}_A$ and $\mathbf{Q}_B$ are unitary matrices and $\mathbf{T}_A$ and $\mathbf{T}_B$ are upper triangular matrices.
			Then, we use the unitary matrices from the Schur decomposition to transform \( \mathbf{C} \) to $\mathbf{C}' = \mathbf{Q}_A^H \mathbf{C} \mathbf{Q}_B$.
			Next, we solve the simplified equation \( \mathbf{T}_A \mathbf{Z}' + \mathbf{Z}' \mathbf{T}_B = \mathbf{C}' \). This can be done using a back-substitution method since \( \mathbf{T}_A \) and \( \mathbf{T}_B \) are upper triangular matrices.
			Finally, we transform \( \mathbf{Z}' \) back to \( \mathbf{Z} \) by $\mathbf{Z} = \mathbf{Q}_A \mathbf{Z}' \mathbf{Q}_B^H$.
			
			However, it's worth noting that the critical frame $\mathbf{Z}$ of the objective function may not necessarily lie within the feasible set defined in \eqref{subProb:Z}. To address this, we can employ a projection operator to find a feasible frame starting from the critical frame $\mathbf{Z}$.
			\begin{align}
				{\prod}_{\mathcal{Z}}(\mathbf{Z})={\mathrm{arg\:min}_{\tilde{\mathbf{Z}}\in\mathcal{Z}}}\:\|\tilde{\mathbf{Z}}-\mathbf{Z}\|_{\text{F}}^{2}
				\label{eq:PiZ}
			\end{align}
			where $\mathcal{Z}=\{\mathbf{Z}|\mathrm{diag}(\mathbf{Z})=0,\mathbf{Z}\geq0\}$.
			For a simple and quick solution to \eqref{eq:PiZ}, we implement the projection
			operator $\prod_{\mathcal{Z}}(\mathbf{Z})$ as follows:
			\begin{align*}
				z_{ij}^{*}=\begin{cases}
					z_{ij} & \mathrm{if}\:z_{ij}\geq0\:\mathrm{and}\:i\neq j\\
					0 & \mathrm{\mathrm{i}f}\:z_{ij}<0\:\mathrm{or}\:i=j
				\end{cases}
			\end{align*}
			where $z_{ij}$ and $z_{ij}^{*}$ are the elements of $\mathbf{Z}$
			and its projection $\prod_{\mathcal{Z}}(Z)$, respectively.

			\subsubsection{$\mathbf{H}$-solution} Fixing $\mathbf{Z}$ and $\mathbf{Q}$, solve for $\mathbf{H}$ by
			\begin{equation}
				\begin{aligned}
					&\underset{\mathbf{H}}{\text{minimize}} \quad\|\mathbf{H}\|_{2,1}+\langle\mathbf{F},\mathbf{H}-\mathbf{Z}\mathbf{G}\rangle\\
					&\quad\quad\quad\quad\quad\quad\quad+\frac{\mathbf{H}-\mathbf{Z}\mathbf{G}}{2}\|\mathbf{H}-\mathbf{Z}\mathbf{G}\|_{\text{F}}^{2}\label{subProb:U}
				\end{aligned}
			\end{equation}
			which is equivalent to minimizing
			$ \quad\|\mathbf{H}\|_{2,1}+\frac{\gamma}{2}\|\mathbf{H}-(\mathbf{Z}\mathbf{G}-(1/\gamma)\mathbf{F})\|_{\text{F}}^{2}
			$
			with respect to $\mathbf{H}$.
			Denote $\mathbf{P}=\mathbf{Z}\mathbf{G}-(1/\gamma)\mathbf{F}$. Then
			the closed-form solution to \eqref{subProb:U} will be given as follows \cite{haynes2017computationally}:
			\[
			\mathbf{h}_{i}=\begin{cases}
				\frac{\|\mathbf{p}_{i}\|-(/\gamma)}{\|\mathbf{p}_{i}\|}\mathbf{p}_{i} & \mathrm{if}\:\|\mathbf{p}_{i}\|>/\gamma\\
				0 & \mathrm{otherwise}
			\end{cases}
			\]
			where $\mathbf{h}_{i}$, $\mathbf{p}_{i}$ are the ith column of $\mathbf{H}$, $\mathbf{P}$,
			respectively.

			{\color{black}	\begin{proposition}[Optimality]
					\label{prop:H-global}
					For fixed $(\mathbf{Z},\mathbf{Q},\mathbf{F},\gamma)$ the subproblem (\ref{subProb:U}) is the proximal operator of the $\ell_{2,1}$ norm and admits the closed-form group-shrinkage solution. Hence the $\mathbf{H}$-update attains the global minimizer of (\ref{subProb:U}) at every iteration.
			\end{proposition}
				\begin{proof}
			See Appendix \ref{sec:proof_H_global}.
		\end{proof}
		}

			\subsubsection{$\mathbf{Q}$-solution} Fixing $\mathbf{Z}$ and $\mathbf{H}$, solve for $\mathbf{Q}$ by
			\begin{align}
				\label{Prob:Q}
				\underset{\mathbf{Q}}{\min}\ \|\mathbf{Z}\|_{\mathbf{Q}},\quad\text{subject to}\quad \mathbf{Q}\in\mathcal{Q}
			\end{align}

			\begin{proposition}
				\label{prop:equalForm}
				We have the following equivalent problem
				$$\underset{\mathbf{Q}}{\min}\ \|\mathbf{Z}\|_{\mathbf{Q}}\iff\underset{\mathbf{Q}}{\min}\ \text{Trace}(\mathbf{Q}^{\top}(\mathbf{D}-(|\mathbf{Z}|+|\mathbf{Z}^{\mathrm{T}}|)/2)\mathbf{Q})$$
				where  the matrix $\mathbf{D}$ is known as the degree matrix. The degree matrix $\mathbf{D}$ is defined as: $D_{ii} = \sum_{j=1}^{N} [(|\mathbf{Z}|+|\mathbf{Z}^{\mathrm{T}}|)/2]_{ij}$. For all off-diagonal elements \( i \neq j \), \( D_{ij} = 0 \).
			\end{proposition}
			\begin{proof}
				See Appendix \ref{app:prop:equalForm}.
			\end{proof}
			
			The objective function in \eqref{Prob:Q} is the traditional normalized cut clustering problem \cite{shi2000normalized} with a TVS constraint.
			The Laplacian matrix $\mathbf{L}$ can be computed using the formula $\mathbf{L}=\mathbf{I}-\mathbf{D}^{-1/2}[(|\mathbf{Z}|+|\mathbf{Z}^{\mathrm{T}}|)/2]\mathbf{D}^{-1/2}$, where $\mathbf{I}$ is an identity matrix. Consequently, the eigenvectors $\mathbf{v}_{1},\mathbf{v}_{2},...,\mathbf{v}_{K}$ corresponding to the first $K$ smallest eigenvalues $\tilde{\lambda}_1,\tilde{\lambda}_2,...,\tilde{\lambda}_{K}$ of $\mathbf{L}$ are computed, satisfying $\mathbf{L}\mathbf{v}_{k}=\tilde{\lambda}_{k}\mathbf{v}_{k}$. These eigenvectors are arranged as columns in a matrix $\mathbf{V}\in\mathbbm{R}^{N\times K}$.
			
			Let $\mathbf{u}_{i} \in \mathbb{R}^{K}$ represent the vector of the $i$th row of $\mathbf{V}$, where $i=1,...,N$. The problem \eqref{Prob:Q} can be relaxed to the following form:
			\begin{align}
				\label{prob:k-means}
				\underset{\{\mathcal{C}_{k},\bm{{\mu}}_{k}\}_{k=1}^{K}}{\mathrm{minimize}}\:\sum_{k=1}^{K}\sum_{i\in\mathcal{C}_{k}}||\mathbf{u}_{i}-\bm{{\mu}}_{k}||_{2}^{2}, \quad\text{subject to}\quad \mathbf{Q}\in\mathcal{Q}
			\end{align}
			where $i\in \mathcal{C}_{k}$ if ${q}_{i,k}=1$.
			However, the requirement $\mathbf{q}_{i} = \mathbf{q}_{j} \, \forall j \in \mathcal{N}_{i}$ in the constraint $\mathbf{Q}\in\mathcal{Q}$ makes solving problem \eqref{prob:k-means} highly challenging. Since the constraint mandates that the clustering assignments of the $i$th frame and its neighbors remain consistent, we relax the constraint to that the clustering center corresponding to the $i$th frame should coincide with the center of its neighbors. This leads us to the formulation of the following problem:
			\begin{equation}
				\underset{\{\mathcal{C}_{k},\bm{{\mu}}_{k}\}_{k=1}^{K}}{\mathrm{minimize}}\:\sum_{k=1}^{K}\sum_{i\in\mathcal{C}_{k}}\Big(||\mathbf{u}_i-\bm{{\mu}}_{k}||_{2}^{2}+\eta\sum_{j\in \mathcal{N}_{i}}||\mathbf{u}_j-\bm{{\mu}}_{k}||_{2}^{2}\Big)\label{eq:Prob:TKm}
			\end{equation}
			where the penalty coefficient $\eta$ is set to $1/\mathcal{N}_{i}$ for weight balance.

			The term $\sum_{k=1}^{K}\sum_{i\in\mathcal{C}_{k}}||\mathbf{u}_{i}-\bm{{\mu}}_{k}||_{2}^{2}$ aims to independently fit all the data with the center $\{\bm{{\mu}}_{k}\}$. However, the term $\eta\sum_{k=1}^{K}\sum_{i\in\mathcal{C}_{k}}\sum_{j\in\mathcal{N}_{i}}||\mathbf{u}_{j}-\bm{{\mu}}_{k}||_{2}^{2}$ desires the $\bm{{\mu}}_{k}$ to be identical temporally, i.e., the center of $\mathbf{u}_{i}$ is the same as the center of $\mathbf{u}_{j}$ for any $j\in\mathcal{N}_{i}$. Consequently, minimizing these two terms simultaneously leads to a trade-off between fitting data to $K$ centers and maintaining temporal of the center assignment, where the desired outcome is to divide the data sequence into multiple small segments.

			It  is still challenging to solve problem \eqref{eq:Prob:TKm} directly due to its NP-hard nature. We first propose the following proposition, which will be utilized to adapt problem \eqref{eq:Prob:TKm} into a new form.
			\begin{proposition}
				The term 
				$\sum_{i\in\mathcal{C}_{k}}(||\mathbf{u}_i-\bm{{\mu}}_{k}||_{2}^{2}+\eta\sum_{j\in\mathcal{N}_{i}}||\mathbf{u}_j-\bm{{\mu}}_{k}||_{2}^{2})$  in (\ref{eq:Prob:TKm}) is equivalent to
				$ \sum_{i=1}^{N}||\mathbf{u}_i-\bm{{\mu}}_{k}||_{2}^{2}(\mathbbm{1}(i\in\mathcal{C}_{k})+ \eta n_{k}(i))$
				where $n_{k}(i)$ is the number of times the frame $\mathbf{u}_{i}$
				appears as a sequential neighbor of a frame in the $k$-th cluster,
				i.e., $n_{k}(i)=\sum_{j\in\mathcal{C}_{k}}\mathbbm{1}(i\in\mathcal{N}_{j})$
				and the indicator function $\mathbbm{1}(s)=1$
				if $s$ is true and zero otherwise. 
				\label{thm:Kmconvert}
			\end{proposition}
			\begin{proof}
				See Appendix \ref{app:Kmconvert}.
			\end{proof}
		According to proposition \ref{thm:Kmconvert}, problem \eqref{eq:Prob:TKm} can be rewritten as the following new weighted problem: \(\underset{\{\mathcal{C}_{k},\bm{{\mu}}_{k}\}_{k=1}^{K}}{\mathrm{minimize}}\:\sum_{k=1}^{K}\sum_{i=1}^{N}||\mathbf{u}_i-\bm{{\mu}}_{k}||_{2}^{2}w_{k,i}\), where \(w_{k,i}=\mathbbm{1}(i\in\mathcal{C}_{k})+ \eta n_{k}(i)\). Observing that the new weighted problem can be solved by addressing two sub-problems for \(\mathcal{C}_{k}\) and \(\bm{{\mu}}_{k}\) in an alternating manner when one is fixed, respectively, we first focus on solving the new problem with the given cluster assignment \(\{\mathcal{C}_{k}\}_{k=1}^{K}\). Denote the objective function of the new weighted problem as \(\mathcal{J}_{1}(\{\bm{{\mu}}_{k}\}_{k=1}^{K})\). If we take the derivative of \(\mathcal{J}_{1}(\{\bm{{\mu}}_{k}\}_{k=1}^{K})\) with respect to \(\bm{{\mu}}_{k}\) and set it to zero, i.e., \(\frac{\partial \mathcal{J}_{1}(\{\bm{{\mu}}_{k}\}_{k=1}^{K})}{\partial \bm{{\mu}}_{k}} = 0\), we obtain \(\bm{{\mu}}_{k} = \frac{1}{\sum_{i=1}^{N}w_{k,i}}\sum_{i=1}^{N}w_{k,i}\mathbf{u}_{i}\). We then solve the cluster assignment with the given cluster center. This is done by evaluating the weighted combination of the residual from the frame to a given center, as well as the residuals of its sequential neighbors, so that the estimated cluster label for the frame \(\mathbf{u}_{i}\) is assigned to the \(l\)-th cluster, where \(l = \underset{k\in\{1,2,...,K\}}{\mathrm{arg\:min}}\quad||\mathbf{u}_i-\bm{{\mu}}_{k}||_{2}^{2}w_{k,i}\).

			The algorithm alternates between center update and cluster assignment steps until convergence. In the center update step, the resulting center represents the global optimum given a cluster assignment. This step learns the center that minimizes the distance to all frames in the cluster, including their sequential neighbors. Therefore, the center update step cannot increase the overall objective function. Similarly, in the cluster assignment step, each frame is assigned to the cluster that minimizes the distance to itself and its sequential neighbors, which also cannot increase the overall objective function. Since there is a finite number of ways the frames can be assigned, and the objective function in the new weighted problem is bounded below by zero, the proposed alternating algorithm must terminate at a locally optimal clustering result. To determine the number of clusters \( K \), we use the silhouette score, which measures the similarity of a sample point to its own cluster in comparison to the nearest cluster. By calculating the silhouette score for different values of \( K \), the optimal number of clusters is chosen as the value of \( K \) that maximizes the silhouette score.

			By iteratively solving \eqref{subProb:Z}, \eqref{subProb:U}, and \eqref{Prob:Q}, we can obtain a solution to \eqref{Prob:NCGC_Lap}. During this process, we group the subspace embeddings by solving \eqref{Prob:Q} and update the embeddings based on feedback from the HMS solution of \eqref{Prob:Q}. The convergence of sub-problem \eqref{subProb:Z}, the closed-form solution of sub-problem \eqref{subProb:U}, and the convergence of solving \eqref{Prob:Q} ensure the overall convergence of the algorithm for \eqref{Prob:NCGC_Lap}. Algorithm \ref{alg:TVSH} presents the pseudocode for our clustering method.
			
			{\color{black}
\begin{theorem}[Convergence]
	\label{thm:convergence}
	Under bounded and lower-semicontinuous augmented Lagrangian, nondecreasing $\gamma_t\!\to\!\gamma_\infty\!\in(0,\infty)$, and bounded $\rho>1$, the proposed ADMM-based alternating scheme ensures monotonic decrease of the objective and convergence of $(\mathbf{Z}^{(t)},\mathbf{H}^{(t)},\mathbf{Q}^{(t)})$ to a first-order stationary point.  
	If each $\mathbf{Q}$-update reaches its relaxed global optimum, every accumulation point satisfies the KKT conditions.
\end{theorem}
		\begin{proof}
	See Appendix \ref{sec:proof_convergence}.
\end{proof}

This theorem confirms that the alternating optimization is theoretically stable and convergent:  
the objective value decreases monotonically, the iterates approach a stationary solution,  
and, with exact subproblem updates, the algorithm attains KKT-level optimality, guaranteeing reliable convergence behaviour in practice.

		}

			\section{Experimental Results and Analysis}
			\label{sec:Experiments}
			
		In this section, we first introduce the human motion datasets used in our experiments (Section \ref{dataset}). We then present a comparison of our method with state-of-the-art techniques (Section \ref{comparison}). Finally, we show the effective analysis of the LLM-based TVS (Section \ref{sec:LLM}).

			\begin{table*}[!htp]
				\caption{Clustering performance of compared methods in terms of \textbf{Acc} and \textbf{NMI} on four human motion videosets. The best result is highlighted in \textbf{bold}. The improvement relative to the second-best method is depicted by $\uparrow$. (M) denotes the need for labeled MAD dataset assistance, and (K) indicates the requirement for labeled Keck dataset assistance.
				}
				\scriptsize
				\begin{minipage}{0.25\linewidth}
					\centering
					\renewcommand\arraystretch{1.3}
					{
						(a) Results on Keck dataset\vspace{-0.15cm}
						\begin{tabular}[t]{|l|c|c|}
							\specialrule{1pt}{0pt}{0pt}
							\rowcolor{mygray} Method  & Acc $\uparrow$& NMI $\uparrow$\\
							\hline
							SSC \cite{elhamifar2013sparse}          & {0.3137} & {0.3858}     \\
							OSC \cite{tierney2014subspace}          & {0.4393} & {0.5931}    \\
							TSC(M) \cite{li2015temporal} &{0.4653}&{0.6935}\\
							LTS \cite{wang2018low}  &0.4924&0.6213\\
							DSAE \cite{bai2020dual}          & {0.5136}& {0.5100} \\
							VSDA  \cite{bai2022human}      & {0.5804}& {0.7397}  \\
							CDMS(M) \cite{zhou2022consistency}      & {0.6044}& {0.7891} \\
							SIBMSC \cite{wang2023self}         & {0.3886}   & {0.4744}     \\
							FSMSC \cite{chen2023fast}       &0.4702    &0.3970\\
							BTMSC \cite{wang2023bi}            & {0.4297}  & {0.4862}  \\
							ARLRR \cite{tang2023affine} &{0.5010}&{0.5270}\\
							DCTMSC \cite{chen2024double}              & 0.4723& 0.4866   \\
							DCMVC \cite{cui2024dual}    &{0.5395} & {0.8049}     \\
							\hline
							TVSH&  \textbf{0.8048}& \textbf{0.8690} \\
							\specialrule{1pt}{0pt}{0pt}
					\end{tabular}}
				\end{minipage}
				\begin{minipage}{0.24\linewidth}
					\centering
					\renewcommand\arraystretch{1.3}
					
					{
						(b) Results on MAD dataset\vspace{-0.15cm}
						\begin{tabular}[t]{|l|c|c|}
							\specialrule{1pt}{0pt}{0pt}
							\rowcolor{mygray} Method  & Acc $\uparrow$ & NMI $\uparrow$\\
							\hline
							SSC \cite{elhamifar2013sparse}          & {0.3817}   &  {0.4758}  \\
							OSC \cite{tierney2014subspace}       &  {0.4327}  &{0.5589}   \\
							TSC(K) \cite{li2015temporal} &{0.5473}&{0.7691}\\
							LTS \cite{wang2018low}  &0.5466&0.6547\\
							DSAE \cite{bai2020dual}     &  0.5898  & 0.6309   \\ 
							VSDA \cite{bai2022human}          & {0.5606}   &  {0.7770}  \\
							CDMS(K) \cite{zhou2022consistency}    & {0.6536} & {0.8251} \\
							SIBMSC \cite{wang2023self}             &  {0.3639} &{0.4309}    \\
							FSMSC \cite{chen2023fast}  &0.3914      &0.3226\\
							BTMSC \cite{wang2023bi}                &   {0.2397} & {0.2249}  \\
							ARLRR \cite{tang2023affine} &{0.5125}&{0.5099}\\
							DCTMSC \cite{chen2024double}               & 0.4885   & 0.5372   \\
							DCMVC \cite{cui2024dual}  & {0.5792}   &  {0.8286} \\
							\hline
							TVSH &  \textbf{0.8372}& \textbf{0.8438} \\
							\specialrule{1pt}{0pt}{0pt}
					\end{tabular}}
				\end{minipage}
				\begin{minipage}{0.24\linewidth}
					\centering
					\renewcommand\arraystretch{1.3}
					{
						(c) Results on UT dataset\vspace{-0.15cm}
						\begin{tabular}[t]{|l|c|c|}
							
							\specialrule{1pt}{0pt}{0pt}
							\rowcolor{mygray} Method   & Acc $\uparrow$& NMI $\uparrow$\\
							\hline
							SSC \cite{elhamifar2013sparse}          & {0.4389} & {0.4998}  \\
							OSC \cite{tierney2014subspace}            & {0.5846}&{0.6877}  \\
							TSC(K) \cite{li2015temporal} & {0.5213 }& {0.7216} \\
							LTS \cite{wang2018low}  &0.6724&0.7435\\
							DSAE \cite{bai2020dual}      & 0.7323& 0.6717  \\
							VSDA  \cite{bai2022human}    & 0.6203& 0.8226   \\
							CDMS(K) \cite{zhou2022consistency}   & {0.6547} & {0.8267}  \\
							SIBMSC \cite{wang2023self}                  & {0.4477}& {0.4894}  \\
							FSMSC \cite{chen2023fast}  &0.4787      &0.4213\\
							BTMSC \cite{wang2023bi}                 & {0.4162}  & {0.4051}   \\
							ARLRR \cite{tang2023affine} &{0.5148}&{0.5121}\\
							DCTMSC \cite{chen2024double}     & 0.5569 &0.5293\\
							DCMVC \cite{cui2024dual}        &{0.5371}&{0.7746} \\
							
							\hline
							TVSH & \textbf{0.8723}& \textbf{0.8488} \\
							\specialrule{1pt}{0pt}{0pt}
					\end{tabular}}
				\end{minipage}
				\begin{minipage}{0.255\linewidth}
					\centering
					\renewcommand\arraystretch{1.3}
					
					{
						(d) Results on Weiz dataset\vspace{-0.15cm}
						\begin{tabular}[t]{|l|c|c|}
							\specialrule{1pt}{0pt}{0pt}
							\rowcolor{mygray} Method& Acc $\uparrow$ & NMI $\uparrow$  \\
							\hline
							SSC \cite{elhamifar2013sparse}          & {0.4576}   &    {0.6009} \\
							OSC \cite{tierney2014subspace}           & {0.5216}   &   {0.7047} \\
							TSC(K) \cite{li2015temporal} &{0.5931}&{0.7971} \\
							LTS \cite{wang2018low}  &0.5674&0.6959\\
							DSAE \cite{bai2020dual}        &  0.6120  & 0.6627   \\
							VSDA \cite{bai2022human}          &  {0.6287} & {0.7992}   \\
							CDMS(K) \cite{zhou2022consistency}    & {0.6465} & {0.8601} \\
							SIBMSC \cite{wang2023self}           &   {0.4127} & {0.5435}   \\
							FSMSC \cite{chen2023fast}  &0.3914      &0.3226\\
							BTMSC \cite{wang2023bi}                 &   {0.3638} &{0.4382}  \\
							ARLRR \cite{tang2023affine} &{0.5436}&{0.5371}\\
							DCTMSC \cite{chen2024double}                 &0.5592    &   0.5906 \\
							DCMVC \cite{cui2024dual}    &{0.6030}    & {0.8326}  \\
							\hline
							TVSH&  \textbf{0.8745}& \textbf{0.9316} \\
							\specialrule{1pt}{0pt}{0pt}
					\end{tabular}}
				\end{minipage}
				\vspace{0.1in}
				\label{tab:Result_Acc_NMI}
			\end{table*}
					\begin{figure}[t]
				\begin{center}
					\includegraphics[width=1\columnwidth]{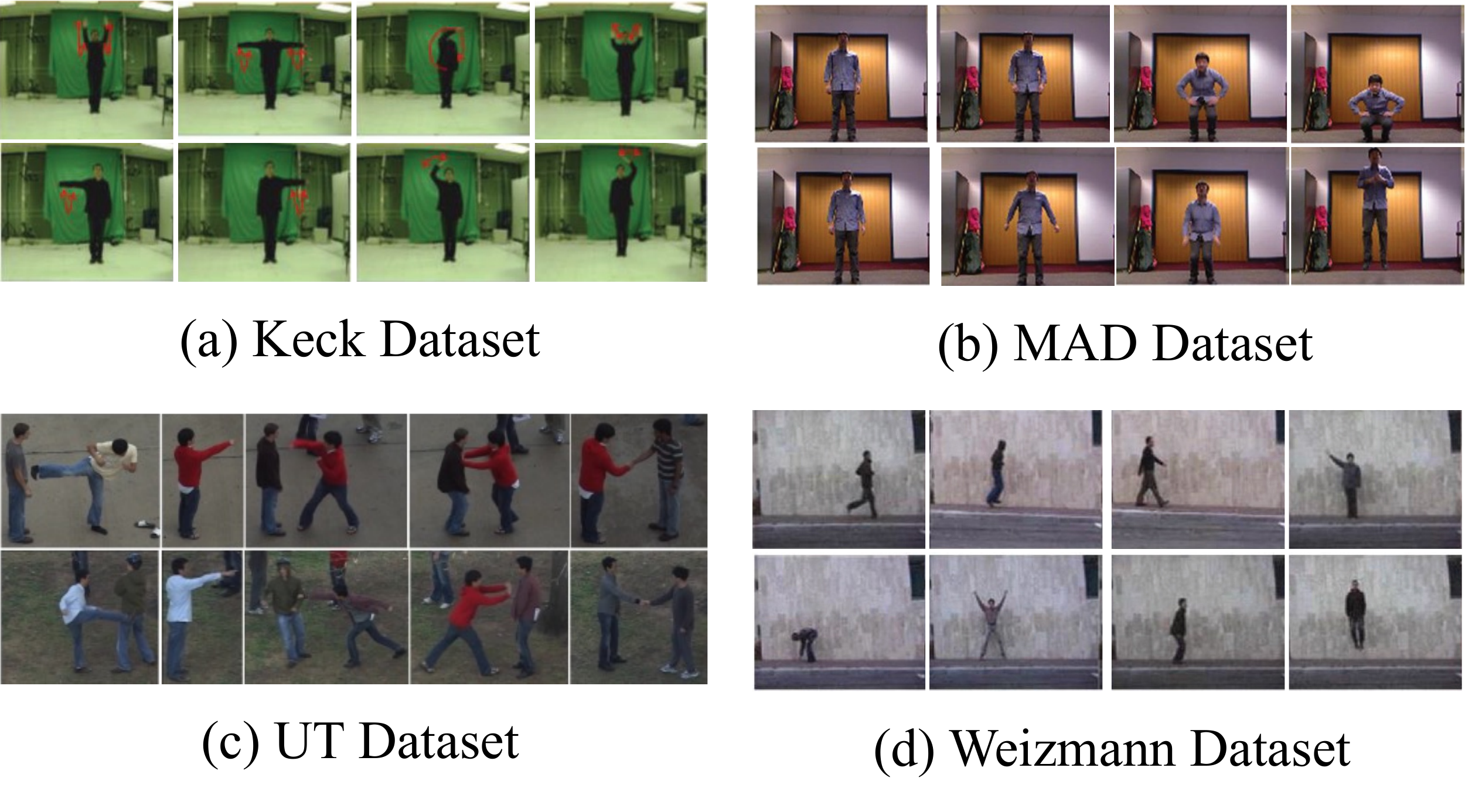}
					\vspace{-0.25cm}
					\caption{Sampling frames from four human motion benchmark datasets, \ie, (a) Keck~\cite{jiang2012recognizing}, (b) MAD~\cite{huang2014sequential}, (c) UT~\cite{ryoo2009spatio}, and (d) Weiz~\cite{gorelick2007actions}.}
					\vspace{-0.25in}
					\label{fig04}
				\end{center}
			\end{figure}
			
			\subsection{Human Motion Datasets and Experimental Setup}
			\label{dataset}
			
			To provide a comprehensive evaluation of the proposed model, we perform experiments on four well-established benchmark human motion datasets. Some example frames from these datasets can be seen in Figure~\ref{fig04}. 
			
			$\bullet$ \textit{Keck Gesture Dataset (Keck)} \cite{jiang2012recognizing} consists of 14 different motions from military signals, in which each subject is carried out 14 motions and gestures. Besides, the videos in this dataset were obtained by a fixed camera when these subjects stand out in a static background. 
			
			$\bullet$ \textit{Multi-Modal Action Detection Dataset (MAD)} \cite{huang2014sequential} consists of motions captured from various modalities using a Microsoft Kinect V2 system, which includes RGB images, depth cues, and skeleton formats. Specifically, the RGB images and 3D depth cues are of a size of $240\times320$. Moreover, each subject performs 35 different motions within two indoor scenes. 
			
			$\bullet$ \textit{UT-Intermotion Dataset (UT)} \cite{ryoo2009spatio} is composed of 20 videos, each of which includes six different motion types of human-human intermotions (such as punching, pushing, pointing, hugging, kicking, and handshaking). 
			
			$\bullet$ \textit{Weizmann Dataset (Weiz)} \cite{gorelick2007actions} is composed of 90 video sequences with 10 motions (running, walking, skipping, bending, etc.) captured by nine subjects in an outdoor environment. All videos have a size of $180\times144$ with 50 fps.

			\begin{table*}[!htp]
				\caption{Clustering performance of compared methods in terms of \textbf{Pr} and \textbf{ARI} on four human motion videosets.
				}
				\scriptsize
				
				\begin{minipage}{0.25\linewidth}
					\centering
					\renewcommand\arraystretch{1.3}
					{
						(a) Results on Keck dataset\vspace{-0.15cm}
						\begin{tabular}[t]{|l|c|c|}
							\specialrule{1pt}{0pt}{0pt}
							\rowcolor{mygray} Method  & Pr $\uparrow$& ARI $\uparrow$\\
							\hline
							SSC \cite{elhamifar2013sparse}         &0.3511  &  0.2446 \\
							OSC \cite{tierney2014subspace}          & 0.3767 & 0.2743  \\
							TSC(M) \cite{li2015temporal} &0.4214&0.3457\\
							LTS \cite{wang2018low}  &0.4457&0.3052\\
							DSAE \cite{bai2020dual}          & 0.4195 & 0.3418  \\
							VSDA  \cite{bai2022human}       & 0.4311 &  0.3529 \\
							CDMS(M) \cite{zhou2022consistency}      & {0.5828}& {0.5174} \\
							SIBMSC \cite{wang2023self}         & 0.3773   & 0.2292    \\
							FSMSC \cite{chen2023fast}       &0.4702    &0.3970\\
							BTMSC \cite{wang2023bi}            &0.3592  & 0.2418  \\
							ARLRR \cite{tang2023affine} &0.5772&0.5046\\
							DCTMSC \cite{chen2024double}               &0.3753  &  0.2741 \\
							DCMVC \cite{cui2024dual}    & 0.3908 &0.3136   \\
							\hline
							TVSH&  \textbf{0.7559}& \textbf{0.7214} \\
							\specialrule{1pt}{0pt}{0pt}
					\end{tabular}}
				\end{minipage}
				\begin{minipage}{0.24\linewidth}
					\centering
					\renewcommand\arraystretch{1.3}
					
					{
						(b) Results on MAD dataset\vspace{-0.15cm}
						\begin{tabular}[t]{|l|c|c|}
							\specialrule{1pt}{0pt}{0pt}
							\rowcolor{mygray} Method  & Pr $\uparrow$& ARI $\uparrow$\\
							\hline
						SSC \cite{elhamifar2013sparse}          &  0.3151  & 0.1994   \\
						OSC \cite{tierney2014subspace}           &   0.4024 & 0.2403   \\
						TSC(K) \cite{li2015temporal} &0.5116&0.3724\\
						LTS \cite{wang2018low}  &0.4673&0.3426\\
						DSAE \cite{bai2020dual}        & 0.5492   &  0.3891  \\
						VSDA \cite{bai2022human}          &  0.5667  &  0.3780  \\
						CDMS(K) \cite{zhou2022consistency}    & {0.5761} & {0.4128} \\
						SIBMSC \cite{wang2023self}           &  0.3227  & 0.1478   \\
						FSMSC \cite{chen2023fast}  &0.3914      &0.3226\\
						BTMSC \cite{wang2023bi}                 &   0.2882 & 0.1809   \\
						ARLRR \cite{tang2023affine} &0.5426&0.4072\\
						DCTMSC \cite{chen2024double}                 & 0.4508   &   0.2917 \\
						DCMVC \cite{cui2024dual}    &0.5487    &0.3623    \\
							\hline
							TVSH &  \textbf{0.7043}& \textbf{0.6973} \\
							\specialrule{1pt}{0pt}{0pt}
					\end{tabular}}
				\end{minipage}
				\begin{minipage}{0.24\linewidth}
					\centering
					\renewcommand\arraystretch{1.3}
					{
						(c) Results on UT dataset\vspace{-0.15cm}
						\begin{tabular}[t]{|l|c|c|}
							\specialrule{1pt}{0pt}{0pt}
							\rowcolor{mygray} Method   & Pr $\uparrow$& ARI $\uparrow$\\
							\hline
						SSC \cite{elhamifar2013sparse}          & 0.5426  & 0.3772   \\
						OSC \cite{tierney2014subspace}            & 0.5426& 0.3966  \\
						TSC(K) \cite{li2015temporal} &0.5864&0.4217\\
						LTS \cite{wang2018low}  &0.5774&0.4457\\
						DSAE \cite{bai2020dual}      & 0.6189& 0.4895  \\
						VSDA \cite{bai2022human}    & 0.6334& 0.5202    \\
						CDMS(K) \cite{zhou2022consistency}   & {0.6466} & {0.5539}  \\
						SIBMSC \cite{wang2023self}                  & 0.4972&0.3193 \\
						FSMSC \cite{chen2023fast}  &0.5242      &0.5047\\
						BTMSC \cite{wang2023bi}                 & 0.5360  & 0.3667   \\
						ARLRR \cite{tang2023affine} &0.6054&0.5213\\
						DCTMSC \cite{chen2024double}     & 0.5635 & 0.4383 \\ 
						DCMVC \cite{cui2024dual}      &0.6131 & 0.4688    \\		
							\hline
							TVSH & \textbf{0.7477}& \textbf{0.7153} \\
							\specialrule{1pt}{0pt}{0pt}
					\end{tabular}}
				\end{minipage}
				\begin{minipage}{0.255\linewidth}
					\centering
					\renewcommand\arraystretch{1.3}
					
					{
						(d) Results on Weiz dataset\vspace{-0.15cm}
						\begin{tabular}[t]{|l|c|c|}
							\specialrule{1pt}{0pt}{0pt}
							\rowcolor{mygray} Method& Pr $\uparrow$& ARI $\uparrow$  \\
							\hline
						SSC \cite{elhamifar2013sparse}          &0.4469    &0.3620    \\
						OSC \cite{tierney2014subspace}           & 0.5126   &0.4422    \\
						TSC(K) \cite{li2015temporal} &0.5667&0.5324\\
						LTS \cite{wang2018low}  &0.5991&0.5724\\
						DSAE \cite{bai2020dual}        &  0.6233  &  0.5406  \\
						VSDA \cite{bai2022human}          &  0.6180  & 0.5378   \\
						CDMS(K) \cite{zhou2022consistency}    & {0.6316} & {0.5561} \\
						SIBMSC \cite{wang2023self}           &   0.3643 & 0.2909   \\
						FSMSC \cite{chen2023fast}  &0.3914      &0.3226\\
						BTMSC \cite{wang2023bi}                 &   0.4397 &  0.3491  \\
						ARLRR \cite{tang2023affine} &0.6211&0.5146\\
						DCTMSC \cite{chen2024double}                 &  0.5502  & 0.4914   \\
						DCMVC \cite{cui2024dual}    &  0.6045  & 0.5280   \\
							\hline
							TVSH  &  \textbf{0.9012}& \textbf{0.8867} \\
							\specialrule{1pt}{0pt}{0pt}
					\end{tabular}}
				\end{minipage}
				\label{tab:Result_Pr_ARI}
				\vspace{0.1in}
			\end{table*}

		We evaluate clustering performance using four metrics: accuracy (Acc), normalized mutual information (NMI), precision (Pr), and adjusted rand index (ARI). These metrics assess the consistency between learned and true labels, with higher values indicating better performance. Let $\mathcal{L} = \{l_1, l_2, ..., l_N\}$ and $\hat{\mathcal{L}} = \{\hat{l}_1, \hat{l}_2, ..., \hat{l}_N\}$ represent the ground-truth and predicted labels, respectively, where $l_i$ and $\hat{l}_i$ denote the true and predicted labels for the $i$th sample. Acc is defined as the proportion of correctly clustered samples: $\mathrm{Acc} = \frac{1}{N} \sum_{i=1}^{N} \delta(l_i, \mathrm{map}(\hat{l}_i))$, where $\delta(a, b)$ is the indicator function ($\delta(a, b) = 1$ if $a = b$, and $0$ otherwise), and $\mathrm{map}(\cdot)$ maps predicted labels to the best matching true labels using the Hungarian algorithm \cite{kuhn1955hungarian}. NMI quantifies the coherence between two sets. Let $H(\mathcal{L})$ and $H(\hat{\mathcal{L}})$ represent the entropies of the sets $\mathcal{L}$ and $\hat{\mathcal{L}}$, respectively. NMI is defined as: $\mathrm{NMI}(\mathcal{L}, \hat{\mathcal{L}}) = {\mathrm{MI}(\mathcal{L}, \hat{\mathcal{L}})}/{\sqrt{H(\mathcal{L}) H(\hat{\mathcal{L}})}}$, where $\mathrm{MI}(\mathcal{L}, \hat{\mathcal{L}})$ measures the mutual information between the sets. Higher mutual information and lower uncertainty result in a higher NMI. If the sets are randomly distributed, NMI equals 0. Pr calculates the percentage of correctly clustered pairs among all pairs with the same clustering label. True positive (TP), false positive (FP), and false negative (FN) represent the numbers of correctly labeled samples in the positive class, misclassified samples in the positive cluster, and misclassified samples in the negative cluster, respectively. Precision is defined as: $\mathrm{Pr} = \frac{\mathrm{TP}}{\mathrm{TP} + \mathrm{FP}}$. ARI \cite{zhan2018multiview} quantifies the similarity between two sets: $\mathrm{ARI}=(\sum_{i,j=1}^{K}C_{n_{ij}}^{2}-\mathbb{E}[\text{RI}])/(C_{0}-\mathbb{E}[\text{RI}])$, where $C_{0}=\frac{1}{2}(\sum_{i=1}^{K}C_{|\hat{\mathcal{L}}^{(i)}|}^{2}+\sum_{i=1}^{K}C_{|\mathcal{L}^{(i)}|}^{2})$ and $\mathbb{E}[\text{RI}] = {\sum_{i=1}^{K} C_{|\mathcal{L}^{(i)}|}^2 \sum_{i=1}^{K} C_{|\hat{\mathcal{L}}^{(i)}|}^2}/{C_N^2}$. Here, $|\mathcal{L}^{(i)}|$ and $|\hat{\mathcal{L}}^{(i)}|$ represent the number of samples in the $i$th cluster of the ground-truth and predicted labels, respectively. The value $n_{ij}$ denotes the number of samples in the $i$th true cluster grouped into the $j$th predicted cluster. The notation $C_n^m$ represents the number of ways to choose $m$ items from $n$.

		We evaluate the performance of our method through a comparative analysis with thirteen approaches, as outlined in Section~\ref{sec:rw}. Each method was independently tested ten times, and the average results were reported. For the proposed scheme, the TVS learning was performed using the following LLMs: {GPT-o1}, {DeepSeek-v3-2-exp}, {Claude-Sonnet-4-5-20250929}, {Gemini-2.0-Flash-exp}, {Grok-4}, and {Qwen3-235B-a22b}.

{\color{black}
\subsection{HMS Performance Comparison}
\label{comparison}
\begin{table}[t]
	\caption{The comparison of run-time (minute) on the Keck dataset.}
	\centering
	\label{fig:RT}
	\resizebox{1\columnwidth}{!}{ 
		\begin{tabular}{l|ccccccccc}
			\toprule[1.5pt]  
			Method  & SSC& OSC & TSC & LTS & DSAE& VSDA & CDMS  & \tabularnewline
			\hline
			Time & 1.1 & 2.5 & 1.9 & 3.3 & 4.5 & 4.4 & 5.1  & \tabularnewline
			\midrule
			Method  & SIBMSC & FSMSC& BTMSC  & ARLRR & DCTMSC & DCMVC & TVSH  & \tabularnewline
			\hline
			Time &  5.8 &4.3 & 4.4 & 5.3 & 4.7 &4.8& 4.2  & \tabularnewline
			\bottomrule[1.5pt]  
	\end{tabular}} 
	
\end{table}

Tables~\ref{tab:Result_Acc_NMI}--\ref{tab:Result_Pr_ARI} summarize results on four benchmarks (Keck, MAD, UT, Weiz) using {Acc}, {NMI}, {Pr}, and \textbf{ARI}. 
TVSH attains the best performance across all datasets and metrics. 

On the Keck dataset, TVSH improves accuracy from 0.6044 (CDMS) to 0.8048 and NMI from 0.8049 (DCMVC) to 0.8690, representing a significant improvement over the best baseline. On MAD, the accuracy increases from 0.6536 to 0.8372, while NMI rises from 0.8286 to 0.8438. On UT, TVSH achieves 0.8723 accuracy and 0.8488 NMI, both higher than those of existing methods. On the more diverse Weiz dataset, TVSH reaches 0.8745 accuracy and 0.9316 NMI, improving by about 0.23 and 0.07, respectively, compared with the best previous method.
Precision and ARI exhibit consistent improvement trends, confirming the robustness of TVSH in maintaining temporal coherence and enhancing motion discriminability. 

Table~\ref{fig:RT} reports wall-clock time on Keck, which shows that TVSH maintains reasonable computational efficiency. Although it requires slightly more time than lightweight baselines such as SSC or OSC, the additional cost is modest and justified by its substantial performance gains. 

\subsubsection{Superiority of the Proposed TVSH Method}

\begin{figure}[t]
	\begin{center}
		\includegraphics[width=0.8\columnwidth]{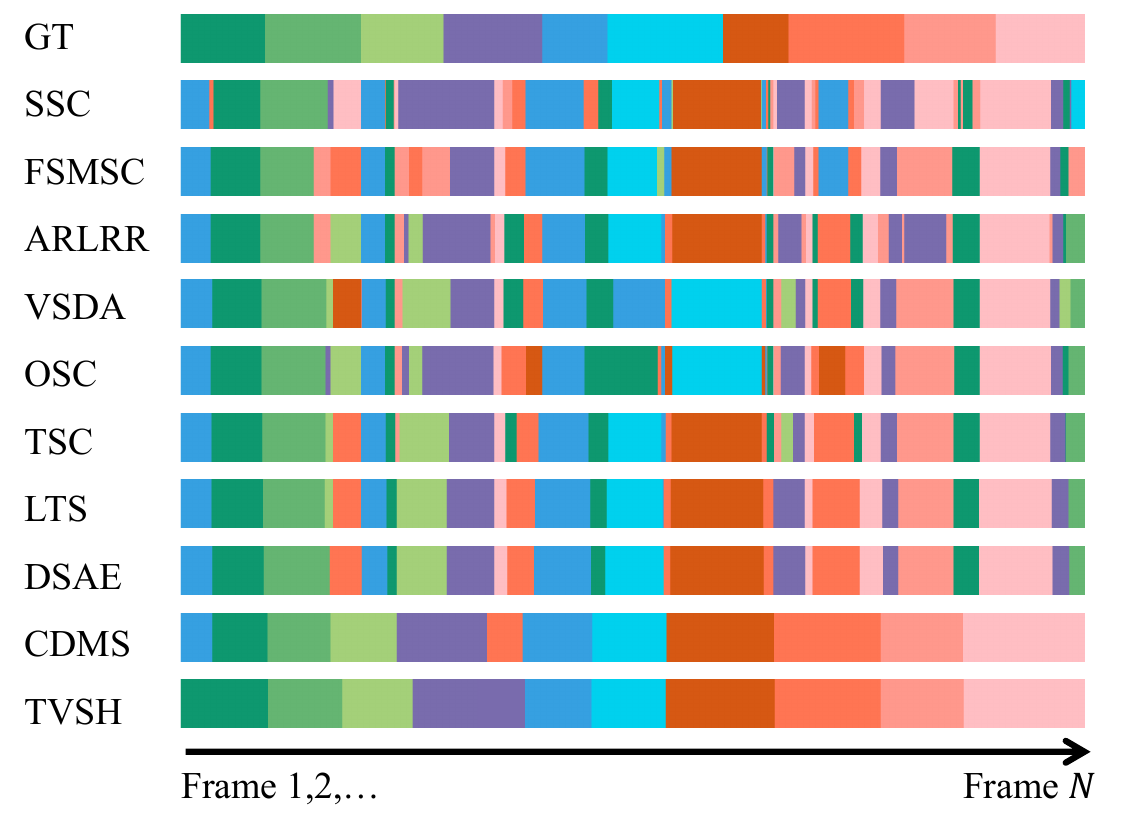}
		\caption{Visualization of motion segmentation results of the proposed method and comparisons on Keck dataset. The different colors denote different motions. GT depicts the ground truth.}
		\label{fig:Performance_bar}
	\end{center}
\end{figure}

\begin{figure}[t]
	\begin{center}
		\includegraphics[width=1\columnwidth]{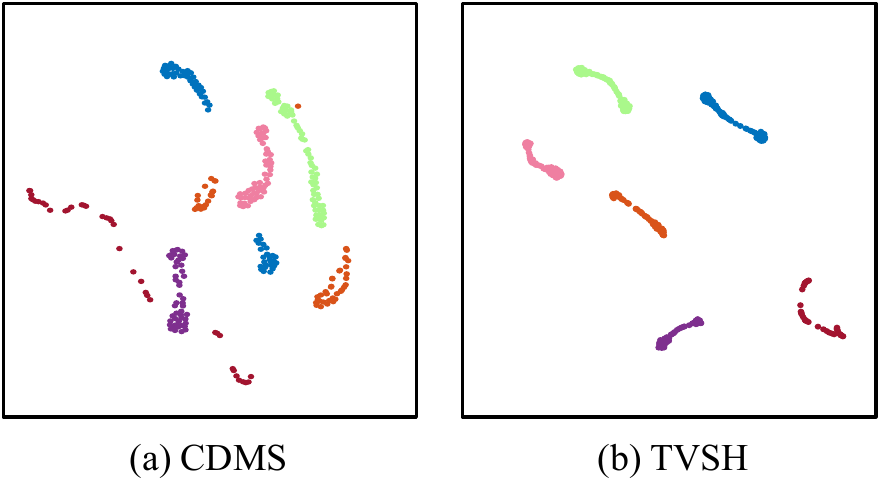}
		\caption{Visualization of the two-dimensional t-SNE of the extracted features from six motions in the Weiz dataset. Points in different colors depict frames of different motions.}
		\label{fig:EmbeddingtSNE}
	\end{center}
\end{figure}

\paragraph{Superiority 1: Temporal Modeling}
Conventional methods often fail to explicitly capture temporal dependencies in human motion sequences, making it difficult to identify gradually transitioning motions. Specifically, when a single motion contains multiple stages with significant amplitude variations, it is often misinterpreted as multiple separate motions. Exploring temporal modeling offers a solution to this challenge. The proposed method first utilizes an LLM to obtain a TVS matrix, which explicitly encodes temporal relationships between frames. Then, the TVS matrix-based regularization is introduced to enforce temporal continuity in both the embedding space and the segmentation result, thereby reducing ambiguity in temporal motion transitions.

Figure~\ref{fig:Performance_bar} presents the motion segmentation results of the proposed method, along with comparisons on the Keck dataset. The proposed TVSH generates temporally coherent motion segmentation with well-aligned motion boundaries. Methods that do not explicitly model temporal dependencies (e.g., SSC and its variants such as FSMSC and ARLRR) achieve only limited temporal coherence. Although approaches incorporating temporal cues (e.g., OSC, TSC, LTS, DSAE, VSDA, and CDMS) perform better in terms of temporal coherence, they still fail to capture the full temporal coherence of motion sequences and cannot guarantee accurate temporal continuity in the motion segmentation results. In contrast, the proposed method leverages LLM to learn more precise temporal coherence and uses it to guide motion segmentation, producing temporally semantically accurate motion segments.

Figure~\ref{fig:EmbeddingtSNE} further visualizes the two-dimensional t-SNE embeddings of the extracted features \( u_i \) from six motions in the Weiz dataset. The proposed method produces more compact clusters than the strong baseline CDMS, highlighting the effectiveness of temporal regularization in the embedding component of the proposed TVSH. Based on the embedded features shown in Figure 4(b), achieving better motion segmentation performance becomes easier. This also explains why the proposed method can achieve more temporally semantically accurate motion segments, as shown in Figure~\ref{fig:Performance_bar}.

\begin{figure}[t] 
	\begin{center}
		\includegraphics[width=1\columnwidth]{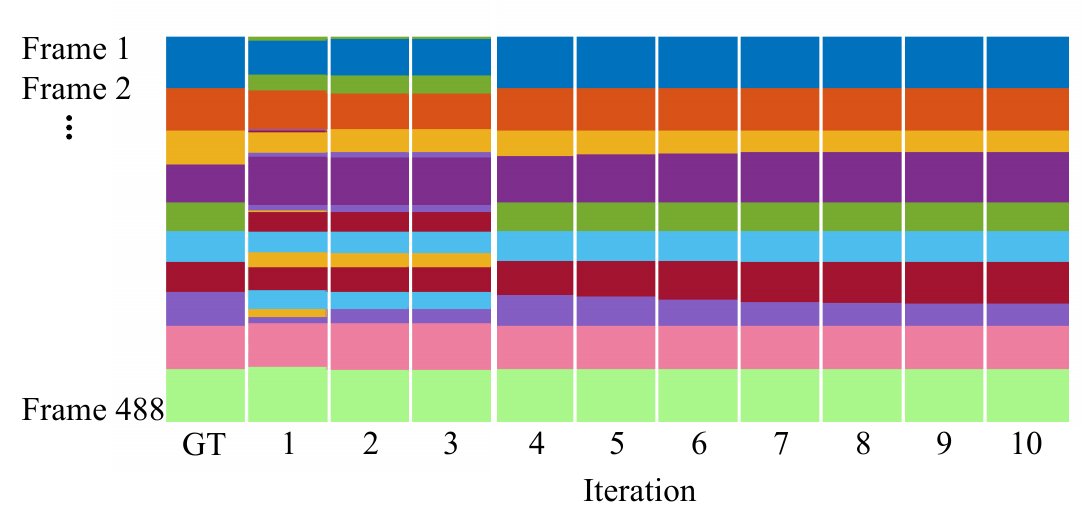}
		\vspace{-0.15cm}
		\caption{
		Visualizations of the motion segmentation results in the different iteration of the proposed TVSH on the Weiz dataset. Different colors represent distinct motion assignments for the frames. 'GT' refers to the 'ground truth' motion segmentation results.
		}
		\label{fig:ConvergeProced}
	\end{center}
\end{figure}

\begin{figure}[t] 
	\begin{center}
		\includegraphics[width=1\columnwidth]{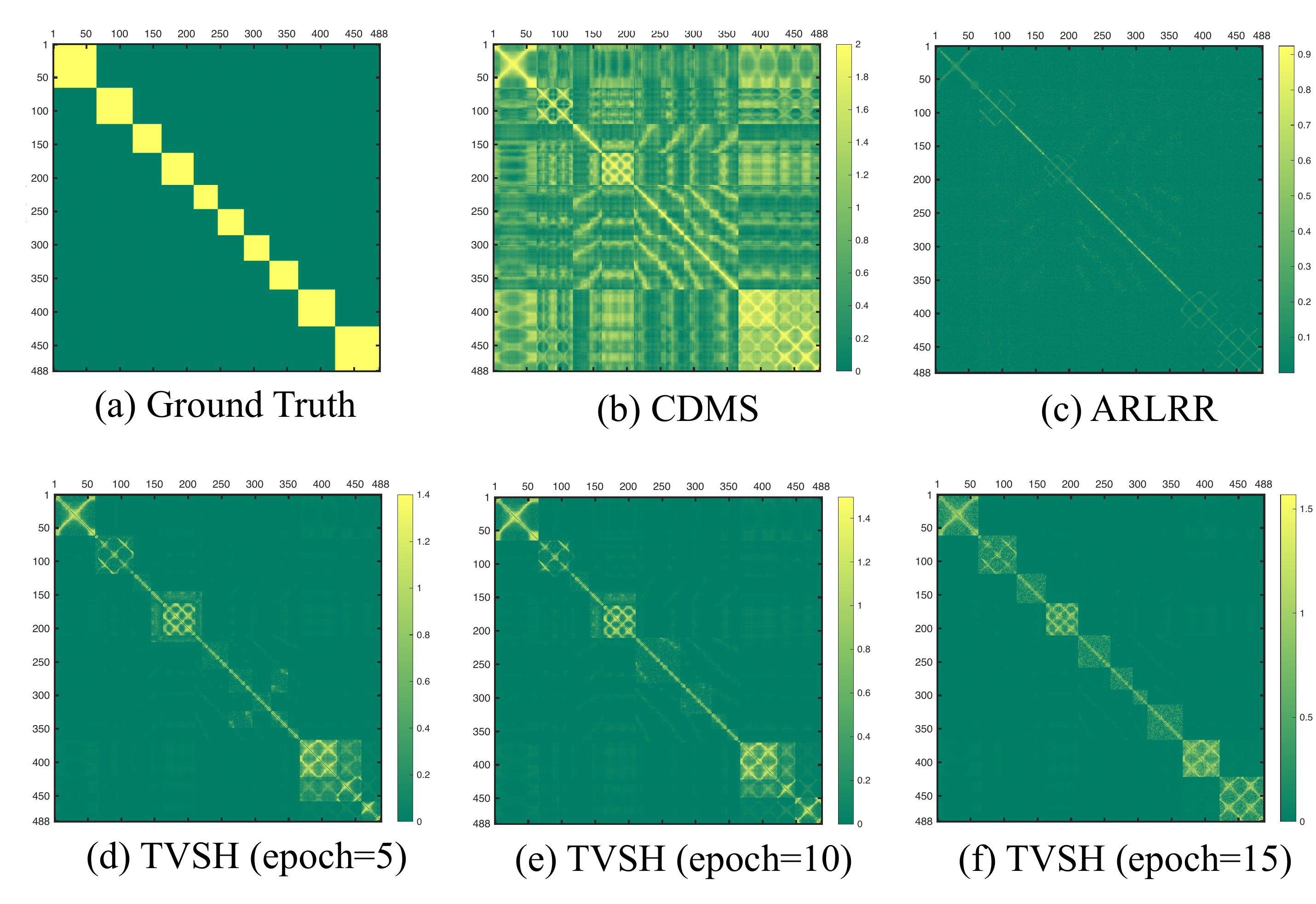}
		\vspace{-0.15cm}
		\caption{Visualizations of the similarity matrix. (a) Ground-truth similarity matrix of the Ido human motion sequence in the Weiz dataset. Yellow regions indicate high similarity between frames of the same motion, while green regions denote zero similarity between frames of the same motion. (b)–(c) show the similarity matrices provided by the baselines CDMS and ARLRR. (d)–(f) show the similarity matrices generated by the proposed TVSH at different iterations.
		}
		\label{fig:selfexpressVariety}
	\end{center}
\end{figure}

\paragraph{Superiority 2: Joint Optimization}  
Traditional clustering-based human motion segmentation methods typically perform feature extraction and clustering of embedded features in separate stages, resulting in an embedding process that does not receive feedback from the clustering outcome. However, a large clustering loss indicates that the embedded features are difficult to cluster, and it is meaningful to adjust the embedding to generate new features that minimize the clustering loss as much as possible. Therefore, our method adopts a feedback-enabled joint optimization framework, where the segmentation results iteratively refine the learned embeddings to achieve the smallest possible clustering loss.

Figure~\ref{fig:ConvergeProced} shows the motion segmentation results of the proposed TVSH on a video from the Weiz dataset. The TVSH converges within ten iterations. During the first three iterations, the proposed TVSH focuses on identifying optimal segmentation boundaries and merging fragmented regions temporally. In the subsequent fourth to tenth iterations, it fine-tunes the boundaries, yielding stable and temporally coherent segmentation results. 

Figures~\ref{fig:selfexpressVariety}(b-c) presents the similarity matrices for different methods, while Figures~\ref{fig:selfexpressVariety}(d)–(f) show the similarity matrices across different iterations from the joint optimization of the proposed TVSH. The similarity matrix of the proposed TVSH framework progressively exhibits a clearer block-diagonal structure, demonstrating strong alignment with the ground truth shown in Figure~\ref{fig:selfexpressVariety}(a). In contrast, the similarity matrices obtained by baseline methods, such as CDMS and ARLRR (Figure~\ref{fig:selfexpressVariety}(b) and (c)), are less structured and more diffuse. Obviously, a similarity matrix that is more consistent with the ground truth in Figure~\ref{fig:selfexpressVariety}(a) facilitates more accurate motion segmentation.

\paragraph{Ablation Study}  

\begin{table}[t]
	\caption{Ablation study of the effects of temporal modeling and joint optimization.}
	\label{fig:ablation}
	\vspace{0.05in}
	\centering
	\resizebox{0.5\textwidth}{!}{ %
		\begin{tabular}{p{2.5cm}|cc|cc|cc}
			\specialrule{1pt}{0pt}{0pt} 
			\rowcolor{mygray}& \multicolumn{2}{c|}{Keck} & \multicolumn{2}{c|}{MAD} & \multicolumn{2}{c|}{UT} \tabularnewline
			\rowcolor{mygray}  & Acc & NMI & Acc & NMI & Acc & NMI \tabularnewline
			\hline 
			TVSH (w/o joint optimization) & {0.7423} & {0.8429} & {0.7848} & {0.8322} & {0.8254} & {0.8371} \tabularnewline
			TVSH (w/o temporal model) & {0.7152} & {0.7428} & {0.7211} & {0.7546} & {0.6714} & {0.6211} \tabularnewline
			TVSH & {0.8048} & {0.8690} & {0.8372} & {0.8438} & {0.8723} & {0.8488} \tabularnewline
			\specialrule{1pt}{0pt}{0pt}
	\end{tabular}}
\end{table}

To assess the contribution of temporal modeling and joint optimization in TVSH, we conduct ablation experiments by selectively removing the temporal model and the joint optimization module. The variant TVSH (w/o temporal model) in Table~\ref{fig:ablation} removes the temporal prior by setting \( G = I \), thereby disabling temporal regularization. 

When the LLM-guided TVS is incorporated as the temporal model, performance consistently improves across all datasets. For instance, accuracy increases from 0.7152 to 0.8048 on the Keck dataset, from 0.7211 to 0.8372 on MAD, and from 0.6714 to 0.8723 on UT. These gains confirm that the TVS-based temporal modeling enhances human motion segmentation performance.

The variant TVSH (w/o joint optimization) in Table~\ref{fig:ablation} runs the proposed TVSH for only one iteration. This version produces weaker segmentation quality, as reflected by a 6–8\% drop in accuracy across the datasets. In contrast, the full TVSH model benefits from iterative feedback between clustering and embedding, progressively refining segment boundaries and aligning the learned representation with semantic motion transitions. These results validate that joint optimization is essential for achieving temporally semantically consistent motion segmentation.

\subsubsection{Performance Bottleneck}

Although the proposed method achieves consistent improvements across all benchmarks, two primary bottlenecks preventing TVSH from reaching perfect segmentation accuracy were identified and analyzed.

\begin{figure}[t]
	\centering
	\includegraphics[width=1\columnwidth]{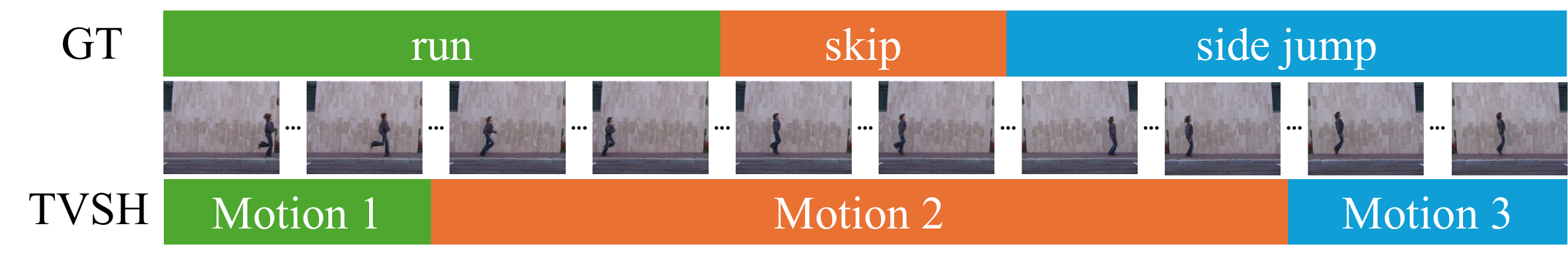}
	\caption{Illustration of a failure case on the Weiz dataset (subject ``ido'') showing gradual transition boundaries between actions.}
	\label{fig:FalseCase1}
\end{figure}

\paragraph{Bottleneck I: Gradual Transition Boundaries}  
The first bottleneck arises in sequences where actions evolve smoothly without clear-cut temporal boundaries. As shown in Figure~\ref{fig:FalseCase1}, when a motion transitions gradually from one motion to another (e.g., run $\to$ skip $\to$ side-jump), both visual and kinematic cues change continuously. These intermediate frames are semantically ambiguous, leading to minor drifts in segmentation boundaries or partial merging of adjacent segments.

\begin{figure}[t]
	\centering
	\includegraphics[width=1\columnwidth]{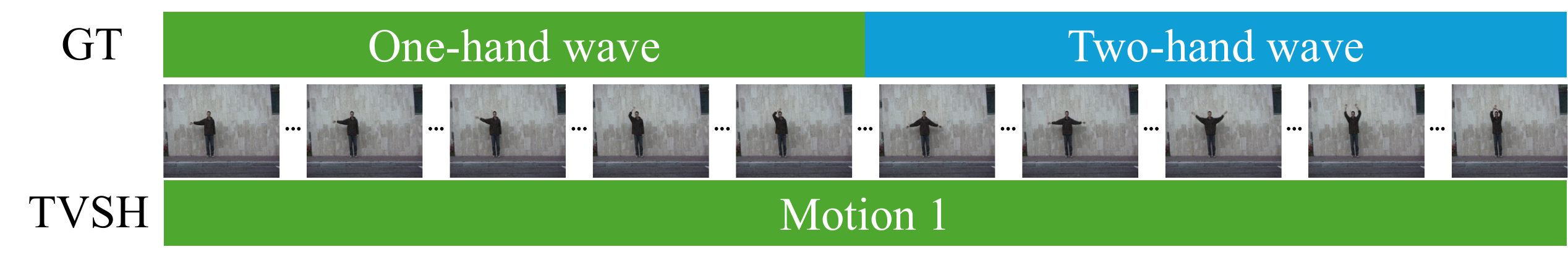}
	\caption{Illustration of a failure case on the Weiz dataset (subject ``ido'') showing visually similar or look-alike actions.}
	\label{fig:FalseCase2}
\end{figure}

\paragraph{Bottleneck II: Look-Alike Motions}  
The second bottleneck arises when motions share highly similar morphological characteristics. As depicted in Figure~\ref{fig:FalseCase2}, for visually related motions such as raising one hand and raising both hands, the motion segmentation algorithm may incorrectly classify these two motions as the same "raising hand" action. This occurs because the visual cues and kinematic features between these motions overlap significantly, making it difficult for the model to distinguish between them.

\subsubsection{Limitations and Future Directions}

Despite the effectiveness of the proposed method, several limitations remain. The first limitation arises in sequences where actions evolve smoothly without clear-cut temporal boundaries, causing minor drifts or partial merging of adjacent segments. The second limitation occurs when motions share highly similar morphological characteristics. This is due to the significant overlap in visual and kinematic features, making it difficult for the model to distinguish between them.

To address the limitations outlined above, future work will focus on enhancing the model's ability to handle gradual motion transitions and look-alike motions. For gradual transition boundaries, we plan to incorporate uncertainty-aware temporal modeling, which can adaptively capture smooth variations and probabilistic transition boundaries. This will help the model distinguish between genuine motion transitions and intra-action fluctuations. Additionally, we aim to integrate multimodal features, such as skeletal joint trajectories, optical flow, and motion energy maps, to provide richer dynamic and geometric context. These modalities will enable the model to better differentiate between visually similar but semantically distinct actions, improving segmentation accuracy and temporal coherence in motion sequences.
}

{\color{black}
\subsection{Effective Analysis of the LLM-Based TVS}
\label{sec:LLM}

This section analyzes the effectiveness, interpretability, and generalization of the proposed LLM-based TVS framework. Section~\ref{subsec:VisualizationLLM} visualizes the generated TVS matrices to assess their ability to capture temporal adjacency and motion coherence. Section~\ref{subsec:differentLLM} compares different LLMs to evaluate how model architectures affect segmentation accuracy. Section~\ref{subsec:differentPromt} examines the impact of prompt design on temporal reasoning. Section~\ref{subsec:differentmotion} investigates performance across diverse motion types. Finally, Section~\ref{subsec:LLMlimitation} summarizes the limitations of LLM-based TVS inference.

\begin{figure}[t]
	\begin{center}
		\includegraphics[width=1\columnwidth]{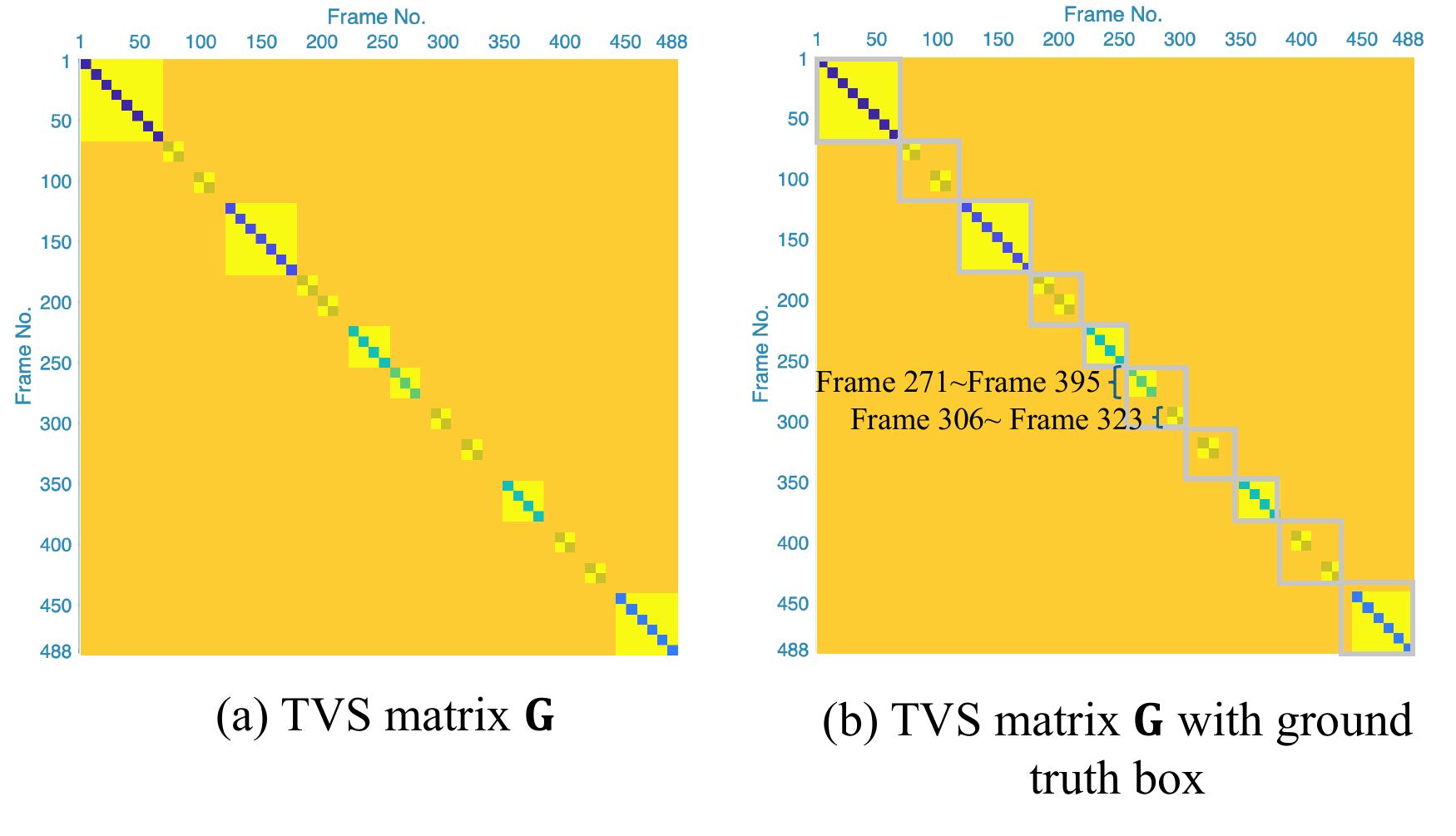}
		\vspace{-0.15cm}
		\caption{{\color{black}(a) The TVS matrix $\mathbf{G}$ on Weiz dataset (person 'ido') generated by GPT-o1. 
				(b) The TVS matrix $\mathbf{G}$ with ground truth label.
				The gray boxes indicate groups of frames that describe the same motion.}}
				\vspace{-0.25cm}
		\label{fig:G}
	\end{center}
\end{figure}

\subsubsection{Visualization of TVS from LLM}
\label{subsec:VisualizationLLM}
We employ \textit{GPT-o1} to generate the TVS matrix \( \mathbf{G} \). Figure~\ref{fig:G}(a) presents an example TVS matrix from the \emph{Weiz} dataset for the person ''ido''. The TVS generated by the LLM exhibits minor inconsistencies. For instance, in the sixth motion, frames \emph{271--323} correspond to the same “side jump” motion in the ground truth, yet the LLM identifies frames \emph{271--295} as depicting the same motion and \emph{306--323} as depicting the same motion separately. As shown in Figure~\ref{fig:G}(b), such partial inconsistencies occur in five out of ten motion segments. Nevertheless, the LLM does not introduce false distinctions, as it never explicitly labels segments \emph{271--395} and \emph{306--323} as different motions. Consequently, these minor inconsistencies do not mislead the subsequent TVSH algorithm. For the motions 1, 3, 5, 8, and 10, the LLM produces nearly perfectly consistent segmentations, demonstrating its effectiveness in capturing semantic motion coherence and temporal adjacency.

\subsubsection{Comparison with Different LLM Models}
\label{subsec:differentLLM}

We further evaluate six widely used multimodal LLMs for TVS learning, including
\textit{GPT-o1}, \textit{DeepSeek-v3-2-exp}, \textit{Claude-Sonnet-4-5-20250929}, 
\textit{Gemini-2.0-Flash-exp}, \textit{Grok-4}, and \textit{Qwen3-235B-a22b}. 
Table~\ref{tab:llm_runtime} summarizes the parameter scales, runtime, and segmentation accuracy of the proposed TVSH, evaluated on the Weiz ``ido'' video (488 frames) using these LLM models under identical prompt and API configurations. All experiments were conducted on a MacBook Air equipped with an M4 chip and 32\,GB of memory.
Overall, \textit{Gemini-2.0-Flash-exp} achieves the fastest runtime, followed by \textit{Claude-Sonnet-4-5-20250929} and \textit{DeepSeek-v3-2-exp}, whereas \textit{Grok-4} is the slowest due to its extremely large parameter size. 
\textit{GPT-o1} and \textit{Qwen3-235B-a22b} fall in the mid range, balancing accuracy and computational cost.

In terms of segmentation accuracy, \textit{GPT-o1} attains the best overall performance (88.27\%), closely followed by \textit{Qwen3-235B-a22b} and \textit{Gemini-2.0-Flash-exp}, which achieve comparable results. 
\textit{Claude-Sonnet-4-5-20250929} and \textit{DeepSeek-v3-2-exp} exhibit slightly lower accuracy, while \textit{Grok-4}, despite its largest parameter scale, yields the lowest performance. 
These results indicate that a larger model size does not necessarily guarantee better temporal reasoning or motion understanding; rather, architectural design and multimodal alignment play a more decisive role. 
Furthermore, model runtime generally increases with parameter size, reflecting the trade-off between computational complexity and inference precision. 
Nevertheless, all evaluated LLMs enhance TVS quality over non-LLM baselines (76.51\%), confirming that integrating multimodal reasoning effectively strengthens temporal semantics and motion segmentation performance.

\begin{table}[t]
	\centering
	\caption{\color{black}Parameter scale, runtime, and accuracy of different LLMs for TVS learning on the Weiz dataset (subject ``ido'').}
	\label{tab:llm_runtime}
	{\color{black}
		\resizebox{0.5\textwidth}{!}{
			\begin{tabular}{p{4cm}ccc}
				\toprule
				\textbf{Method} & \textbf{Parameters (B)} & \textbf{Runtime/question (s)} & \textbf{Acc (\%)} \\
				\midrule
				GPT-o1 & 175 & 1.84 & 88.27 \\
				DeepSeek-v3-2-exp & 67 & 0.95 & 85.71 \\
				Claude-Sonnet-4-5-20250929 & 70 & 0.72 & 86.88 \\
				Gemini-2.0-Flash-exp & 120 & 0.36 & 87.47 \\
				Grok-4 & 314 & 3.69 & 84.17 \\
				Qwen3-235B-a22b & 235 & 2.20 & 87.52 \\
				\bottomrule
			\end{tabular}
		}
	}
\end{table}

\subsubsection{Comparison with Different Prompts}
\label{subsec:differentPromt}

To examine the effect of prompt design on temporal semantic inference, we compared several variants of the prompt used to instruct the LLM.  
While the baseline prompt offers simplicity and generalization, additional prompt designs can enhance precision, robustness, and interpretability in identifying temporal consistency between frames.

\textit{(a) Baseline prompt.}  
This concise binary question directly queries the LLM’s perception of motion similarity:  
\begin{quote}
	\textit{``Do these two neighboring frames depict the same human motion? Answer Yes or No.''}
\end{quote}
Although simple and generalizable, this form provides limited guidance on how the model should evaluate visual similarity. 
Therefore, we explored more detailed formulations that explicitly direct attention toward motion dynamics and semantic continuity.

\textit{(b) Attribute-focused prompt.}  
The LLM was asked to compare explicit aspects of human motion, such as global body posture, limb configuration, contact state (e.g., feet or hand support), and motion direction, while ignoring irrelevant visual variations like background and illumination:  
\begin{quote}
	\textit{``Carefully compare the two human figures. Focus on body posture, limb angles, contact with the ground, and movement direction. Ignore lighting, clothing, and background. Decide if they represent the same stage of an action. Answer Yes or No.''}
\end{quote}

\begin{table}[t]
	\centering
	\caption{Accuracy of different prompts for TVS Learning on Weiz dataset.}
	\label{tab:Prompt}
	{\color{black}
		\begin{tabular}{lcccccc}
			\toprule
			\textbf{Prompt} & a) & b) & c) & d) & e) & f) \\
			\midrule
			\textbf{Acc} & 88.27 & 89.14 & 88.89 & 89.58 & 89.77 & 89.94 \\
			\bottomrule
		\end{tabular}
	}
\end{table}

\textit{(c) Confidence-based prompt.}  
To quantify uncertainty in LLM judgment, we introduced a structured response that requests a confidence score:  
\begin{quote}
	\textit{``Do these two frames depict the same human motion? Provide your answer (Yes/No) and a confidence score between 0 and 1.''}
\end{quote}
This allows adaptive thresholding during TVS construction and enables selective re-querying of low-confidence pairs.

\textit{(d) Step-aware prompt.}  
For temporally distant frames, the model was instructed to reason about motion continuity across a temporal gap:  
\begin{quote}
	\textit{``Compare frame $i$ and frame $i+\Delta t$. Decide whether they correspond to the same stage of motion despite intermediate movement. Ignore viewpoint and background differences.''}
\end{quote}

\textit{(e) Phase-aware prompt.}  
Incorporating explicit temporal reasoning, the LLM is asked to determine whether the two frames belong to the same \emph{action phase} (e.g., preparation, execution, or completion):  
\begin{quote}
	\textit{``Identify whether these two frames occur in the same phase of an action (preparation, execution, or completion). Focus on body posture and motion trajectory. Answer Yes or No.''}
\end{quote}
This helps capture fine-grained transitions within a continuous action.

\textit{(f) Causal-motion prompt.}  
To leverage the model’s reasoning ability, we designed a prompt that emphasizes causal understanding of movement progression:  
\begin{quote}
	\textit{``Analyze how the motion evolves between these two frames. Determine if the second frame naturally follows from the first as part of the same continuous action. Answer Yes or No.''}
\end{quote}
This causal formulation improves temporal coherence by aligning LLM reasoning with the physical progression of motion.

Table~\ref{tab:Prompt} reports the accuracy of six prompt variants used to guide the LLM in HMS. The results exhibit a steady improvement from the baseline formulation to the more context-aware and causality-driven designs, indicating that richer semantic cues lead to better temporal reasoning. The \textit{causal-motion prompt (f)} achieves the highest accuracy (89.94\%), confirming that prompting the LLM to reason about physical motion progression enhances its ability to capture temporal continuity and human-intuitive semantics. The \textit{phase-aware prompt (e)} yields a comparable result (89.77\%), showing that explicitly considering action phases (such as preparation, execution, and completion) helps distinguish fine-grained temporal transitions within continuous motions.  

Both the \textit{step-aware (d)} and \textit{attribute-focused (b)} prompts demonstrate stable performance, as emphasizing motion continuity or detailed physical attributes effectively reduces ambiguity and reinforces local consistency in semantic comparison. The \textit{confidence-based prompt (c)} provides moderate improvement by quantifying uncertainty in LLM judgments, which benefits reliability but contributes less to deeper semantic reasoning. In contrast, the simple \textit{baseline prompt (a)} performs the weakest, as it lacks guidance on how the LLM should interpret motion similarity, relying only on its implicit visual understanding.  

Overall, the results reveal a clear semantic progression: as the prompts evolve from generic and perception-based instructions to structured and reasoning-oriented formulations, the inferred temporal semantics become increasingly coherent. This progression (from the baseline to attribute-focused, confidence-based, step-aware, phase-aware, and finally causal-motion prompts) reflects the shift from surface-level perceptual matching toward a deeper, causality-driven understanding of human motion dynamics.

\subsubsection{Performance on Diverse Motion Types}
\label{subsec:differentmotion}
\begin{figure}[t]
	\begin{center}
		\includegraphics[width=1\columnwidth]{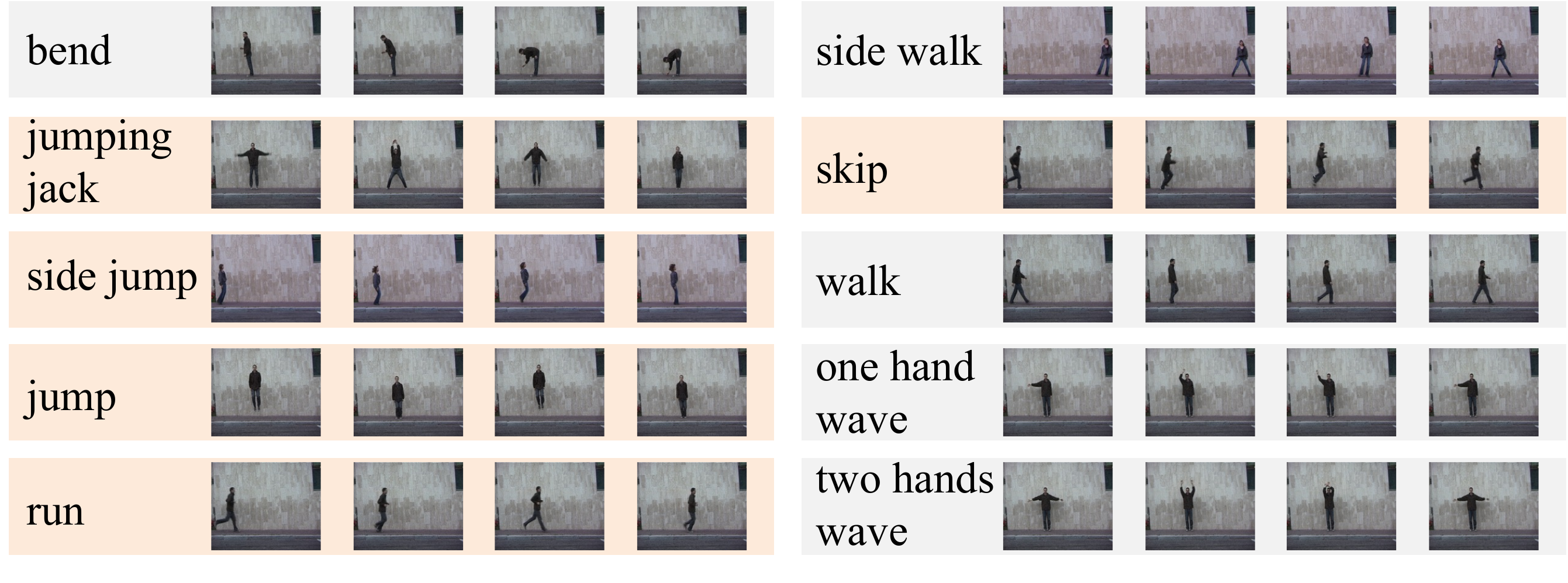}
		\vspace{-0.15cm}
		\caption{Frame examples of the ten motions in the Weiz dataset (person 'ido').}
		\label{fig:Examples}
	\end{center}
	\vspace{-0.25in}
\end{figure}

Representative motion examples in the Weiz dataset are shown in Figure~\ref{fig:Examples}. We observe that motions characterized by slower and more structured movements, such as \textit{bend}, \textit{side walk}, \textit{walk}, \textit{one hand wave}, and \textit{two hands wave}, produce more precise and complete TVS representations. In contrast, faster and more dynamic motions, including \textit{jumping jack}, \textit{side jump}, \textit{jump}, \textit{run}, and \textit{skip}, often result in incomplete TVS coverage due to rapid posture transitions and higher motion variability. As previously discussed for the TVS in Figure \ref{fig:G}(a), such motions may introduce partial inconsistencies in temporal segmentation. For example, frames \emph{271--323} correspond to the same "side jump" motion in the ground truth, yet the LLM identifies frames \emph{271--295} as one continuous motion and frames \emph{306--323} as one continuous motion.

\subsubsection{Limitation}
\label{subsec:LLMlimitation}
A notable limitation of the LLM-driven TVS is its tendency to misidentify a single motion as multiple distinct motions. This issue is evident in Figure \ref{fig:G}(a), where the LLM splits a continuous motion into several segments. While using appropriate prompts can help mitigate this problem, it cannot be completely avoided. As a result, LLM-driven TVS cannot be directly applied to human motion segmentation in a straightforward manner. Additional steps, such as the proposed TVSH, are required. This issue becomes particularly pronounced when a motion involves multiple stages and occurs at high speed, but the camera’s frame rate is low. 
Furthermore, different LLM models may lead to TVS matrices with varying accuracy, and calling the LLM API is typically time-consuming, with each query requiring between 0.36 and 3.69 seconds.

}

		\section{Conclusion}
		\label{sec:conclusion}
		
	In this paper, we introduced a novel feedback-enabled subspace embedding approach for HMS, leveraging TVS embedded in human motion videos. We formulated the subspace embedding problem by integrating a temporal regularizer to capture the underlying temporal structure. Furthermore, we incorporated clustering with a temporal constraint to ensure that the clustering assignments reflect temporal characteristics. Finally, we developed a feedback-enabled framework to optimize the subspace embedding based on the segmentation results. Experimental results on benchmark datasets for HMS consistently demonstrated the superior performance of our approach compared to existing state-of-the-art techniques.

\appendix

\subsection{Proof of Theorem~\ref{thm:graph_tv_equiv}}
\label{sec:proof_graph_tv}

Let $\{\mathcal{N}_i\}_{i=1}^N$ be the temporal neighborhoods and define the (possibly directed) adjacency 
$A\in\{0,1\}^{N\times N}$ by $A_{i\ell}=1$ iff $\ell\in\mathcal{N}_i$, otherwise $A_{i\ell}=0$.
Let $\deg(i)=\sum_{\ell}A_{i\ell}$ and let $\mathbf{Z}=[\mathbf{z}_1,\dots,\mathbf{z}_N]\in\mathbb{R}^{d\times N}$ denote the column-wise embeddings.
Recall the TVS matrix $\mathbf{G}\in\mathbb{R}^{N\times N}$ is defined nodewise by
\begin{equation}
	G_{ii}=-|\mathcal{N}_i|, 
	G_{i\ell}=1\;\; \text{if } \ell\in\mathcal{N}_i,  
	G_{i\ell}=0\;\; \text{otherwise}.
	\label{eq:tvsG_def}
\end{equation}
Hence each row $i$ of $\mathbf{G}$ encodes a \emph{star Laplacian} centered at $i$.

\begin{lemma}[Discrete divergence-of-differences identity]
	\label{lem:ZG_identity}
	For every $i\in\{1,\dots,N\}$,
	\begin{equation}
		(\mathbf{Z}\mathbf{G})_{:i} 
		= \sum_{\ell\in\mathcal{N}_i}(\mathbf{z}_i-\mathbf{z}_\ell).
		\label{eq:ZG_column_form}
	\end{equation}
\end{lemma}

\begin{proof}
	By \eqref{eq:tvsG_def}, the $i$-th column of $\mathbf{G}$ is
	$G_{ii}=-\deg(i)$ and $G_{\ell i}=1$ for $\ell\in\mathcal{N}_i$, zero otherwise. Therefore
	\begin{align*}
		(\mathbf{Z}\mathbf{G})_{:i}
		&= \sum_{j=1}^N \mathbf{z}_j \, G_{ji}
		= \mathbf{z}_i\,G_{ii}
		+ \sum_{\ell\in\mathcal{N}_i} \mathbf{z}_\ell\,G_{\ell i}\\
		&= -\deg(i)\,\mathbf{z}_i + \sum_{\ell\in\mathcal{N}_i}\mathbf{z}_\ell
		= \sum_{\ell\in\mathcal{N}_i}(\mathbf{z}_\ell - \mathbf{z}_i).
	\end{align*}
	Changing sign inside the sum yields \eqref{eq:ZG_column_form}.
\end{proof}

Define, for each node $i$, the \emph{edge-difference stack}
\[
\mathbf{v}_i \;=\; \big[\,(\mathbf{z}_i-\mathbf{z}_\ell)\,:\,\ell\in\mathcal{N}_i \big]\;\in\;\mathbb{R}^{d\,\deg(i)},
\]
i.e., concatenate all incident differences at node $i$. 
The \emph{node-wise isotropic graph total variation} is then
\begin{equation}
	\mathrm{GTV}_{\mathrm{iso}}(\mathbf{Z};A)
	\;=\;
	\sum_{i=1}^N \big\| \mathbf{v}_i \big\|_2
	\;=\;
	\sum_{i=1}^N \Big(\sum_{\ell\in\mathcal{N}_i} \|\mathbf{z}_i-\mathbf{z}_\ell\|_2^2\Big)^{1/2}.
	\label{eq:GTV_iso}
\end{equation}
Using Lemma~\ref{lem:ZG_identity}, we can write
\begin{equation}
	\|\mathbf{Z}\mathbf{G}\|_{2,1}
	=\sum_{i=1}^N \big\|(\mathbf{Z}\mathbf{G})_{:i}\big\|_2
	=\sum_{i=1}^N \Big\| \sum_{\ell\in\mathcal{N}_i}(\mathbf{z}_i-\mathbf{z}_\ell)\Big\|_2.
	\label{eq:ZG_21_as_sum}
\end{equation}
Observe that \eqref{eq:GTV_iso} aggregates edge differences at node $i$ by a \emph{stack-then-$\ell_2$} operation, while \eqref{eq:ZG_21_as_sum} aggregates them by a \emph{sum-then-$\ell_2$} operation.
These two node-wise aggregations are \emph{equivalent up to degree-dependent constants}:
by the triangle inequality and Cauchy–Schwarz, for any collection $\{ \mathbf{a}_\ell\}_{\ell=1}^{m}$,
\[
\Big\|\sum_{\ell=1}^m \mathbf{a}_\ell\Big\|_2 
\;\le\; \sum_{\ell=1}^m \|\mathbf{a}_\ell\|_2
\;\le\; \sqrt{m}\Big(\sum_{\ell=1}^m \|\mathbf{a}_\ell\|_2^2\Big)^{1/2}.
\]
Applying this to $\mathbf{a}_\ell=\mathbf{z}_i-\mathbf{z}_\ell$ with $m=\deg(i)$ gives, nodewise,
\begin{align}
	\Big\| \sum_{\ell\in\mathcal{N}_i}(\mathbf{z}_i-\mathbf{z}_\ell)\Big\|_2
	\;&\le\;
	\sum_{\ell\in\mathcal{N}_i}\|\mathbf{z}_i-\mathbf{z}_\ell\|_2
	\;\nonumber\\
	&\le\;
	\sqrt{\deg(i)}\Big(\sum_{\ell\in\mathcal{N}_i}\|\mathbf{z}_i-\mathbf{z}_\ell\|_2^2\Big)^{1/2}.
	\label{eq:norm_chain}
\end{align}
Summing \eqref{eq:norm_chain} over $i$ yields the sandwich bound
\begin{align}
	\|\mathbf{Z}\mathbf{G}\|_{2,1}
	\;&\le\;
	\sum_{i=1}^N \sum_{\ell\in\mathcal{N}_i}\|\mathbf{z}_i-\mathbf{z}_\ell\|_2
	\;	\label{eq:equiv_up_to_constants}\\
	&\le\;
	\sqrt{\deg_{\max}}\;\mathrm{GTV}_{\mathrm{iso}}(\mathbf{Z};A),
	\deg_{\max}=\max_i \deg(i).\nonumber
\end{align}
Conversely, Jensen's inequality implies
$\big(\sum_{\ell}\|\mathbf{z}_i-\mathbf{z}_\ell\|_2^2\big)^{1/2}
\le
\sum_{\ell}\|\mathbf{z}_i-\mathbf{z}_\ell\|_2$,
and combining with the triangle inequality in the other direction gives constants $c_1,c_2>0$ (depending only on $\deg_{\max}$) such that
\begin{equation}
	c_1\,\mathrm{GTV}_{\mathrm{iso}}(\mathbf{Z};A)
	\;\le\;
	\|\mathbf{Z}\mathbf{G}\|_{2,1}
	\;\le\;
	c_2\,\mathrm{GTV}_{\mathrm{iso}}(\mathbf{Z};A).
	\label{eq:norm_equivalence}
\end{equation}
Therefore, $\|\mathbf{Z}\mathbf{G}\|_{2,1}$ is \emph{equivalent} (up to fixed multiplicative constants on degree-bounded temporal graphs) to the isotropic Graph-TV \eqref{eq:GTV_iso}. This establishes the first statement in Theorem~\ref{thm:graph_tv_equiv}.

Two key properties follow immediately:
	(i) \textit{Zero iff nodewise constancy.}
	By \eqref{eq:ZG_column_form}, $\|\mathbf{Z}\mathbf{G}\|_{2,1}=0$ iff for every $i$, $\sum_{\ell\in\mathcal{N}_i}(\mathbf{z}_i-\mathbf{z}_\ell)=\mathbf{0}$.
	Since all summands are nonnegative in norm and the graph is degree-bounded, this holds iff $\mathbf{z}_i=\mathbf{z}_\ell$ for all $\ell\in\mathcal{N}_i$.
	Thus $\mathbf{z}$ is constant on every connected component induced by $A$ (i.e., within each TVS-consistent segment).
	(ii) \textit{Local smoothness with boundary preservation.}
	Minimizing \eqref{eq:GTV_iso} (and hence $\|\mathbf{Z}\mathbf{G}\|_{2,1}$ by \eqref{eq:norm_equivalence}) penalizes \emph{intra-segment} variations $\|\mathbf{z}_i-\mathbf{z}_\ell\|_2$ along edges while not penalizing differences across \emph{absent} edges. 
	On temporal data, edges are present within action-consistent neighborhoods and absent across motion transitions. 
	Therefore, the minimum-energy configurations are \emph{piecewise-constant} on segments separated by motion boundaries (where edges vanish), exactly capturing the “smooth-inside / sharp-across” behavior typical of Graph-TV.

The identity \eqref{eq:ZG_column_form} and the norm-equivalence \eqref{eq:norm_equivalence} show that $\|\mathbf{Z}\mathbf{G}\|_{2,1}$ is an isotropic Graph-TV (up to degree-dependent constants) on the temporal graph defined by $\{\mathcal{N}_i\}$.
Consequently, minimizing $\|\mathbf{Z}\mathbf{G}\|_{2,1}$ enforces local smoothness within temporal neighborhoods while allowing discontinuities at motion boundaries, yielding piecewise-constant embeddings aligned with human motion transitions, thus proving Theorem~\ref{thm:graph_tv_equiv}.

{\color{black}{
\subsection{Proof of Theorem~\ref{thm:consistency_llm}}
\label{sec:proof_consistency_llm}

Consider a sequence of $N$ frames with true binary adjacency labels 
$\mathrm{eq}^{\star}_k\in\{0,1\}$ on consecutive pairs $(k,k{+}1)$, $k=1,\dots,N{-}1$, where $\mathrm{eq}^{\star}_k=1$ indicates same motion and $\mathrm{eq}^{\star}_k=0$ indicates a true motion boundary. 
Let the observed labels be $\mathrm{eq}_k$, obtained by flipping each $\mathrm{eq}^{\star}_k$ independently with probability $p<\tfrac12$. 
Let the minimum true segment length be $L_{\min}\ge 1$. 
The TVS neighborhoods $\{\mathcal{N}_i\}$ and matrix $\mathbf{G}$ are built from $\{\mathrm{eq}_k\}$ via the rule described in the method section, and embeddings $\mathbf{Z}=[\mathbf{z}_1,\dots,\mathbf{z}_N]$ are obtained by minimizing the TVS regularizer $\|\mathbf{Z}\mathbf{G}\|_{2,1}$ together with other convex terms of the objective.

Define the set of indices of flipped adjacencies 
$\mathcal{F}=\{\,k\in\{1,\dots,N{-}1\}: \mathrm{eq}_k\neq \mathrm{eq}^{\star}_k\,\}$ and let $F=|\mathcal{F}|$. 
By independence and linearity of expectation,
\[
\mathbb{E}[F]=\sum_{k=1}^{N-1}\mathbb{P}(\mathrm{eq}_k\neq \mathrm{eq}^{\star}_k)=p\,(N{-}1)=O(pN).
\]
Each flip can at most create one \emph{spurious} boundary (when a true within-segment edge is flipped from $1$ to $0$) or remove one \emph{true} boundary (when a boundary edge is flipped from $0$ to $1$). 
Hence, if $B_{\text{err}}$ denotes the number of erroneous TVS boundaries inferred from $\{\mathrm{eq}_k\}$, then deterministically $B_{\text{err}}\le F+F=2F$, so 
\[
\mathbb{E}[B_{\text{err}}]\le 2\,\mathbb{E}[F]=2p\,(N{-}1)=O(pN).
\]
This proves the first claim.

A single flipped within-segment edge (changing a run of $1$'s to $1,0,1$ locally) splits a true segment into two pieces whose lengths sum to the original length and differ by at most $1$. 
Since segment lengths are at least $L_{\min}$, any isolated flip creates at worst a one-edge ``notch'' inside a long run. 
Similarly, flipping a boundary edge can merge two adjacent true segments but only across a single location.

By Lemma~(Graph-TV identity) used in Theorem~\ref{thm:graph_tv_equiv}, the TVS penalty can be written nodewise as
\[
\|\mathbf{Z}\mathbf{G}\|_{2,1}=\sum_{i=1}^{N}\Big\|\sum_{\ell\in\mathcal{N}_i}(\mathbf{z}_i-\mathbf{z}_\ell)\Big\|_2,
\]
which is equivalent (up to degree-dependent constants on a degree-bounded temporal graph) to an isotropic graph total variation that penalizes within-neighborhood differences $\|\mathbf{z}_i-\mathbf{z}_\ell\|_2$. 
Consider a true segment $S=\{s,\dots,t\}$ of length $|S|\ge L_{\min}$. 
If all adjacencies inside $S$ are correct ($\mathrm{eq}_k=1$ for $k\in\{s,\dots,t{-}1\}$), the minimal TVS energy within $S$ is attained at \emph{constant} $\mathbf{z}_i$ over $S$ (zero variation). 
If there is a single flip at index $k\in\{s,\dots,t{-}1\}$ producing a spurious cut, the TVS graph inside $S$ loses one local edge near $k$, but all remaining adjacent edges still connect the two sides across many nodes. 
Any nonconstant jump that honors the spurious cut introduces at least one additional nonzero difference along the many surviving edges across the two sides; keeping the embedding constant on $S$ sets all those differences to zero. 
Thus, for an isolated spurious cut, the constant solution on $S$ weakly dominates any split solution in TVS energy.

Formally, let $S$ be partitioned into $S_1=\{s,\dots,k\}$ and $S_2=\{k{+}1,\dots,t\}$ by a single flipped edge at $k$. 
Let $\mathbf{z}_{S_1}$ and $\mathbf{z}_{S_2}$ denote the respective constants of a piecewise-constant candidate on $S_1,S_2$. 
For every surviving edge $(i,\ell)$ with $i\in S_1$, $\ell\in S_2$ that remains in $\mathcal{N}_i$ or $\mathcal{N}_\ell$ (these are the edges not removed by the single flip and they are $\Omega(|S|)$ many when $L_{\min}$ is large), the TVS contribution adds $\|\mathbf{z}_{S_1}-\mathbf{z}_{S_2}\|_2$. 
Hence the total TVS cost increases by at least $c\,\|\mathbf{z}_{S_1}-\mathbf{z}_{S_2}\|_2$ for some $c=\Omega(L_{\min})$. 
Setting $\mathbf{z}_{S_1}=\mathbf{z}_{S_2}$ brings this increase to zero, so the minimum is achieved by the constant solution unless opposed by a substantially large data term. 
In our stated theorem, the conclusion is in \emph{expectation} and for \emph{small} $p$ with \emph{large} $L_{\min}$; the probability that multiple adjacent flips accumulate to remove most cross-side edges within $S$ decays geometrically in the number of required flips, hence their contribution is negligible for small $p$.

Combining Steps 2–3 with the bound $\mathbb{E}[B_{\text{err}}]=O(pN)$, the flips appear sparsely along the chain for small $p$. 
With high probability, the flips are isolated at $\Theta(1/p)$ spacing, while segment lengths are at least $L_{\min}$. 
When $L_{\min}$ is sufficiently large relative to the typical spacing and the TVS weight is positive, the TVS minimization favors constant embeddings over each true segment and suppresses the local notches caused by isolated flips. 
Therefore the recovered embeddings are piecewise-constant on the true segments in expectation, yielding \emph{segment-level consistency} for small $p$ and large $L_{\min}$.

The expected number of erroneous TVS boundaries scales as $O(pN)$ by linearity of expectation. 
Moreover, because $\|\mathbf{Z}\mathbf{G}\|_{2,1}$ acts as an isotropic graph total variation on the temporal graph, isolated boundary errors are smoothed out by the regularizer on sufficiently long segments, so that the final embeddings remain piecewise constant per true segment in expectation when $p$ is small and $L_{\min}$ is large. 
This completes the proof of Theorem~\ref{thm:consistency_llm}.

\subsection{Proof of Proposition~\ref{prop:BDconstraint}}
\label{app:prop-BDconstraint}
We first consider the first two term of \eqref{Prob:TVSH}, which is denoted as
\begin{small}
	\begin{align*}
		f(\mathbf{Z}) 
		=\|{\mathbf{X}}{\mathbf{Z}}-{\mathbf{X}}\|_{\text{F}}^{2}+\|\mathbf{Z}^{\mathrm{T}}\mathbf{Z}\|_{1} =\|{\mathbf{X}}{\mathbf{Z}}-{\mathbf{X}}\|_{\text{F}}^{2}+\mathbf{e}^{\mathrm{T}}{\mathbf{Z}}^{\mathrm{T}}{\mathbf{Z}}\mathbf{e}
	\end{align*}
\end{small}
The columns of $\mathbf{X}$
are in general position: $\mathbf{X}=[\mathbf{X}_{1},\mathbf{X}_{2},...,\mathbf{X}_{K}]$,
where all the columns of submatrix $\mathbf{X}_{\alpha}$ lie in the
same subspace $\mathcal{S}_{\alpha}$.

Assume $\mathbf{Z}^{*}$ minimizes the function $f(\mathbf{Z})$, and we decompose
$\mathbf{Z}^{*}$ to be the sum of two matrices
\begin{footnotesize}
	\begin{align*}
		\mathbf{Z}^{*} & =\mathbf{Z}^{D}+\mathbf{Z}^{C}\\
		& =\left[\begin{array}{cccc}
			\mathbf{Z}_{11}^{*} &  &  & \mathbf{0}\\
			& \mathbf{Z}_{22}^{*}\\
			&  & \ddots\\
			\mathbf{0} &  &  & \mathbf{Z}_{KK}^{*}
		\end{array}\right]+\left[\begin{array}{cccc}
			\mathbf{0} & \mathbf{Z}_{12}^{*} & \cdots & \mathbf{Z}_{1K}^{*}\\
			\mathbf{Z}_{21}^{*} & \mathbf{0} & \cdots & \mathbf{Z}_{2K}^{*}\\
			\vdots &  & \ddots & \vdots\\
			\mathbf{Z}_{K1}^{*} & \mathbf{Z}_{K2}^{*} & \cdots & \mathbf{0}
		\end{array}\right]
	\end{align*}
\end{footnotesize}
where $\mathbf{Z}_{ij}^{*}\in\mathbb{R}^{\mathcal{N}_{i}\times \mathcal{N}_{j}}$. Note
that both $\mathbf{Z}^{D}$ and $\mathbf{Z}^{C}$ are non-negative.

According to the decomposition of $\mathbf{Z}^{*}$, any column of
$\mathbf{Z}^{*}$ can be written as $\mathbf{z}_{i}^{*}=\mathbf{z}_{i}^{D}+\mathbf{z}_{i}^{C}$,
with $\mathbf{z}_{i}^{D}$ and $\mathbf{z}_{i}^{C}$ supported on
disjointed subset of indices. We can write $\|\mathbf{X}\mathbf{Z}^{*}-\mathbf{X}\|_{\text{F}}^{2}$
as
\begin{footnotesize}
	\begin{align*}
		& \|\mathbf{X}\mathbf{Z}^{*}-\mathbf{X}\|_{\text{F}}^{2}
		=  \sum_{i=1}^{N}\|\mathbf{X}\mathbf{z}_{i}^{*}-\mathbf{x}_{i}\|_{2}^{2}
		= \sum_{i=1}^{N}\|\mathbf{X}\mathbf{z}_{i}^{D}+\mathbf{X}\mathbf{z}_{i}^{C}-\mathbf{x}_{i}\|_{2}^{2}\\
		= & \sum_{i=1}^{N}\|\mathbf{X}\mathbf{z}_{i}^{D}-\mathbf{x}_{i}\|_{2}^{2}+\sum_{i=1}^{N}\|\mathbf{X}\mathbf{z}_{i}^{C}\|_{2}^{2}+2\sum_{i=1}^{N}\mathrm{cos}\theta_{i}\|\mathbf{X}\mathbf{z}_{i}^{D}-\mathbf{x}_{i}\|_{2}\|\mathbf{X}\mathbf{z}_{i}^{C}\|_{2}
	\end{align*}
\end{footnotesize}
where $\theta_{i}$ is the angle between vector $\mathbf{X}\mathbf{z}_{i}^{D}-\mathbf{x}_{i}$
and $\mathbf{X}\mathbf{z}_{i}^{C}$.

Since the matrix $\mathbf{X}=[\mathbf{X}_{1},\mathbf{X}_{2},...,\mathbf{X}_{K}]$
is well arranged, any column $\mathbf{x}_{i}\in\mathbf{X}_{\alpha}$
and $\mathbf{x}_{j}\in\mathbf{X}_{\beta}$ lie in different subspaces
if $\alpha\neq\beta$. Let $\mathbf{x}_{i}\in\mathcal{S}_{\alpha}$,
according to the definition of $\mathbf{z}_{i}^{D}$ and $\mathbf{z}_{i}^{C}$,
we have $\mathbf{X}\mathbf{z}_{i}^{D}\in\mathcal{S}_{\alpha}$ and
$\mathbf{X}\mathbf{z}_{i}^{C}\notin\mathcal{S}_{\alpha}$. Based on
the orthogonal subspace assumption, we have $(\mathbf{X}\mathbf{z}_{i}^{D}-\mathbf{x}_{i})\perp\mathbf{X}\mathbf{z}_{i}^{C}$and
$\theta_{i}=\pi/2$, thus
\begin{align}
	\|\mathbf{X}\mathbf{Z}^{*}-\mathbf{X}\|_{\text{F}}^{2} & =\|\mathbf{X}\mathbf{Z}^{D}-\mathbf{X}\|_{\text{F}}^{2}+\|_{2}\|\mathbf{X}\mathbf{Z}^{C}\|_{\text{F}}^{2}\nonumber \\
	& \geq\|\mathbf{X}\mathbf{Z}^{D}-\mathbf{X}\|_{\text{F}}^{2}\label{eq:e7}
\end{align}

Based on the nonnegativity of $\mathbf{Z}^{*}$, $\mathbf{Z}^{C}$,
and $\mathbf{Z}^{D}$, we have
\begin{small}
	\begin{align}
		&\|(\mathbf{Z}^{*})^{\mathrm{T}}\mathbf{Z}^{*}\|_{1}\nonumber\\& =\sum_{i,j}|(\mathbf{z}_{i}^{*})^{\mathrm{T}}\mathbf{z}_{j}^{*}| =\sum_{i,j}(\mathbf{z}_{i}^{*})^{\mathrm{T}}\mathbf{z}_{j}^{*} =\sum_{i,j}(\mathbf{z}_{i}^{C}+\mathbf{z}_{i}^{D})^{\mathrm{T}}(\mathbf{z}_{j}^{C}+\mathbf{z}_{j}^{D})\nonumber \\
		& \geq\sum_{i,j}(\mathbf{z}_{i}^{D})^{\mathrm{T}}\mathbf{z}_{j}^{D}+\sum_{i,j}(\mathbf{z}_{i}^{C})^{\mathrm{T}}\mathbf{z}_{j}^{C} =\|(\mathbf{z}^{D})^{\mathrm{T}}\mathbf{z}^{D}\|_{1}+\|(\mathbf{z}^{C})^{\mathrm{T}}\mathbf{z}^{C}\|_{1}\nonumber\\& \geq\|(\mathbf{z}^{D})^{\mathrm{T}}\mathbf{z}^{D}\|_{1}\label{eq:e8}
	\end{align}
\end{small}
From inequalities (\ref{eq:e7}) and (\ref{eq:e8}) we have $f(\mathbf{Z}^{*})\geq f(\mathbf{Z}^{D})$.
Because $\mathbf{Z}_{ij}^{*}\in\mathbb{R}^{\mathcal{N}_{i}\times \mathcal{N}_{j}}$, we have $f(\mathbf{Z}^{*})=f(\mathbf{Z}^{D})$
and $\mathbf{Z}^{C}=\mathbf{0}$, thus $\mathbf{Z}^{*}=\mathbf{Z}^{D}$.

We then consider the third term in \eqref{Prob:TVSH}, namely, $g(\mathbf{Z}) = \lambda_2 \|\mathbf{Z} \mathbf{G}\|_{2,1} = \sum_{i=1}^{N}\sum_{l \in \mathcal{N}_{i}} \|\mathbf{z}_{i} - \mathbf{z}_{l}\|_{2}$. Given the subspace assumption for the samples, the construction of the neighbor set $\mathcal{N}_{i}$ for the $i$th sample based on cosine measurements captures all the temporal neighbors of the $i$th sample. It is therefore straightforward to demonstrate that $g(\mathbf{Z}) \geq \sum_{i=1}^{N}\sum_{l \in \mathcal{N}_{i}} \|\mathbf{z}_{i}^{*} - \mathbf{z}_{l}^{*}\|_{2}$,
confirming that $\mathbf{Z}^{*}$ also minimizes $g(\mathbf{Z})$.

Finally, we address the last term in \eqref{Prob:TVSH}. With $\mathbf{Z}$ being block-diagonal, the $(i,j)$th element of $\mathbf{Z}$ is nonzero if and only if the $i$th and $j$th samples are situated in the same subspace. Consequently, $\|\mathbf{Z}^{*}\|_{\mathbf{Q}^{*}} = 0$.

Thus, Proposition~\ref{prop:BDconstraint} is upheld.

\subsection{Proof of Theorem~\ref{thm:tvs_influence_simplified}}
\label{sec:proof_tvs_influence}

Let the sequence be partitioned into ground-truth motion segments
$\{\mathcal{S}_g\}_{g=1}^K$, where each segment $\mathcal{S}_g$ has length
$|\mathcal{S}_g|\ge L_{\min}$ and generates observations from a linear
subspace $\mathcal{U}_g\subset\mathbb{R}^D$ with within-segment variance
$\sigma^2$. The between-subspace separation is
$\Delta_{\text{sub}}^2 := \min_{g\neq h}\mathrm{dist}^2(\mathcal{U}_g,\mathcal{U}_h)>0$.
Let $\mathbf{Z}=[\mathbf{z}_1,\dots,\mathbf{z}_N]$ denote the learned
subspace embeddings and $\mathbf{Q}$ the cluster indicator matrix obtained by
solving problem~(\ref{Prob:TVSH}) with the TVS matrix $\mathbf{G}$.
Assume LLM adjacency on consecutive pairs is independently flipped with
probability $p<\tfrac12$ when constructing $\mathbf{G}$.
We evaluate segmentation error $\mathrm{Err}_{\text{HMS}}$ as the normalized
mis-segmentation rate (e.g., the Hamming error of predicted boundaries or the
misclustering fraction, up to label permutation).

The objective~(\ref{Prob:TVSH}) couples a data-fitting term and two regularizers:
(i) the \emph{subspace fidelity} driven by
$\|\mathbf{X}-\mathbf{X}\mathbf{Z}\|_F^2+\|\mathbf{Z}^\top\mathbf{Z}\|_1$ and the
clustering regularizer $\|\mathbf{Z}\|_{\mathbf{Q}}$; and
(ii) the \emph{temporal smoothness} driven by the TVS term
$\lambda_G\|\mathbf{Z}\mathbf{G}\|_{2,1}$. Accordingly, we bound
\[
\mathbb{E}\!\left[\mathrm{Err}_{\text{HMS}}\right]
~\le~
\underbrace{\mathbb{E}\!\left[\mathrm{Err}_{\text{TVS}}\right]}_{\text{temporal adjacency noise}}
~+~
\underbrace{\mathbb{E}\!\left[\mathrm{Err}_{\text{sub}}\right]}_{\text{subspace separability/noise}},
\]
and control each component in turn.

Let $\mathrm{eq}^\star_k\in\{0,1\}$ be the true adjacency on $(k,k{+}1)$ and
$\mathrm{eq}_k$ the observed (noisy) label used to build $\mathbf{G}$. The number
of flipped adjacencies $F=\sum_{k=1}^{N-1}\mathbb{I}\{\mathrm{eq}_k\neq\mathrm{eq}^\star_k\}$
satisfies $\mathbb{E}[F]=p(N{-}1)$ by independence. Each flip can create at most
one spurious cut or remove one true cut, hence the number of erroneous TVS
boundaries $B_{\text{err}}\le 2F$ and
$\mathbb{E}[B_{\text{err}}]=O(pN)$.

However, the TVS penalty is an isotropic graph total-variation on the temporal
graph (Theorem~\ref{thm:graph_tv_equiv}):
\[
\|\mathbf{Z}\mathbf{G}\|_{2,1}
= \sum_{i=1}^N \Big\|\sum_{\ell\in\mathcal{N}_i}(\mathbf{z}_i-\mathbf{z}_\ell)\Big\|_2,
\]
which favors constant embeddings on long runs and penalizes isolated notches.
Inside a true segment $\mathcal{S}_g$ with $|\mathcal{S}_g|\ge L_{\min}$,
a single flipped edge breaks one local link but leaves $\Omega(L_{\min})$ many
cross-links among neighbors intact; any nonconstant split of $\mathbf{Z}$ across
that notch incurs at least $c\,\|\Delta\|_2$ extra TVS cost with
$c=\Omega(L_{\min})$. Thus, for sufficiently large $\lambda_G$ (bounded away from
zero and not exceeding the data term scale), the optimizer prefers to \emph{heal}
isolated flips and keep $\mathbf{z}_i$ constant within $\mathcal{S}_g$.
Since the flips are sparse in expectation and segments are long, the fraction of
frames affected at the \emph{segment level} scales as the number of flips
divided by the segment length, yielding
\[
\mathbb{E}\!\left[\mathrm{Err}_{\text{TVS}}\right]
~\le~ C_1\,\frac{p}{L_{\min}},
\]
for a constant $C_1$ depending on neighborhood width and the TVS weight.

Within each true segment, the data lie near a subspace $\mathcal{U}_g$ with
variance $\sigma^2$, and different segments correspond to subspaces separated by
$\Delta_{\text{sub}}^2$. Standard perturbation arguments for subspace clustering
(and nearest-subspace assignment) imply a misassignment probability bounded by
$C_2\,\sigma^2/\Delta_{\text{sub}}^2$ when the separation dominates the noise
(large-margin regime). In our formulation, the terms
$\|\mathbf{X}-\mathbf{X}\mathbf{Z}\|_F^2+\|\mathbf{Z}^\top\mathbf{Z}\|_1$ and
$\|\mathbf{Z}\|_{\mathbf{Q}}$ promote embeddings that are (approximately)
block-sparse/diagonal across subspaces and cluster-coherent; hence the induced
clustering error satisfies
\[
\mathbb{E}\!\left[\mathrm{Err}_{\text{sub}}\right]
~\le~ C_2\,\frac{\sigma^2}{\Delta_{\text{sub}}^2},
\]
with $C_2$ depending on the regularization weights and the conditioning of
$\{\mathcal{U}_g\}$.

By the decomposition above,
\[
\mathbb{E}\!\left[\mathrm{Err}_{\text{HMS}}\right]
~\le~
C_1\,\frac{p}{L_{\min}}
~+~
C_2\,\frac{\sigma^2}{\Delta_{\text{sub}}^2},
\]
which proves the claimed bound.

If the TVS neighborhoods $(l_i,r_i)$ coincide with true segments, then the TVS
graph has no spurious cross-segment edges. Minimizing
$\|\mathbf{Z}\mathbf{G}\|_{2,1}$ forces $\mathbf{z}_i$ to be constant within
each segment, and the data-fitting plus sparsity terms make inter-segment
connections in $\mathbf{Z}$ suboptimal when $\Delta_{\text{sub}}^2>0$.
Therefore, at any optimum $(\mathbf{Z}^*,\mathbf{Q}^*)$, the matrix $\mathbf{Z}^*$
is (after permutation) block-diagonal with blocks aligned to
$\{\mathcal{S}_g\}$, which yields exact segmentation up to label permutation.

The constants $C_1,C_2$ absorb factors due to neighborhood width, TVS weight
$\lambda_G$, and spectral clustering relaxation tightness. The bound holds for
any fixed choice of regularization weights in a compact interval
$[\underline{\lambda}_G,\overline{\lambda}_G]$ that preserves the healing effect
of TVS while not overwhelming the data-fitting terms.

\subsection{Proof of Proposition~\ref{prop:H-global}}
\label{sec:proof_H_global}

For fixed $(\mathbf{Z},\mathbf{Q},\mathbf{F},\gamma)$, the $\mathbf{H}$-subproblem
\begin{equation}
	\label{subProb:U-again}
	\min_{\mathbf{H}}\;
	\|\mathbf{H}\|_{2,1}
	+\big\langle \mathbf{F},\,\mathbf{H}-\mathbf{Z}\mathbf{G}\big\rangle
	+\frac{\gamma}{2}\,\|\mathbf{H}-\mathbf{Z}\mathbf{G}\|_{F}^{2}
\end{equation}
is the proximal operator of the $\ell_{2,1}$ norm and admits the closed-form group-shrinkage solution. Hence the $\mathbf{H}$-update attains the global minimizer of \eqref{subProb:U-again} at every iteration.

Define
\[
\mathbf{P}\;\coloneqq\;\mathbf{Z}\mathbf{G}-\frac{1}{\gamma}\mathbf{F}.
\]
Expanding the quadratic and completing the square yields
\begin{align*}
	&\frac{\gamma}{2}\|\mathbf{H}-\mathbf{Z}\mathbf{G}\|_{F}^{2}
	+\big\langle \mathbf{F},\,\mathbf{H}-\mathbf{Z}\mathbf{G}\big\rangle\\
	&=
	\frac{\gamma}{2}\Big\|\mathbf{H}-\Big(\mathbf{Z}\mathbf{G}-\tfrac{1}{\gamma}\mathbf{F}\Big)\Big\|_{F}^{2}
	-\frac{1}{2\gamma}\|\mathbf{F}\|_{F}^{2}\\
	&=
	\frac{\gamma}{2}\|\mathbf{H}-\mathbf{P}\|_{F}^{2}
	+\text{const}.
\end{align*}
Since the additive constant does not affect the minimizer, problem \eqref{subProb:U-again} is equivalent to
\begin{equation}
	\label{eq:prox-form}
	\min_{\mathbf{H}}\;\|\mathbf{H}\|_{2,1}+\frac{\gamma}{2}\|\mathbf{H}-\mathbf{P}\|_{F}^{2},
\end{equation}
which is the proximal mapping of the $\ell_{2,1}$ norm at point $\mathbf{P}$ with parameter $1/\gamma$:
\[
\mathbf{H}^{\star}
=\operatorname{prox}_{(1/\gamma)\|\cdot\|_{2,1}}(\mathbf{P})
\;\;=\;\;\arg\min_{\mathbf{H}}\;
\frac{\gamma}{2}\|\mathbf{H}-\mathbf{P}\|_{F}^{2}+\|\mathbf{H}\|_{2,1}.
\]

Let $\mathbf{H}=[\mathbf{h}_1,\ldots,\mathbf{h}_N]$ and $\mathbf{P}=[\mathbf{p}_1,\ldots,\mathbf{p}_N]$ denote the column partitions. Since
\[
\|\mathbf{H}\|_{2,1}=\sum_{i=1}^{N}\|\mathbf{h}_i\|_2,\qquad
\|\mathbf{H}-\mathbf{P}\|_{F}^{2}=\sum_{i=1}^{N}\|\mathbf{h}_i-\mathbf{p}_i\|_2^{2},
\]
the objective in \eqref{eq:prox-form} decomposes into $N$ independent vector problems:
\begin{equation}
	\label{eq:col-prox}
	\mathbf{h}_i^{\star}
	=\arg\min_{\mathbf{h}\in\mathbb{R}^D}\;
	\|\mathbf{h}\|_2+\frac{\gamma}{2}\|\mathbf{h}-\mathbf{p}_i\|_2^{2}
	\qquad (i=1,\ldots,N).
\end{equation}
Thus, it suffices to solve \eqref{eq:col-prox} for a single column and apply the solution to each $i$.

Consider the convex function $\phi(\mathbf{h})=\|\mathbf{h}\|_2+\frac{\gamma}{2}\|\mathbf{h}-\mathbf{p}\|_2^2$ with $\mathbf{p}\in\mathbb{R}^D$ fixed. A vector $\mathbf{h}^\star$ is optimal iff
\[
\mathbf{0}\;\in\;\partial\|\mathbf{h}^\star\|_2\;+\;\gamma(\mathbf{h}^\star-\mathbf{p}).
\]
The subdifferential of $\|\cdot\|_2$ is
\[
\partial\|\mathbf{h}\|_2=
\begin{cases}
	\displaystyle\left\{\frac{\mathbf{h}}{\|\mathbf{h}\|_2}\right\}, & \mathbf{h}\neq\mathbf{0},\\[0.5em]
	\left\{\mathbf{u}\in\mathbb{R}^D:\|\mathbf{u}\|_2\le 1\right\}, & \mathbf{h}=\mathbf{0}.
\end{cases}
\]
\emph{Case A ($\mathbf{h}^\star\neq\mathbf{0}$).} Then there exists $\alpha>0$ such that $\mathbf{h}^\star=\alpha\mathbf{p}$ (the solution must align with $\mathbf{p}$ by symmetry). Plugging into the optimality condition:
\begin{align*}
	\mathbf{0}&=\frac{\mathbf{h}^\star}{\|\mathbf{h}^\star\|_2}+\gamma(\mathbf{h}^\star-\mathbf{p})
	=\frac{\alpha\mathbf{p}}{\alpha\|\mathbf{p}\|_2}+\gamma(\alpha-1)\mathbf{p}\nonumber\\
	&=\left(\frac{1}{\|\mathbf{p}\|_2}+\gamma(\alpha-1)\right)\mathbf{p},
\end{align*}
which yields
\[
\alpha=1-\frac{1}{\gamma\|\mathbf{p}\|_2}.
\]
Feasibility requires $\alpha>0$, i.e., $\|\mathbf{p}\|_2>1/\gamma$.

\emph{Case B ($\mathbf{h}^\star=\mathbf{0}$).} The optimality condition becomes
$\mathbf{0}\in\partial\|\mathbf{0}\|_2-\gamma\mathbf{p}$, i.e., there exists
$\mathbf{u}$ with $\|\mathbf{u}\|_2\le 1$ such that $\mathbf{u}=\gamma\mathbf{p}$, which is possible iff $\|\mathbf{p}\|_2\le 1/\gamma$.

Combining the two cases gives the \emph{block (group) soft-thresholding} operator
\[
\mathbf{h}^\star
=\begin{cases}
	\displaystyle\left(1-\frac{1}{\gamma\|\mathbf{p}\|_2}\right)\mathbf{p}, & \text{if }\;\|\mathbf{p}\|_2> \frac{1}{\gamma},\\[0.8em]
	\mathbf{0}, & \text{otherwise}.
\end{cases}
\]
Applying this column-wise with $\mathbf{p}_i=\mathbf{P}_{:i}$ yields
\begin{equation}
	\label{eq:group-shrink}
		\mathbf{h}_i^{\star}
		=\max\!\left(1-\frac{1}{\gamma\|\mathbf{p}_i\|_2},\,0\right)\mathbf{p}_i,
		\qquad i=1,\ldots,N.
\end{equation}

The objective in \eqref{eq:prox-form} is strictly convex in each column due to the strongly convex quadratic term. Therefore, the solution in \eqref{eq:group-shrink} is the unique global minimizer for each subproblem \eqref{eq:col-prox}; stacking the columns gives the unique global minimizer of \eqref{eq:prox-form}, and hence of \eqref{subProb:U-again}. This establishes that the $\mathbf{H}$-update is a \emph{global} proximal step at every iteration, proving Proposition~\ref{prop:H-global}.

\subsection{Proof of Proposition \ref{prop:equalForm}}
\label{app:prop:equalForm}
To prove that the given problem is equivalent to the spectral clustering problem, that is, solving \(\min_{\mathbf{Q}} \ \text{Tr}(\mathbf{Q}^\top (\mathbf{D} - \mathbf{A}) \mathbf{Q})\), where \(\mathbf{D}\) is a diagonal matrix with elements \(\mathbf{D}_{j,j} = \sum_{i} A_{i,j}\), we begin by expanding the objective function. Recall that our problem is \(\min_{\mathbf{Q}} \quad \frac{1}{2}\sum_{i,j} A_{i,j} \|\mathbf{q}_i - \mathbf{q}_j\|_2^2\), which can be expanded as \(\frac{1}{2}\sum_{i,j} A_{i,j} \|\mathbf{q}_i - \mathbf{q}_j\|_2^2 = \frac{1}{2}\sum_{i,j} A_{i,j} (\mathbf{q}_i^\top \mathbf{q}_i - 2 \mathbf{q}_i^\top \mathbf{q}_j + \mathbf{q}_j^\top \mathbf{q}_j)\). This expression can be broken into three parts: \(\frac{1}{2}\sum_{i,j} A_{i,j} \mathbf{q}_i^\top \mathbf{q}_i -  \sum_{i,j} A_{i,j} \mathbf{q}_i^\top \mathbf{q}_j + \frac{1}{2}\sum_{i,j} A_{i,j} \mathbf{q}_j^\top \mathbf{q}_j\). Now, observing each part, we can express it in matrix form. First, \(\sum_{i,j} A_{i,j} \mathbf{q}_i^\top \mathbf{q}_i = \sum_i \mathbf{q}_i^\top \mathbf{q}_i \sum_j A_{i,j} = \sum_i \mathbf{q}_i^\top \mathbf{q}_i \mathbf{D}_{i,i} = \text{Tr}(\mathbf{Q}^\top \mathbf{D} \mathbf{Q})\), and similarly, \(\sum_{i,j} A_{i,j} \mathbf{q}_j^\top \mathbf{q}_j = \sum_j \mathbf{q}_j^\top \mathbf{q}_j \sum_i A_{i,j} = \sum_j \mathbf{q}_j^\top \mathbf{q}_j D_{j,j} = \text{Tr}(\mathbf{Q}^\top \mathbf{D} \mathbf{Q})\). Finally, we have \(-  \sum_{i,j} A_{i,j} \mathbf{q}_i^\top \mathbf{q}_j = - \ \text{Tr}(\mathbf{Q}^\top \mathbf{A} \mathbf{Q})\). Thus, combining these parts, we arrive at the following: \(\frac{1}{2}\sum_{i,j} A_{i,j} \|\mathbf{q}_i - \mathbf{q}_j\|_2^2 = \frac{1}{2}\text{Tr}(\mathbf{Q}^\top \mathbf{D} \mathbf{Q}) + \frac{1}{2}\text{Tr}(\mathbf{Q}^\top \mathbf{D} \mathbf{Q}) -  \ \text{Tr}(\mathbf{Q}^\top \mathbf{A} \mathbf{Q}) = \text{Tr}(\mathbf{Q}^\top \mathbf{D} \mathbf{Q}) - \text{Tr}(\mathbf{Q}^\top \mathbf{A} \mathbf{Q}) = \text{Tr}(\mathbf{Q}^\top (\mathbf{D} - \mathbf{A}) \mathbf{Q})\). This establishes the equivalence between the given problem and the spectral clustering problem.

\subsection{Proof of Proposition~\ref{thm:Kmconvert}}
\label{app:Kmconvert}
The cost associated with a sample located in the \( k \)-th cluster, whose center is denoted by \( \bm{\mu}_k \), can be expressed as the sum of the squared Euclidean distances between the sample and the cluster center, augmented by a term accounting for the influence of neighboring samples. Specifically, the cost is given by \( \sum_{i \in \mathcal{C}_k} \left( || \mathbf{u}_i - \bm{\mu}_k ||_2^2 + \eta \sum_{j \in \mathcal{N}_i} || \mathbf{u}_j - \bm{\mu}_k ||_2^2 \right) \), which can be expanded as \( \sum_{i=1}^{N} \mathbbm{1}(i \in \mathcal{C}_k) || \mathbf{u}_i - \bm{\mu}_k ||_2^2 + \sum_{i=1}^{N} \mathbbm{1}(i \in \mathcal{C}_k) \sum_{j=1}^{N} \mathbbm{1}(j \in \mathcal{N}_i) || \mathbf{u}_j - \bm{\mu}_k ||_2^2 \). To simplify, we define \( n_k(i) \), the number of times the \( i \)-th frame is considered a neighbor of samples in the \( k \)-th cluster, as \( n_k(i) = \sum_{j \in \mathcal{C}_k} \mathbbm{1}(i \in \mathcal{N}_j) \). Substituting this into the previous expression, the cost becomes \( \sum_{i=1}^{N} \mathbbm{1}(i \in \mathcal{C}_k) || \mathbf{u}_i - \bm{\mu}_k ||_2^2 + \sum_{j=1}^{N} \eta n_k(j) || \mathbf{u}_j - \bm{\mu}_k ||_2^2 \). This formulation can be further compacted into the following expression \( \sum_{i=1}^{N} || \mathbf{u}_i - \bm{\mu}_k ||_2^2 \left( \mathbbm{1}(i \in \mathcal{C}_k) + \eta n_k(i) \right) \). This equation reveals the total cost, which is the sum of the direct distance between each sample and the cluster center, and the weighted influence of its neighbors, with the weight determined by \( \eta \). The term \( n_k(i) \) quantifies how many times sample \( i \) is considered a neighbor within the \( k \)-th cluster.

\subsection{Proof of Theorem~\ref{thm:convergence}}
\label{sec:proof_convergence}

Recall the augmented formulation
\begin{align}
	\min_{\mathbf{Z},\mathbf{H},\mathbf{Q}}~
	& \Phi(\mathbf{Z},\mathbf{H},\mathbf{Q})
	~:=~
	\underbrace{\|\mathbf{X}-\mathbf{X}\mathbf{Z}\|_{F}^{2}+\|\mathbf{Z}^{\top}\mathbf{Z}\|_{1}}_{\triangleq f(\mathbf{Z})}
	~\nonumber\\
	&+~
	\underbrace{\|\mathbf{H}\|_{2,1}}_{\triangleq g(\mathbf{H})}
	~+~
	\underbrace{\|\mathbf{Z}\|_{\mathbf{Q}}+\iota_{\mathcal{Q}}(\mathbf{Q})}_{\triangleq h(\mathbf{Q})}
	\label{eq:main_objective}
	\\[-0.2em]
	\text{s.t.}&\quad \mathbf{H}=\mathbf{Z}\mathbf{G},\;\;\mathrm{diag}(\mathbf{Z})=0,\;\;\mathbf{Z}\ge 0,\nonumber
\end{align}
where $\iota_{\mathcal{Q}}$ is the indicator of the feasible set $\mathcal{Q}$. The (scaled) augmented Lagrangian is
\begin{align}
	\label{eq:augLag}
	\mathcal{L}_{\gamma}(\mathbf{Z},\mathbf{H},\mathbf{Q};\mathbf{F})
	~&:=~
	f(\mathbf{Z})+g(\mathbf{H})+h(\mathbf{Q})
	~+~
	\langle \mathbf{F},\,\mathbf{H}-\mathbf{Z}\mathbf{G}\rangle
	~\nonumber\\
	&+~
	\frac{\gamma}{2}\,\|\mathbf{H}-\mathbf{Z}\mathbf{G}\|_{F}^{2},
\end{align}
with $\gamma>0$ and multiplier $\mathbf{F}$. One outer iteration performs:
\begin{itemize}
	\item[(i)] $\mathbf{Z}$-update: minimize $\mathcal{L}_{\gamma_t}(\cdot,\mathbf{H}^{t},\mathbf{Q}^{t};\mathbf{F}^{t})$ over $\mathbf{Z}\in\mathcal{Z}:=\{\mathrm{diag}(\mathbf{Z})=0,\ \mathbf{Z}\ge 0\}$ (via a Sylvester step followed by the projection $\Pi_{\mathcal{Z}}$).
	\item[(ii)] $\mathbf{H}$-update: minimize $\mathcal{L}_{\gamma_t}(\mathbf{Z}^{t+1},\cdot,\mathbf{Q}^{t};\mathbf{F}^{t})$ over $\mathbf{H}$, i.e., a proximal $\ell_{2,1}$ step with closed form.
	\item[(iii)] $\mathbf{Q}$-update: minimize $h(\mathbf{Q})$ for fixed $\mathbf{Z}^{t+1}$ (normalized-cut relaxation with TVS).
	\item[(iv)] Dual update: $\mathbf{F}^{t+1}=\mathbf{F}^{t}+\gamma_t\,(\mathbf{H}^{t+1}-\mathbf{Z}^{t+1}\mathbf{G})$; update $\gamma_{t+1}\ge \gamma_t$, $\gamma_t\to\gamma_\infty\in(0,\infty)$; keep $\rho>1$ bounded.
\end{itemize}

We assume:
\begin{enumerate}
	\item[{A1}] $\mathcal{L}_{\gamma}$ is proper, lower-semicontinuous, and bounded below on the feasible set.
	\item[{A2}] Each block subproblem admits a minimizer; the $\mathbf{Z}$-step is solved exactly for the quadratic subproblem followed by the exact projection onto $\mathcal{Z}$; the $\mathbf{H}$-step is the exact proximal minimizer; the $\mathbf{Q}$-step attains a (relaxed) global minimizer of $h(\cdot)$ or at least a value not exceeding $h(\mathbf{Q}^t)$.
	\item[{A3}] $\{\gamma_t\}$ is nondecreasing with $\gamma_t\to\gamma_\infty\in(0,\infty)$, and the penalty growth factor $\rho>1$ is bounded.
\end{enumerate}

\begin{lemma}[Blockwise descent]
	\label{lem:descent}
	For any $t$, the updates (i)–(iii) satisfy
	\[
	\mathcal{L}_{\gamma_t}(\mathbf{Z}^{t+1},\mathbf{H}^{t},\mathbf{Q}^{t};\mathbf{F}^{t})
	~\le~
	\mathcal{L}_{\gamma_t}(\mathbf{Z}^{t},\mathbf{H}^{t},\mathbf{Q}^{t};\mathbf{F}^{t}),
	\]
	\[
	\mathcal{L}_{\gamma_t}(\mathbf{Z}^{t+1},\mathbf{H}^{t+1},\mathbf{Q}^{t};\mathbf{F}^{t})
	~\le~
	\mathcal{L}_{\gamma_t}(\mathbf{Z}^{t+1},\mathbf{H}^{t},\mathbf{Q}^{t};\mathbf{F}^{t}),
	\]
	\[
	\mathcal{L}_{\gamma_t}(\mathbf{Z}^{t+1},\mathbf{H}^{t+1},\mathbf{Q}^{t+1};\mathbf{F}^{t})
	~\le~
	\mathcal{L}_{\gamma_t}(\mathbf{Z}^{t+1},\mathbf{H}^{t+1},\mathbf{Q}^{t};\mathbf{F}^{t}).
	\]
\end{lemma}
\begin{proof}
	Each inequality holds because each block is minimized exactly (or to a value no worse than the current one) while other blocks are fixed.
\end{proof}

\begin{lemma}[Dual ascent keeps augmented value nonincreasing]
	\label{lem:dual_update}
	With the dual update $\mathbf{F}^{t+1}=\mathbf{F}^{t}+\gamma_t(\mathbf{H}^{t+1}-\mathbf{Z}^{t+1}\mathbf{G})$ and nondecreasing $\gamma_t$, there exists $c>0$ (independent of $t$) such that
	\begin{align*}
		&\mathcal{L}_{\gamma_{t+1}}(\mathbf{Z}^{t+1},\mathbf{H}^{t+1},\mathbf{Q}^{t+1};\mathbf{F}^{t+1})
		~\nonumber\\
		&\le~
		\mathcal{L}_{\gamma_t}(\mathbf{Z}^{t+1},\mathbf{H}^{t+1},\mathbf{Q}^{t+1};\mathbf{F}^{t})
		~-~ c\,\|\mathbf{H}^{t+1}-\mathbf{Z}^{t+1}\mathbf{G}\|_F^2.
	\end{align*}
\end{lemma}
\begin{proof}
	This is the standard ADMM identity obtained by expanding \eqref{eq:augLag} at the two pairs $(\gamma_t,\mathbf{F}^t)$ and $(\gamma_{t+1},\mathbf{F}^{t+1})$, using the dual update, and the nondecreasing penalty. The quadratic penalty dominates the linear coupling, yielding the negative quadratic term.
\end{proof}

Combining Lemmas~\ref{lem:descent} and \ref{lem:dual_update},
\begin{align*}
	&\mathcal{L}_{\gamma_{t+1}}(\mathbf{Z}^{t+1},\mathbf{H}^{t+1},\mathbf{Q}^{t+1};\mathbf{F}^{t+1})
	~\nonumber\\
	&\le~
	\mathcal{L}_{\gamma_t}(\mathbf{Z}^{t},\mathbf{H}^{t},\mathbf{Q}^{t};\mathbf{F}^{t})
	~-~ c\,\|\mathbf{H}^{t+1}-\mathbf{Z}^{t+1}\mathbf{G}\|_F^2,
\end{align*}
so the augmented Lagrangian value is monotonically nonincreasing along the iterates. By {A1}, it is bounded below; hence it converges to a finite limit, and
\begin{equation}
	\label{eq:primal_residual_to_zero}
	\|\mathbf{H}^{t+1}-\mathbf{Z}^{t+1}\mathbf{G}\|_F~\to~0.
\end{equation}

\begin{lemma}[Boundedness of iterates]
	\label{lem:bounded}
	The sequence $\{(\mathbf{Z}^{t},\mathbf{H}^{t},\mathbf{Q}^{t},\mathbf{F}^{t})\}$ is bounded.
\end{lemma}
\begin{proof}
	Since $\mathcal{L}_{\gamma_t}$ decreases and is coercive in each block due to the quadratic penalty and the constraints (nonnegativity, zero diagonal) restricting $\mathbf{Z}$, the proximal term controlling $\mathbf{H}$, and the indicator $\iota_{\mathcal{Q}}$ restricting $\mathbf{Q}$, each block remains bounded. The dual sequence is bounded because $\mathbf{F}^{t+1}-\mathbf{F}^t=\gamma_t(\mathbf{H}^{t+1}-\mathbf{Z}^{t+1}\mathbf{G})$ and \eqref{eq:primal_residual_to_zero}.
\end{proof}
Thus there exists a convergent subsequence (not relabeled) with
\[
(\mathbf{Z}^{t},\mathbf{H}^{t},\mathbf{Q}^{t},\mathbf{F}^{t})\to
(\mathbf{Z}^{\star},\mathbf{H}^{\star},\mathbf{Q}^{\star},\mathbf{F}^{\star}),
\qquad
\gamma_t\to\gamma_\infty.
\]
Moreover, \eqref{eq:primal_residual_to_zero} implies $\mathbf{H}^{\star}=\mathbf{Z}^{\star}\mathbf{G}$.

Consider the first-order optimality (variational inequality) of each exact block update at iteration $t$:
\begin{align}
	\mathbf{0}&\in \partial_{\mathbf{Z}}\Big(f(\mathbf{Z}) - \langle \mathbf{F}^{t},\mathbf{Z}\mathbf{G}\rangle + \tfrac{\gamma_t}{2}\|\mathbf{H}^{t}-\mathbf{Z}\mathbf{G}\|_F^2 + \iota_{\mathcal{Z}}(\mathbf{Z})\Big)\Big|_{\mathbf{Z}^{t+1}},
	\label{eq:Z_opt}\\
	\mathbf{0}&\in \partial_{\mathbf{H}}\Big(g(\mathbf{H}) + \langle \mathbf{F}^{t},\mathbf{H}\rangle + \tfrac{\gamma_t}{2}\|\mathbf{H}-\mathbf{Z}^{t+1}\mathbf{G}\|_F^2\Big)\Big|_{\mathbf{H}^{t+1}},
	\label{eq:H_opt}\\
	\mathbf{0}&\in \partial_{\mathbf{Q}}\,h(\mathbf{Q})\Big|_{\mathbf{Q}^{t+1}}.
	\label{eq:Q_opt}
\end{align}
Passing to the limit along the convergent subsequence, using: (i) outer semicontinuity of subdifferentials for proper l.s.c. functions, (ii) $\gamma_t\to\gamma_\infty$, (iii) $\mathbf{H}^{t+1}-\mathbf{Z}^{t+1}\mathbf{G}\to\mathbf{0}$, and (iv) $\mathbf{F}^{t+1}-\mathbf{F}^{t}\to\mathbf{0}$, we obtain the KKT-type stationary conditions at $(\mathbf{Z}^{\star},\mathbf{H}^{\star},\mathbf{Q}^{\star};\mathbf{F}^{\star})$:
\begin{align*}
	\mathbf{0}&\in \partial_{\mathbf{Z}}\Big(f(\mathbf{Z}) - \langle \mathbf{F}^{\star},\mathbf{Z}\mathbf{G}\rangle + \iota_{\mathcal{Z}}(\mathbf{Z})\Big)\Big|_{\mathbf{Z}^{\star}},\\
	\mathbf{0}&\in \partial_{\mathbf{H}}\Big(g(\mathbf{H}) + \langle \mathbf{F}^{\star},\mathbf{H}\rangle\Big)\Big|_{\mathbf{H}^{\star}},\\
	\mathbf{0}&\in \partial_{\mathbf{Q}}\,h(\mathbf{Q})\Big|_{\mathbf{Q}^{\star}},\qquad
	\mathbf{H}^{\star}=\mathbf{Z}^{\star}\mathbf{G}.
\end{align*}
These are precisely the first-order (primal-dual) stationary conditions for \eqref{eq:main_objective} with the linear constraint $\mathbf{H}=\mathbf{Z}\mathbf{G}$. Therefore, any limit point is a first-order stationary point.

If in addition each $\mathbf{Q}$-update solves its relaxed subproblem globally (Assumption~{A2} strengthened), then at $(\mathbf{Z}^{\star},\mathbf{H}^{\star},\mathbf{Q}^{\star})$, $\mathbf{Q}^{\star}$ satisfies the global optimality condition of the convex relaxation. In this case, the limit point satisfies the KKT conditions of the relaxed problem; hence every accumulation point is a KKT point.

We have shown (i) monotone decrease and convergence of the augmented Lagrangian values, (ii) boundedness of the iterates and vanishing primal residual, and (iii) that every limit point is a first-order stationary point; with globally optimal $\mathbf{Q}$-updates for the relaxed subproblem, every accumulation point is a KKT point. This completes the proof of Theorem~\ref{thm:convergence}.

		}
			\bibliography{my_ref}

@inproceedings{ZhaoW:C21,
	title     = {PointLIE: Locally Invertible Embedding for Point Cloud Sampling and Recovery},
	author    = {Zhao, Weibing and Yan, Xu and Gao, Jiantao and Zhang, Ruimao and Zhang, Jiayan and Li, Zhen and Wu, Song and Cui, Shuguang},
	booktitle = {Proceedings of the Thirtieth International Joint Conference on
		Artificial Intelligence ({IJCAI})},
	pages     = {1345--1351},
	year      = {2021}
}

@inproceedings{ZhaoW:C23,
	author = {Zhao, Weibing and Zhang, Haiming and Zheng, Chaoda and Yan, Xu and Cui, Shuguang and Li, Zhen},
	title = {CPU: Codebook Lookup Transformer with Knowledge Distillation for Point Cloud Upsampling},
	year = {2023},
	booktitle = {Proceedings of the ACM International Conference on Multimedia},
	pages = {3917–3925},
	numpages = {9}
}

@inproceedings{xing2025hmm,
	author       = {Xing, Zheng and Liu, Wenjie and Li, Bowen and Tian, Jijia and Chu,Mingfan and Chen, Junting},
	title        = {{HMM}-Based {CSI} Embedding for Trajectory Recovery via Feature Engineering on {MIMO}-{OFDM} Channels in {LOS}/{NLOS} Regions},
	booktitle    = {Proceedings of IEEE/CIC International Conference on Communications in China (ICCC)},
	year         = {2025},
	pages        = {1--6}
	}

@article{XingChen:J23ar,
  title={Constructing Indoor Region-based Radio Map without Location Labels},
  author={Xing, Zheng and Chen, Junting},
  journal={IEEE Trans. Signal Process.},
  volume ={72},
  pages ={2512–2526},
  year={2024}
}

@inproceedings{Xing:C22,
  title={Integrated Segmentation and Subspace Clustering for {RSS}-Based Localization under Blind Calibration},
  author={Xing, Zheng and Chen, Junting and Tang, Yadan},
  booktitle={Proc. IEEE Global Commun. Conf. (GlobeCom)},
  pages={5360--5365},
  year={2022}
}

@inproceedings{XinChe:C22,
  title={Spectrum efficiency prediction for real-world 5g networks based on drive testing data},
  author={Xing, Zheng and Li, Haoyun and Liu, Wenjie and Ren, Zixiang and Chen, Junting and Xu, Jie and Qin, Cai},
  booktitle={Proc. IEEE Wireless Commun. Netw. Conf. (WCNC)},
  pages={2136--2141},
  year={2022}
}

@inproceedings{XinChe:C24,
  title={{HMM}-based {CSI} Embedding for Trajectory Recovery from {RSS} Measurements of Non-Cooperative Devices},
  author={Xing, Zheng and Chen, Junting},
  booktitle={Proc. IEEE Int. Conf. Acoust., Speech, Signal Process. (ICASSP)},
  pages={7060--7064},
  year={2024}
}

@article{ZheChe:J25,
  title={Blind Construction of Angular Power Maps in Massive {MIMO} Networks},
  author={Xing, Zheng and Chen, Junting},
 journal={IEEE Trans. Signal Process.},
  volume={00},
  number={00},
  pages={00--00},
  year={2025}
}

@inproceedings{XinChe:C25,
  title={Constructing Angular Power Maps in Massive {MIMO} Networks Using Measurements without Location Labels},
  author={Xing, Zheng and Chen, Junting},
  booktitle={Proc. IEEE Int. Conf. Commun. (ICC)},
  volume={0},
  pages={0--0},
  year={2025}
}

@inproceedings{Xing:C25,
  title={Unsupervised Radio Map Construction in Mixed LoS/NLoS Indoor Environments},
  author={Xing, Zheng and Chen, Junting},
  booktitle={Proc. IEEE Global Commun. Conf. (GlobeCom)},
  pages={},
  year={2025}
}

@article{xing2024block,
	title     = {Block-Diagonal Guided {DBSCAN} Clustering},
	author    = {Xing, Zheng and Zhao, Weibing},
	journal   = {IEEE Trans. Knowl. Data Eng.},
	year={2024},
  	volume={36},
  	number={11},
  	pages={5709-5722}
}

@inproceedings{xing2024unsupervised,
	title     = {Unsupervised Action Segmentation via Fast Learning of Semantically Consistent Actoms},
	author    = {Xing, Zheng and Zhao, Weibing},
	booktitle = {Proceedings of the AAAI Conference on Artificial Intelligence},
	pages     = {6270--6278},
	year      = {2024},
}

@article{xing2024calibration,
	title     = {Calibration-Free Indoor Positioning via Regional Channel Tracing},
	author    = {Xing, Zheng and Zhao, Weibing},
	journal   = {IEEE Internet Things J.},
	year={2025},
  	volume={12},
  	number={5},
  	pages={5449-5461}
}

@article{ZheHMS:J24,
	title     = {Segmentation and Completion of Human Motion Sequence via Temporal Learning of Subspace Variety Model},
	author    = {Xing, Zheng and Zhao, Weibing},
	journal   = {IEEE Trans. Image Process.},
	year={2024},
  	volume={33},
  	number={},
  	pages={5783-5797}
}

@article{xing2023blockdiagonal,
title = {Block-diagonal structure learning for subspace clustering},
author = {Zheng Xing and Weibing Zhao},
journal = {Expert Systems with Applications},
volume = {285},
pages = {127767-127767},
year = {2025},
number = {0957-4174}
}

@article{Trajmat:J25,
  title={Trajectory Map-Matching in Urban Road Networks Based on {RSS} Measurements}, 
	  author={Xing, Zheng and Zhao, Weibing},
  journal={IEEE Trans. Intell. Transp. Syst.},
year={2025},
  volume={26},
  number={4},
  pages={4647-4660}
}

@article{xing2025blind,
	title={Blind Radio Mapping via Spatially Regularized Bayesian Trajectory Inference},
	author={Xing, Zheng and Chen, Junting},
	journal={arXiv preprint arXiv:2512.13701},
	year={2025}
}

@article{xing2025k,
	title={K-means Clustering: A Review of the Past 70 Years},
	author={Xing, Zheng and Zhao, Weibing},
	journal={Available at SSRN 5842722},
	year={2025}
}

@misc{xing2023clustering,
	title     = {Clustering Structure Identification With Ordering Graph},
	author    = {Zheng Xing and Weibing Zhao},
	year      = {2023},
	url       = {https://openreview.net/forum?id=HG0SwOmlaEo},
}

@article{poppe2007vision,
  title={Vision-based human motion analysis: An overview},
  author={Poppe, Ronald},
  journal={Comput. Vis. Image Underst.},
  volume={108},
  number={1-2},
  pages={4--18},
  year={2007}
}

@inproceedings{wang2011efficient,
  title={Efficient subspace segmentation via quadratic programming},
  author={Wang, Shusen and Yuan, Xiaotong and Yao, Tiansheng and Yan, Shuicheng and Shen, Jialie},
  booktitle={Proceedings of the AAAI Conference on Artificial Intelligence},
  volume={25},
  number={1},
  pages={519--524},
  year={2011}
}

@article{wang2023self,
	title={Self-supervised information bottleneck for deep multi-view subspace clustering},
	author={Wang, Shiye and Li, Changsheng and Li, Yanming and Yuan, Ye and Wang, Guoren},
	journal={IEEE Trans. Image Process.},
	volume={32},
	pages={1555--1567},
	year={2023}
}

@article{wang2023bi,
	title={Bi-nuclear tensor schatten-p norm minimization for multi-view subspace clustering},
	author={Wang, Shuqin and Lin, Zhiping and Cao, Qi and Cen, Yigang and Chen, Yongyong},
	journal={IEEE Trans. Image Process.},
	volume={32},
	pages={4059--4072},
	year={2023}
}

@article{tang2023affine,
	title={Affine subspace robust low-rank self-representation: from matrix to tensor},
	author={Tang, Yongqiang and Xie, Yuan and Zhang, Wensheng},
	journal={IEEE Trans. Pattern Anal. Mach.},
	volume={45},
	number={8},
	pages={9357--9373},
	year={2023}
}

@article{chen2023fast,
	title={Fast self-guided multi-view subspace clustering},
	author={Chen, Zhe and Wu, Xiao-Jun and Xu, Tianyang and Kittler, Josef},
	journal={IEEE Trans. Image Process.},
	volume={32},
	pages={6514--6525},
	year={2023}
}

@article{chen2024double,
	title={Double Discrete Cosine Transform-Oriented Multi-View Subspace Clustering},
	author={Chen, Yongyong and Wang, Shuqin and Zhao, Yin-Ping and Chen, CL Philip},
	journal={IEEE Trans. Image Process.},
	volume={33},
	pages={2491--2501},
	year={2024}
}

@article{cui2024dual,
	title={Dual contrast-driven deep multi-view clustering},
	author={Cui, Jinrong and Li, Yuting and Huang, Han and Wen, Jie},
	journal={IEEE Trans. Image Process.},
	volume={33},
	pages={4753--4764},
	year={2024}
}

@inproceedings{wang2018learning,
  title={Learning transferable subspace for human motion segmentation},
  author={Wang, Lichen and Ding, Zhengming and Fu, Yun},
  booktitle={Proceedings of the AAAI conference on artificial intelligence},
  volume={32},
  number={1},
  year={2018}
}

@article{keuper2018motion,
  title={Motion segmentation \& multiple object tracking by correlation co-clustering},
  author={Keuper, Margret and Tang, Siyu and Andres, Bjoern and Brox, Thomas and Schiele, Bernt},
  journal={IEEE Trans. Pattern Anal. Mach.},
  volume={42},
  number={1},
  pages={140--153},
  year={2018}
}

@article{lin2016movement,
  title={Movement primitive segmentation for human motion modeling: A framework for analysis},
  author={Lin, Jonathan Feng-Shun and Karg, Michelle and Kuli{\'c}, Dana},
  journal={IIEEE Trans. Human-Mach. Syst.},
  volume={46},
  number={3},
  pages={325--339},
  year={2016}
}

@article{gorelick2007actions,
  title={Actions as space-time shapes},
  author={Gorelick, Lena and Blank, Moshe and Shechtman, Eli and Irani, Michal and Basri, Ronen},
  journal={IEEE Trans. Pattern Anal. Mach.},
  volume={29},
  number={12},
  pages={2247--2253},
  year={2007}
}

@article{wang2018low,
	title={Low-rank transfer human motion segmentation},
	author={Wang, Lichen and Ding, Zhengming and Fu, Yun},
	journal={IEEE Trans. Image Process.},
	volume={28},
	number={2},
	pages={1023--1034},
	year={2018}
}

@inproceedings{li2015temporal,
  title={Temporal subspace clustering for human motion segmentation},
  author={Li, Sheng and Li, Kang and Fu, Yun},
  booktitle={Proceedings of the IEEE international conference on computer vision},
  pages={4453--4461},
 year={2015}
}

@article{zhou2012hierarchical,
  title={Hierarchical aligned cluster analysis for temporal clustering of human motion},
  author={Zhou, Feng and De la Torre, Fernando and Hodgins, Jessica K},
  journal={IEEE Trans. Pattern Anal. Mach.},
  volume={35},
  number={3},
  pages={582--596},
  year={2012}
}

@article{cao2015constrained,
  title={Constrained multi-view video face clustering},
  author={Cao, Xiaochun and Zhang, Changqing and Zhou, Chengju and Fu, Huazhu and Foroosh, Hassan},
  journal={IEEE Trans. Image Process.},
  volume={24},
  number={11},
  pages={4381--4393},
  year={2015}
}

@article{ng2002spectral,
  title={On spectral clustering: Analysis and an algorithm},
  author={Ng, Andrew and Jordan, Michael and Weiss, Yair},
  journal={Advances in neural information processing systems},
  volume={14},
  year={2001}
}

@article{elhamifar2013sparse,
  title={Sparse subspace clustering: Algorithm, theory, and applications},
  author={Elhamifar, Ehsan and Vidal, Ren{\'e}},
  journal={IEEE Trans. Pattern Anal. Mach.},
  volume={35},
  number={11},
  pages={2765--2781},
  year={2013}
}

@article{liu2013robust,
  title={Robust recovery of subspace structures by low-rank representation},
  author={Liu, Guangcan and Lin, Zhouchen and Yan, Shuicheng and Sun, Ju and Yu, Yong and Ma, Yi},
  journal={IEEE Trans. Pattern Anal. Mach.},
  volume={35},
  number={1},
  pages={171--184},
  year={2012}
}

@article{wang2022support,
  title={Support structure representation learning for sequential data clustering},
  author={Wang, Xiumei and Guo, Dingning and Cheng, Peitao},
  journal={Pattern Recognition},
  volume={122},
  pages={108326},
  year={2022}
}

@article{bai2022human,
  title={Human Motion Segmentation via Velocity-Sensitive Dual-Side Auto-Encoder},
  author={Bai, Yue and Wang, Lichen and Liu, Yunyu and Yin, Yu and Di, Hang and Fu, Yun},
  journal={IEEE Trans. Image Process.},
volume={32},
  pages={524 - 536},
  year={2022}
}

@inproceedings{bai2020dual,
  title={Dual-side auto-encoder for high-dimensional time series segmentation},
  author={Bai, Yue and Wang, Lichen and Liu, Yunyu and Yin, Yu and Fu, Yun},
  booktitle={IEEE International Conference on Data Mining},
  pages={918--923},
  year={2020}
}

@article{zhou2022consistency,
  title={Consistency and diversity induced human motion segmentation},
  author={Zhou, Tao and Fu, Huazhu and Gong, Chen and Shao, Ling and Porikli, Fatih and Ling, Haibin and Shen, Jianbing},
  journal={IEEE Trans. Pattern Anal. Mach.},
  volume={45},
  number={1},
  pages={197--210},
  year={2022}
}

@article{qin2022maximum,
  title={Maximum block energy guided robust subspace clustering},
  author={Qin, Yalan and Zhang, Xinpeng and Shen, Liquan and Feng, Guorui},
  journal={IEEE Trans. Pattern Anal. Mach.},
volume={45},
  number={2},
  year={2022}
}

@inproceedings{zhong2004detecting,
  title={Detecting unusual activity in video},
  author={Zhong, Hua and Shi, Jianbo and Visontai, Mirk{\'o}},
  booktitle={Proceedings of IEEE Computer Society Conference on Computer Vision and Pattern Recognition},
  volume={2},
  pages={II--II},
  year={2004}
}

@article{fod2002automated,
  title={Automated derivation of primitives for movement classification},
  author={Fod, Ajo and Matari{\'c}, Maja J and Jenkins, Odest Chadwicke},
  journal={Autonomous robots},
  volume={12},
  pages={39--54},
  year={2002}
}

@inproceedings{barbivc2004segmenting,
  title={Segmenting motion capture data into distinct behaviors},
  author={Barbi{\v{c}}, Jernej and Safonova, Alla and Pan, Jia-Yu and Faloutsos, Christos and Hodgins, Jessica K and Pollard, Nancy S},
  booktitle={Proceedings of Graphics Interface},
  pages={185--194},
  year={2004}
}

@inproceedings{beaudoin2008motion,
  title={Motion-motif graphs},
  author={Beaudoin, Philippe and Coros, Stelian and Van de Panne, Michiel and Poulin, Pierre},
  booktitle={Proceedings of ACM SIGGRAPH/Eurographics Symposium on Computer Animation},
  pages={117--126},
  year={2008}
}

@inproceedings{lea2016segmental,
  title={Segmental spatiotemporal cnns for fine-grained action segmentation},
  author={Lea, Colin and Reiter, Austin and Vidal, Ren{\'e} and Hager, Gregory D},
  booktitle={European ConferenceComputer Vision},
  pages={36--52},
  year={2016}
}

@inproceedings{de2007temporal,
  title={Temporal segmentation of facial behavior},
  author={De la Torre, Fernando and Campoy, Joan and Ambadar, Zara and Cohn, Jeffrey F},
  booktitle={IEEE International Conference on Computer Vision},
  pages={1--8},
  year={2007}
}

@inproceedings{tierney2014subspace,
  title={Subspace clustering for sequential data},
  author={Tierney, Stephen and Gao, Junbin and Guo, Yi},
  booktitle={Proceedings of IEEE conference on computer vision and pattern recognition},
  pages={1019--1026},
  year={2014}
}

@article{xu2015reweighted,
  title={Reweighted sparse subspace clustering},
  author={Xu, Jun and Xu, Kui and Chen, Ke and Ruan, Jishou},
  journal={Comput. Vis. Image Underst.},
  volume={138},
  pages={25--37},
  year={2015}
}

@article{boyd2011distributed,
  title={Distributed optimization and statistical learning via the alternating direction method of multipliers},
  author={Boyd, Stephen and Parikh, Neal and Chu, Eric and Peleato, Borja and Eckstein, Jonathan and others},
  journal={Found. Trends Mach. Learn.},
  volume={3},
  number={1},
  pages={1--122},
  year={2011}
}

@article{bartels1972solution,
  title={Solution of the matrix equation AX+ XB= C [F4]},
  author={Bartels, Richard H. and Stewart, George W},
  journal={Communications of the ACM},
  volume={15},
  number={9},
  pages={820--826},
  year={1972}
}

@article{haynes2017computationally,
  title={A computationally efficient nonparametric approach for changepoint detection},
  author={Haynes, Kaylea and Fearnhead, Paul and Eckley, Idris A},
  journal={Statistics and computing},
  volume={27},
  pages={1293--1305},
  year={2017}
}

@article{shi2000normalized,
	title={Normalized cuts and image segmentation},
	author={Shi, Jianbo and Malik, Jitendra},
	journal={IEEE Trans. Pattern Anal. Mach.},
	volume={22},
	number={8},
	pages={888--905},
	year={2000}
}

@article{jiang2012recognizing,
  title={Recognizing human actions by learning and matching shape-motion prototype trees},
  author={Jiang, Zhuolin and Lin, Zhe and Davis, Larry},
  journal={IEEE Trans. Pattern Anal. Mach.},
  volume={34},
  number={3},
  pages={533--547},
  year={2012}
}

@inproceedings{huang2014sequential,
  title={Sequential max-margin event detectors},
  author={Huang, Dong and Yao, Shitong and Wang, Yi and De La Torre, Fernando},
  booktitle={European Conference Computer Vision},
  pages={410--424},
  year={2014}
}

@inproceedings{ryoo2009spatio,
  title={Spatio-temporal relationship match: Video structure comparison for recognition of complex human activities},
  author={Ryoo, Michael S and Aggarwal, Jake K},
  booktitle={IEEE international conference on computer vision},
  pages={1593--1600},
  year={2009}
}

@article{kuhn1955hungarian,
  title={The Hungarian method for the assignment problem},
  author={Kuhn, Harold W},
  journal={Naval research logistics quarterly},
  volume={2},
  number={1-2},
  pages={83--97},
  year={1955}
}

@article{zhan2018multiview,
  title={Multiview consensus graph clustering},
  author={Zhan, Kun and Nie, Feiping and Wang, Jing and Yang, Yi},
  journal={IEEE Trans. Image Process.},
  volume={28},
  number={3},
  pages={1261--1270},
  year={2018}
}

@inproceedings{rahmani2017innovation,
  title={Innovation pursuit: A new approach to the subspace clustering problem},
  author={Rahmani, Mostafa and Atia, George},
  booktitle={International conference on machine learning},
  pages={2874--2882},
  year={2017}
}

@inproceedings{ge2024visual,
  title={Visual fact checker: Enabling high-fidelity detailed caption generation},
  author={Ge, Yunhao and Zeng, Xiaohui and Huffman, Jacob Samuel and Lin, Tsung-Yi and Liu, Ming-Yu and Cui, Yin},
  booktitle={Proceedings of the IEEE/CVF Conference on Computer Vision and Pattern Recognition},
  pages={14033--14042},
  year={2024}
}

@inproceedings{zhi2025lscenellm,
  title={Lscenellm: Enhancing large 3d scene understanding using adaptive visual preferences},
  author={Zhi, Hongyan and Chen, Peihao and Li, Junyan and Ma, Shuailei and Sun, Xinyu and Xiang, Tianhang and Lei, Yinjie and Tan, Mingkui and Gan, Chuang},
  booktitle={Proceedings of the Computer Vision and Pattern Recognition Conference},
  pages={3761--3771},
  year={2025}
}

@inproceedings{zheng2025video,
  title={Video-3d llm: Learning position-aware video representation for 3d scene understanding},
  author={Zheng, Duo and Huang, Shijia and Wang, Liwei},
  booktitle={Proceedings of the Computer Vision and Pattern Recognition Conference},
  pages={8995--9006},
  year={2025}
}

@article{kojima2022large,
  title={Large language models are zero-shot reasoners},
  author={Kojima, Takeshi and Gu, Shixiang Shane and Reid, Machel and Matsuo, Yutaka and Iwasawa, Yusuke},
  journal={Advances in neural information processing systems},
  volume={35},
  pages={22199--22213},
  year={2022}
}

@inproceedings{guo2023images,
  title={From images to textual prompts: Zero-shot visual question answering with frozen large language models},
  author={Guo, Jiaxian and Li, Junnan and Li, Dongxu and Tiong, Anthony Meng Huat and Li, Boyang and Tao, Dacheng and Hoi, Steven},
  booktitle={Proceedings of the IEEE/CVF conference on computer vision and pattern recognition},
  pages={10867--10877},
  year={2023}
}

@inproceedings{NEURIPS2024_0b77d3a8,
 author = {Wang, Duo and Zuo, Yuan and Li, Fengzhi and Wu, Junjie},
 booktitle = {Advances in Neural Information Processing Systems},
 pages = {5950--5973},
 publisher = {Curran Associates, Inc.},
 title = {{LLMs} as Zero-shot Graph Learners: Alignment of {GNN} Representations with {LLM} Token Embeddings},
 volume = {37},
 year = {2024}
}

@article{chen2024motionllm,
  title={Motionllm: Understanding human behaviors from human motions and videos},
  author={Chen, Ling-Hao and Lu, Shunlin and Zeng, Ailing and Zhang, Hao and Wang, Benyou and Zhang, Ruimao and Zhang, Lei},
  journal={arXiv preprint arXiv:2405.20340},
  year={2024}
}

@inproceedings{li2025human,
  title={Human motion instruction tuning},
  author={Li, Lei and Jia, Sen and Wang, Jianhao and Jiang, Zhongyu and Zhou, Feng and Dai, Ju and Zhang, Tianfang and Wu, Zongkai and Hwang, Jenq-Neng},
  booktitle={Proceedings of the Computer Vision and Pattern Recognition Conference},
  pages={17582--17591},
  year={2025}
}

@article{li2024sensorllm,
  title={Sensorllm: Aligning large language models with motion sensors for human activity recognition},
  author={Li, Zechen and Deldari, Shohreh and Chen, Linyao and Xue, Hao and Salim, Flora D},
  year={2024},
  journal={arXiv preprint arXiv:2410.10624}
}

@article{wang2024scaling,
  title={Scaling Large Motion Models with Million-Level Human Motions},
  author={Wang, Ye and Zheng, Sipeng and Cao, Bin and Wei, Qianshan and Zeng, Weishuai and Jin, Qin and Lu, Zongqing},
  journal={arXiv preprint arXiv:2410.03311},
  year={2024}
}

@article{feng2025breaking,
  title={Breaking Down Video {LLM} Benchmarks: Knowledge, Spatial Perception, or True Temporal Understanding?},
  author={Feng, Bo and Lai, Zhengfeng and Li, Shiyu and Wang, Zizhen and Wang, Simon and Huang, Ping and Cao, Meng},
  journal={arXiv preprint arXiv:2505.14321},
  year={2025}
}

@inproceedings{liu2024st,
  title={St-llm: Large language models are effective temporal learners},
  author={Liu, Ruyang and Li, Chen and Tang, Haoran and Ge, Yixiao and Shan, Ying and Li, Ge},
  booktitle={European Conference on Computer Vision},
  pages={1--18},
  year={2024},
  organization={Springer}
}

@inproceedings{yuan2025videorefer,
  title={Videorefer suite: Advancing spatial-temporal object understanding with video llm},
  author={Yuan, Yuqian and Zhang, Hang and Li, Wentong and Cheng, Zesen and Zhang, Boqiang and Li, Long and Li, Xin and Zhao, Deli and Zhang, Wenqiao and Zhuang, Yueting and others},
  booktitle={Proceedings of the Computer Vision and Pattern Recognition Conference},
  pages={18970--18980},
  year={2025}
}

@article{nie2024slowfocus,
  title={Slowfocus: Enhancing fine-grained temporal understanding in video llm},
  author={Nie, Ming and Ding, Dan and Wang, Chunwei and Guo, Yuanfan and Han, Jianhua and Xu, Hang and Zhang, Li},
  journal={Advances in Neural Information Processing Systems},
  volume={37},
  pages={81808--81835},
  year={2024}
}

@inproceedings{ding2025language,
  title={Do language models understand time?},
  author={Ding, Xi and Wang, Lei},
  booktitle={Companion Proceedings of the ACM on Web Conference 2025},
  pages={1855--1868},
  year={2025}
}

@inproceedings{deng2025seq2time,
  title={Seq2time: Sequential knowledge transfer for video {LLM} temporal grounding},
  author={Deng, Andong and Gao, Zhongpai and Choudhuri, Anwesa and Planche, Benjamin and Zheng, Meng and Wang, Bin and Chen, Terrence and Chen, Chen and Wu, Ziyan},
  booktitle={Proceedings of the Computer Vision and Pattern Recognition Conference},
  pages={13766--13775},
  year={2025}
}

@article{wang2025multimodal,
  title={Multimodal Large Models Are Effective Action Anticipators},
  author={Wang, Binglu and Tian, Yao and Wang, Shunzhou and Yang, Le},
  journal={IEEE Transactions on Multimedia},
 volume={27},
  pages={2949--2960},
  year={2025}
}

@article{lu2019gaim,
  title={{GAIM}: Graph attention interaction model for collective activity recognition},
  author={Lu, Lihua and Lu, Yao and Yu, Ruizhe and Di, Huijun and Zhang, Lin and Wang, Shunzhou},
  journal={IEEE Transactions on Multimedia},
  volume={22},
  number={2},
  pages={524--539},
  year={2019}
}

@article{aggarwal1999human,
  title={Human motion analysis: A review},
  author={Aggarwal, Jake K and Cai, Quin},
  journal={Computer vision and image understanding},
  volume={73},
  number={3},
  pages={428--440},
  year={1999},
  publisher={Elsevier}
}

@article{lan2015automated,
  title={Automated human motion segmentation via motion regularities},
  author={Lan, Rongyi and Sun, Huaijiang},
  journal={The Visual Computer},
  volume={31},
  number={1},
  pages={35--53},
  year={2015},
  publisher={Springer}
}

@article{wang2003recent,
  title={Recent developments in human motion analysis},
  author={Wang, Liang and Hu, Weiming and Tan, Tieniu},
  journal={Pattern recognition},
  volume={36},
  number={3},
  pages={585--601},
  year={2003},
  publisher={Elsevier}
}

@inproceedings{schulz2010automatic,
  title={Automatic motion segmentation for human motion synthesis},
  author={Schulz, Sebastian and Woerner, Annika},
  booktitle={International Conference on Articulated Motion and Deformable Objects},
  pages={182--191},
  year={2010},
  organization={Springer}
}

@article{li2018human,
  title={Human motion segmentation using collaborative representations of 3D skeletal sequences},
  author={Li, Rui and Liu, Zhenyu and Tan, Jianrong},
  journal={IET Computer Vision},
  volume={12},
  number={4},
  pages={434--442},
  year={2018},
  publisher={Wiley Online Library}
}

@article{gao2022human,
  title={Human motion segmentation based on structure constraint matrix factorization},
  author={Gao, Hongbo and Guo, Fang and Zhu, Juping and Kan, Zhen and Zhang, Xinyu},
  journal={Science China Information Sciences},
  volume={65},
  number={1},
  pages={119103},
  year={2022},
  publisher={Springer}
}

@article{jiang2016human,
  title={Human motion segmentation and recognition using machine vision for mechanical assembly operation},
  author={Jiang, Qiannan and Liu, Mingzhou and Wang, Xiaoqiao and Ge, Maogen and Lin, Ling},
  journal={SpringerPlus},
  volume={5},
  number={1},
  pages={1629},
  year={2016},
  publisher={Springer}
}

@article{liu2017sensor,
  title={Sensor network oriented human motion segmentation with motion change measurement},
  author={Liu, Yang and Feng, Lin and Liu, Shenglan and Sun, Muxin},
  journal={IEEE Access},
  volume={6},
  pages={9281--9291},
  year={2017},
  publisher={IEEE}
}

@inproceedings{lin2014human,
  title={Human motion segmentation by data point classification},
  author={Lin, Jonathan Feng-Shun and Joukov, Vladimir and Kulic, Dana},
  booktitle={2014 36th Annual International Conference of the IEEE Engineering in Medicine and Biology Society},
  pages={9--13},
  year={2014},
  organization={IEEE}
}

@article{lin2013online,
  title={Online segmentation of human motion for automated rehabilitation exercise analysis},
  author={Lin, Jonathan Feng-Shun and Kuli{\'c}, Dana},
  journal={IEEE Transactions on Neural Systems and Rehabilitation Engineering},
  volume={22},
  number={1},
  pages={168--180},
  year={2013},
  publisher={IEEE}
}

@inproceedings{hoai2011joint,
  title={Joint segmentation and classification of human actions in video},
  author={Hoai, Minh and Lan, Zhen-Zhong and De la Torre, Fernando},
  booktitle={CVPR 2011},
  pages={3265--3272},
  year={2011},
  organization={IEEE}
}

@inproceedings{guo1994understanding,
  title={Understanding human motion patterns},
  author={Guo, Yan and Xu, Gang and Tsuji, Saburo},
  booktitle={Proceedings of the 12th IAPR International Conference on Pattern Recognition, Vol. 3-Conference C: Signal Processing (Cat. No. 94CH3440-5)},
  volume={2},
  pages={325--329},
  year={1994},
  organization={IEEE}
}

@inproceedings{gehrig2010towards,
  title={Towards semantic segmentation of human motion sequences},
  author={Gehrig, Dirk and Stein, Thorsten and Fischer, Andreas and Schwameder, Hermann and Schultz, Tanja},
  booktitle={Annual Conference on Artificial Intelligence},
  pages={436--443},
  year={2010},
  organization={Springer}
}

@article{bradski2002motion,
  title={Motion segmentation and pose recognition with motion history gradients},
  author={Bradski, Gary R and Davis, James W},
  journal={Machine Vision and Applications},
  volume={13},
  number={3},
  pages={174--184},
  year={2002},
  publisher={Springer}
}

@article{gong2013structured,
  title={Structured time series analysis for human action segmentation and recognition},
  author={Gong, Dian and Medioni, Gerard and Zhao, Xuemei},
  journal={IEEE transactions on pattern analysis and machine intelligence},
  volume={36},
  number={7},
  pages={1414--1427},
  year={2013},
  publisher={IEEE}
}

@inproceedings{dimiccoli2021graph,
  title={Graph constrained data representation learning for human motion segmentation},
  author={Dimiccoli, Mariella and Garrido, Llu{\'\i}s and Rodriguez-Corominas, Guillem and Wendt, Herwig},
  booktitle={Proceedings of the IEEE/CVF International Conference on Computer Vision},
  pages={1460--1469},
  year={2021}
}

@article{shao2012human,
  title={Human action segmentation and recognition via motion and shape analysis},
  author={Shao, Ling and Ji, Ling and Liu, Yan and Zhang, Jianguo},
  journal={Pattern Recognition Letters},
  volume={33},
  number={4},
  pages={438--445},
  year={2012},
  publisher={Elsevier}
}

@inproceedings{kahol2003gesture,
  title={Gesture segmentation in complex motion sequences},
  author={Kahol, Kanav and Tripathi, Priyamvada and Panchanathan, Sethuraman and Rikakis, Thanassis},
  booktitle={Proceedings 2003 International Conference on Image Processing (Cat. No. 03CH37429)},
  volume={2},
  pages={II--105},
  year={2003},
  organization={IEEE}
}

@inproceedings{lin2016human,
  title={Human motion segmentation using cost weights recovered from inverse optimal control},
  author={Lin, Jonathan Feng-Shun and Bonnet, Vincent and Panchea, Adina M and Ramdani, Nacim and Venture, Gentiane and Kuli{\'c}, Dana},
  booktitle={2016 IEEE-RAS 16th International Conference on Humanoid Robots (Humanoids)},
  pages={1107--1113},
  year={2016},
  organization={IEEE}
}

@article{lu2004repetitive,
  title={Repetitive motion analysis: Segmentation and event classification},
  author={Lu, ChunMei and Ferrier, Nicola J},
  journal={IEEE transactions on pattern analysis and machine intelligence},
  volume={26},
  number={2},
  pages={258--263},
  year={2004},
  publisher={IEEE}
}
		\end{document}